\documentclass{article}

\usepackage{microtype}
\usepackage{graphicx}
\usepackage{subcaption}
\usepackage{booktabs} 
\usepackage{adjustbox}
\usepackage{multirow}
\usepackage{cancel}
\usepackage{adjustbox}
\usepackage[dvipsnames]{xcolor}

\usepackage[backref=page]{hyperref}
\renewcommand*{\backref}[1]{}
\renewcommand*{\backrefalt}[4]{%
    \ifcase #1 (Not cited.)%
    \or        (Cited on page~#2.)%
    \else      (Cited on pages~#2.)%
    \fi}

\usepackage[accepted]{icml2024}

\usepackage{amsmath}
\usepackage{amssymb}
\usepackage{mathtools}
\usepackage{amsthm}

\newcommand{\mat}[1]{\mathbf{#1}}
\renewcommand{\vec}[1]{\boldsymbol{#1}}

\usepackage[capitalize,noabbrev]{cleveref}

\theoremstyle{plain}
\newtheorem{theorem}{Theorem}
\newtheorem{proposition}[theorem]{Proposition}

\newtheorem{example}[theorem]{Example}
\newtheorem{definition}[theorem]{Definition}
\newtheorem{remark}[theorem]{Remark}

\usepackage[textsize=tiny]{todonotes}

\icmltitlerunning{Gaussian Processes on Cellular Complexes}

\begin{document}

\twocolumn[
\icmltitle{Gaussian Processes on Cellular Complexes}



\icmlsetsymbol{equal}{*}

\begin{icmlauthorlist}
\icmlauthor{Mathieu Alain}{equal,yyy}
\icmlauthor{So Takao}{equal,yyy,xxx}
\icmlauthor{Brooks Paige}{yyy}
\icmlauthor{Marc Peter Deisenroth}{yyy}
\end{icmlauthorlist}

\icmlaffiliation{yyy}{Centre for Artificial Intelligence, University College London, London, UK.}
\icmlaffiliation{xxx}{Department of Computing and Mathematical Sciences, California Institute of Technology, Pasadena, CA, USA}

\icmlcorrespondingauthor{Mathieu Alain}{mathieu.alain.21@ucl.ac.uk}

\icmlkeywords{Machine Learning, ICML}

\vskip 0.3in
]



\printAffiliationsAndNotice{\icmlEqualContribution} 

\begin{abstract}
In recent years, there has been considerable interest in developing machine learning models on graphs to account for topological inductive biases.
In particular, recent attention has been given to Gaussian processes on such structures since they can additionally account for uncertainty.
However, graphs are limited to modelling relations between two vertices. In this paper, we go beyond this dyadic setting and consider polyadic relations that include interactions between vertices, edges and one of their generalisations, known as cells.
Specifically, we propose Gaussian processes on cellular complexes, a generalisation of graphs that captures interactions between these higher-order cells. One of our key contributions is the derivation of two novel kernels, one that generalises the graph Matérn kernel and one that additionally mixes information of different cell types. 
\end{abstract}

\section{Introduction}
The abundance of graph-structured problems in science and engineering stimulates the development of machine learning (ML) models, such as graph neural networks (GNNs) \citep{Scarselli2008} and graph kernel machines \citep{smola2003kernels}.
The former has achieved great success in a broad range of tasks, from molecular docking \citep{Corso2023} to text summarisation \citep{Fernandes2019}. However, GNNs do not provide predictive uncertainty, which is an essential feature in decision-making applications. Recent work on Gaussian processes (GPs) defined on graphs \citep{borovitskiy2020matern, Nikitin2022, Opolka2022, Zhi2023} takes the latter approach, which naturally accounts for uncertainty quantification, but may lack the expressibility of GNNs.

Although graphs are an invaluable data structure, one of their main limitations is that they cannot represent interactions beyond the dyadic setting (i.e., between two vertices). However, these interactions do exist and have important applications, such as group interactions in social networks \citep{Alvarez-Rodriguez2021}, neuronal dynamics in cortexes \citep{Yu2011}, and trigenic interactions in gene networks \citep{Kuzminz2018}. Cellular complexes are a generalisation of graphs that have the ability to model such `polyadic' interactions \citep{Hatcher2001} (see Figure \ref{fig:illustrations_graph_cc}). For this reason, they are gradually being used in ML \citep{Hajij2020, Bodnar2021} and signal processing \citep{Barbarossa2020, Roddenberry2022}. 
Although these models have helped to expand the variety of problems one can tackle with ML, as with GNNs, they typically do not quantify uncertainty.

\begin{figure}[t]
    \centering
    \begin{subfigure}[t]{0.38\hsize}
        \centering  \includegraphics[height=1.8cm]{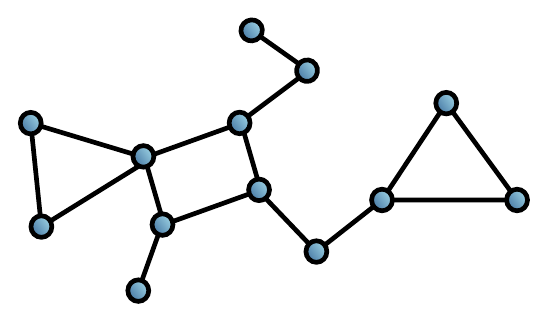}
        \caption[Network2]%
        {{Graph}}    
        \label{fig:graph}
    \end{subfigure}
    \hfill
    \begin{subfigure}[t]{0.6\hsize} 
        \centering 
\includegraphics[height=1.2cm]{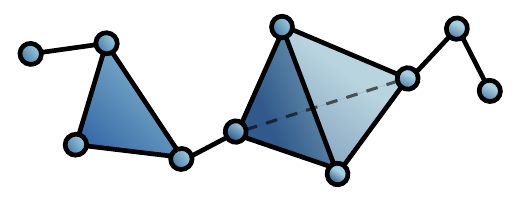}
        \caption[]%
        {{Simplicial complex}}    
        \label{fig:sc}
    \end{subfigure}
    \\
    \begin{subfigure}[t]{0.5\hsize} 
        \centering 
\includegraphics[height=1.7cm]{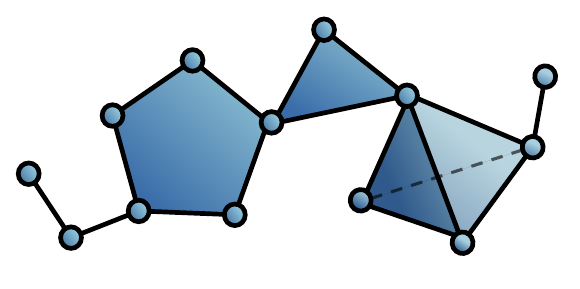}
        \caption[]%
        {{Cellular complex}}    
        \label{fig:cc}
    \end{subfigure}
    \caption
    {Graph, simplicial complex, and cellular complex (specifically: polyhedral complex). A simplicial complex cannot represent arbitrary polygons like the pentagon in (\subref{fig:cc}). }
    \label{fig:illustrations_graph_cc}
\end{figure}

In this paper, we fill this gap by proposing GPs defined on cellular complexes, which enables predictions on different types of cells, such as vertices \emph{and} edges, while equipping them with uncertainty that is consistent with prior assumptions about the model and data. In particular, we propose the cellular Matérn kernel, a generalisation of the graph Matérn kernel \citep{borovitskiy2020matern}, which enables modelling of signals on arbitrary cell types. By fixing an orientation on each cell (i.e., a reference `direction'), this produces predictions on cells that are {\em directed}, allowing us to consider signals that have an associated direction.

Furthermore, prompted by settings with strong correlations between data supported on different cell types, we also develop a new kernel, the reaction-diffusion kernel, which leverages the Dirac operator \citep{Bianconi2021, calmon2023dirac} to mix information between different types of cells.
This enables us to model various types of signal jointly so that inference on one type of cell can help the other, and vice versa. Interestingly, the Dirac operator has also been used in higher-order neural networks \citep{Battiloro2023}. 

Our main contribution is to define GPs on cellular complexes, which allows us to extend the modelling capability of graph GPs in the following novel ways:
\begin{itemize}
    \item Model quantities supported on the edges, and on higher-order `cells' such as volumes.
    \item Handles orientation on cells naturally. In particular, we can make directed predictions.
    \item Allows for joint modelling of data supported on vertices, edges and higher-order cells.
\end{itemize}

\section{Gaussian Processes}\label{sec:GPs}
A {\em Gaussian process} (GP) on a set $X$ is a random function $f : X \rightarrow \mathbb{R}$, such that for any finite collection of points $x_1, \ldots, x_N \in X$, the random vector $(f(x_1), \ldots, f(x_N)) \in \mathbb{R}^N$ is jointly Gaussian.
GPs are defined by a mean function $\mu : X \rightarrow \mathbb{R}$ and a kernel $\kappa : X \times X \rightarrow \mathbb{R}$, which satisfy $\mu(x) = \mathbb{E}[f(x)]$ and $\kappa(x, x') = \mathrm{Cov}[f(x), f(x')]$, respectively, for any $ x, x' \in X$.

Let $(\vec x, \vec y):=\{(x_i, y_i)\}_{i=1}^N$ be training data, where $y_i = f(x_i) + \epsilon_i$, $\epsilon_i \sim \mathcal{N}(0, \sigma^2)$. We can make inference about $\vec{f}_* := f(\vec{x}_*)$ at arbitrary test points $\vec{x}_*$ by computing the posterior predictive mean and covariance 
\begin{align}
    &\vec{\mu}_{\mat{f}_*|\mat{y}} = \vec{\mu}_{\mat{f}_*} + \mat{K}_{\mat{f}_*\mat{f}}(\mat{K}_{\mat{f}\mat{f}} + \sigma^2 \mat{I})^{-1}(\vec{y} - \vec{\mu}_{\mat{f}}) \label{eq:posterior-mean} \\
    &\mat{\Sigma}_{\mat{f}_*|\mat{y}} = \mat{K}_{\mat{f}_*\mat{f}_*} - \mat{K}_{\mat{f}_*\mat{f}}(\mat{K}_{\mat{f}\mat{f}} + \sigma^2 \mat{I})^{-1}\mat{K}_{\mat{f}\mat{f}_*}, \label{eq:posterior-cov}
\end{align}
respectively, where $\vec{\mu}_{\mat{f}} = \mu(\vec{x}), \vec{\mu}_{\mat{f}_*} = \mu(\vec{x}_*)$ denotes the mean and $\mat{K}_{\mat{f}\mat{f}} = \kappa(\vec{x}, \vec{x})$, $\mat{K}_{\mat{f}_*\mat{f}_*} = \kappa(\vec{x}_*, \vec{x}_*)$, $\mat{K}_{\mat{f}\mat{f}_*} = \kappa(\vec{x}, \vec{x}_*)$ denotes the covariance and cross-covariance of $f$ at $\vec x$ and $\vec x_*$.

\subsection{Gaussian processes on graphs}
While most existing GPs are defined on continuous domains, in this work, we are interested in graphs $\mathcal{G} = (V, E)$, where we take the input set $X$ to be the set of vertices $V$, and the edges $E$ to model the `proximity' between two vertices. As a mathematical object, a GP in this sense is identical to a multivariate Gaussian $\vec{f} \in \mathbb{R}^{|V|}$, whose indices are the vertices of the graph. However, the extra information provided by the edges allows one to impose a more rigid correlation structure, whereby two vertices are more strongly correlated if they are `closer' in the graph. Thus, this enables $\vec{f}$ to take on the characteristics of a continuous GP while being discrete.

A common way of encoding graph structures into multivariate Gaussians is by imposing specific sparsity patterns in the precision matrix, as seen in Gaussian Markov random fields \cite{rue2005gaussian}. For example, the Matérn GP on graphs \cite{Borovitskiy2021} is formally defined as $\vec{f} \sim \mathcal{N}(\vec{0}, (\frac{2\nu}{\ell^2} + \mat{\Delta})^{-\nu})$, where $\mat{\Delta}$ is the graph Laplacian matrix, a discrete analogue of the standard Laplacian operator on $\mathbb{R}^n$. If we observe the sparsity pattern of the corresponding precision matrix $(\frac{2\nu}{\ell^2} + \mat{\Delta})^{\nu}$ for $\nu \in \mathbb{N}$, we see that each vertex $x$ is connected to vertices that are within a radius $\nu$ of $x$ in the graph, controlling the `smoothness' of the process. More generally, we can impose a graph structure by penalising frequencies in the spectral domain of the kernel \cite{smola2003kernels}. However, graphs are limited in the data they can support. If the data contains polyadic interactions, then a more general structure than graphs is needed. 

\section{Modelling with Cellular Complexes}\label{sec:modelling-with-CCs}
Cellular complexes generalise graphs by incorporating higher-order interactions via `cells', extending the dyadic relation modelled by edges on a graph.
Concretely, a $k$-{\em cell} is a topological space that is homeomorphic to the unit disk $D^k := \{x \in \mathbb{R}^k : \|x\| < 1\}$. 
We will employ the standard notation $\partial$ to refers to the boundary of a topological space. A finite cellular complex $X$ of dimension $n < \infty$ is constructed iteratively, as follows \citep{Hatcher2001}:

{\bf (Step   $0$)} Start with a collection $X^0 = \{e^0_\alpha\}_{\alpha=1}^{N_0}$ of $0$-cells (i.e., points), called the $0$-skeleton.
\\[2mm]
{\bf (Steps $k=1, \ldots, n$)} Take a collection $\{e^k_\alpha\}_{\alpha=1}^{N_k}$ of $k$-cells and glue their boundaries to points in $X^{k-1}$ via a continuous {\em attaching map} $\phi_\alpha^k : \partial e^k_\alpha \rightarrow X^{k-1}$. The resulting space $X^k$ is the $k$-skeleton (see Figure \ref{fig:cell_complex_construction}).
\\[2mm]
{\bf (Step $n+1$)} Define the cellular complex to be the topological space $X := \cup_{k=0}^n X^k$.

\begin{figure}[t]
\centering
\includegraphics[width=0.32\textwidth]{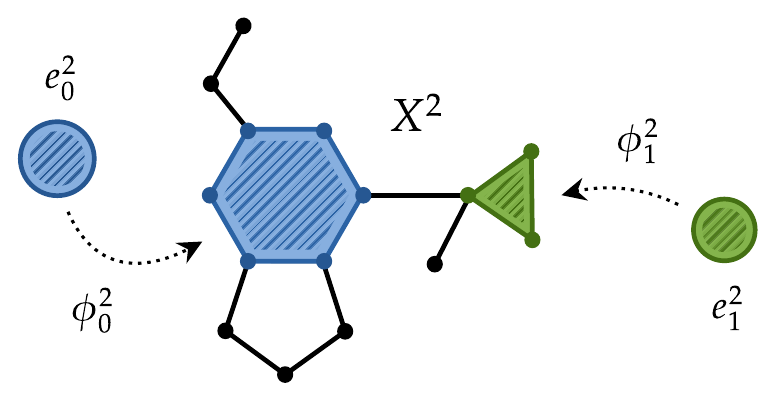}
\caption{ A cellular complex is constructed by attaching boundaries of $k$-cells $e_\alpha^k$ to the $(k-1)$-skeleton $X^{k-1}$ via a continuous map $\phi^k_\alpha$.}
\label{fig:cell_complex_construction}
\end{figure}

In the special case where the attaching maps $\phi_\alpha^k: \partial e^k_\alpha \rightarrow X^{k-1}$ are embeddings, we say that the cellular complex is \emph{regular}. This omits pathological cases, such as when boundaries of a $k$-cell collapse to a point or fold into itself. Hereafter, when we refer to cellular complexes, we will always assume that they are regular and finite. Further, when we fix an orientation on each cell (viewed as a topological manifold), we say that $X$ is \emph{oriented}, which we will also assume hereafter.

\begin{example}
A directed/undirected graph (see Figure \ref{fig:graph}) is a one-dimensional oriented/non-oriented cellular complex, where the vertices are the $0$-cells, the edges are the $1$-cells, and the attaching maps associate the two endpoints of an edge to a pair of vertices.
\end{example}

\begin{example}
Another important class are the simplicial complexes (see Figure \ref{fig:sc}), where the $k$-cells are taken to be the $k$-simplices and the attaching maps glue the boundaries 
of two $k$-simplices homeomorphically. Simplices are cells that contain all their sub-cells. 
\end{example}

\subsection{Chains and Cochains}
Chains and cochains are key concepts on cellular complexes that formalise the notion of paths and functions over $k$-cells, respectively. Given an $n$-dimensional cellular complex $X$, we define a {\em $k$-chain} $c_k$ for $0 \leq k \leq n$ as a formal sum of $k$-cells
\begin{align}
c_k = \sum_{\alpha=1}^{N_k} n_\alpha e^k_\alpha, \quad n_\alpha \in \mathbb{Z}.
\end{align}
Intuitively, this generalises the notion of directed paths on a graph, as we show in Figure \ref{fig:chains}. We denote the set of all $k$-chains on $X$ by $C_k(X)$, which has the algebraic structure of a free Abelian group\footnote{One may interpret this as a vector space with coefficients restricted to the integers.} with basis $\{e^k_\alpha\}_{\alpha=1}^{N_k}$.

\begin{figure}[t]
\centering
\includegraphics[width=0.3\textwidth]{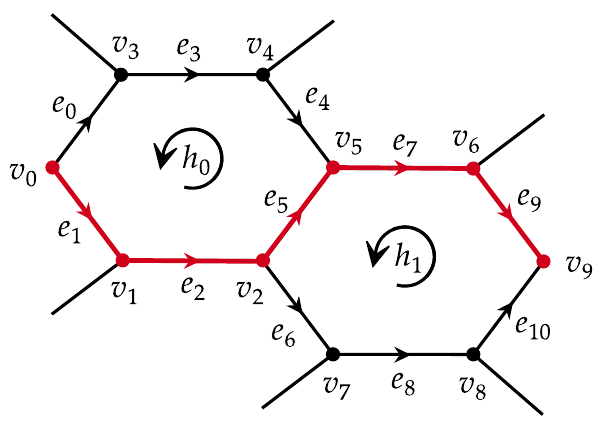}
\caption{ The path going from vertex $v_0$ to $v_9$ (in \textcolor{red}{red}) can be expressed as a $1$-chain $c = e_1 + e_2 + e_5 + e_7 + e_9$, while $-c$ represents the reverse path going from $v_9$ to $v_0$. Boundaries of cells can be expressed as chains, whose direction is consistent with the cell orientation. For example, $\partial h_0 = e_1 + e_2 + e_5 - e_4 - e_3 - e_0$. }
\label{fig:chains}
\end{figure}

For any cellular complex $X$, there is a canonical operation defined on chains called the {\em boundary operator} $\partial_k: C_k(X) \rightarrow C_{k-1}(X)$, associating the boundary of a $k$-chain to a $(k-1)$-chain. This is defined as a linear map\footnote{\label{homomorphism} A group homomorphism to be more precise.}
\begin{align}
    \partial_k \left(\sum_{\alpha=1}^{N_k} n_a e^k_\alpha\right) = \sum_{\alpha=1}^{N_k} n_a \partial e^k_\alpha,
\end{align}
where $\partial e^k_\alpha$ is the boundary of the $k$-cell $e^k_\alpha$, expressed as a $(k-1)$-chain. The cycle direction of $\partial e^k_\alpha$ must agree with the orientation of $e^k_\alpha$, as illustrated in Figure \ref{fig:chains} (see Appendix \ref{app:orientation} for more details).

Parallel to chains, there exist dual objects known as {\em cochains} on a cellular complex $X$, defined as follows.

\begin{definition}
A $k$-cochain on $X$ is a linear map$^2$ $f : C_k(X) \rightarrow \mathbb{R}$ assigning real numbers to $k$-chains, i.e.,
\begin{align} \label{eq:cochain-def}
    f\Big(\sum_{\alpha=1}^{N_k} n_\alpha e^k_\alpha\Big) = \sum_{\alpha=1}^{N_k} n_\alpha f(e_\alpha^k),
\end{align}
where $f(e_\alpha^k) \in \mathbb{R}$ is the value of $f$ at cell $e_\alpha^k$.
\end{definition}

The space of $k$-cochains, denoted $C^k(X)$ (with a superscript), forms a real vector space with dual basis $\{(e_\alpha^k)^*\}_{\alpha=1}^{N_k}$, satisfying $(e_\alpha^k)^*e_\beta^k = \delta_{\alpha\beta}$. The boundary operators on chains naturally induce analogous operators on the space of cochains, referred to as the {\em coboundary operators}, defined as follows.
\begin{definition}
    For $0 \leq k < n$, the coboundary operator $d_k : C^k(X) \rightarrow C^{k+1}(X)$ is defined via the relation
    \begin{align} \label{eq:coboundary-def}
        d_k f(c) = f(\partial_{k+1} c)
    \end{align}
    for all $f \in C^k(X)$ and $c \in C_{k+1}(X)$.
    To simplify our presentation, we also define $d_k f \equiv 0$ for $k \in \{-1, n\}$.
\end{definition}

\subsection{Hodge Laplacian and Dirac Operator}
We introduce a generalisation of the graph Laplacian to cellular complexes, which will be used in our construction of kernels.
For each $k$, let $\{w_\alpha^k\}_{\alpha=1}^{N_k}$ be a set of real, positive weights. Then, for any $f, g \in C^k(X)$, we define the weighted $L^2$-inner product as
\begin{align} \label{eq:k-cochain-inner-product}
    \left<f, g\right>_{L^2(\vec{w}^k)} := \sum_{\alpha=1}^{N_k} w_\alpha^k \,f(e^k_\alpha) \,g(e^k_\alpha).
\end{align}
The inner product induces an adjoint of the coboundary operator, denoted $d_k^* : C^{k+1}(X) \rightarrow C^{k}(X)$, i.e.,
\begin{align} \label{eq:codifferential-def}
    \left<d_k^* f, g\right>_{L^2(\vec{w}^{k})} = \left<f, d_{k}g\right>_{L^2(\vec{w}^{k+1})}
\end{align}
for any $f \in C^{k+1}(X)$ and $g \in C^{k}(X)$.
This leads to the following definition.
\begin{definition}\label{def:hodge-laplacian}
    The Hodge Laplacian $\Delta_k : C^k(X) \rightarrow C^k(X)$ on the space of $k$-cochains is defined as
    \begin{align}\label{eq:hodge-laplacian}
        \Delta_k := d_{k-1} \circ d_{k-1}^* + d_{k}^* \circ d_{k}.
    \end{align}
\end{definition}
We defer to Section \ref{sec:GPs-on-CCs} to see that this generalises the graph Laplacian by considering its numerical representation.

The {\em Dirac operator}
\begin{align} \label{eq:dirac-operator}
    \delta_k := d_{k-1}^* \oplus d_k : C^k(X) \rightarrow C^{k-1} \oplus C^{k+1}(X)
\end{align}
maps a $k$-cochain onto a direct sum of $(k-1)$ and $(k+1)$-cochains. One can verify that the relation
$
    \Delta_k = \delta_k^* \delta_k
$
holds. Hence, the Dirac operator is understood as a formal `square root' of the Hodge Laplacian.

\subsection{Numerical Representation}\label{sec:numerical-representation}
\begin{table}[t]
\begin{adjustbox}{width=\columnwidth,center}
\begin{tabular}{ c  c } 
     \toprule
     Object & Representation \\
     \midrule
     Chain $c = \sum_{\alpha=1}^{N_k} n_\alpha e^k_\alpha$ & $\vec{c} = (n_1, \ldots, n_{N_k})^\top$ \\
     & \\[-2mm]
     Cochain $f = \sum_{\alpha=1}^{N_k} f_\alpha (e^k_\alpha)^*$ & $\vec{f} = (f_1, \ldots, f_{N_k})^\top$ \\ & \\[-2mm]
     Boundary operator $\partial_k$ & $\mat{B}_k \in \mathbb{Z}^{N_{k-1} \times N_k}$ \\ & \\[-2mm]
     Coboundary operator $d_k$ & $\mat{B}_{k+1}^\top \in \mathbb{Z}^{N_{k+1} \times N_{k}}$ \\ & \\[-2mm]
     Inner-product $\left<f, g\right>_{L^2(\vec{w}^k)}$ & $\vec{f}^\top \mat{W}_k \vec{g}$ \\ & \\[-2mm]
     \multirow{2}{*}{Hodge Laplacian $\Delta_k$} & $\mathbf{\Delta}_k := \mathbf{B}_{k}^\top \mathbf{W}_{k-1}^{-1} \mathbf{B}_{k} \mathbf{W}_{k}$ \\ & $+ \mathbf{W}_{k}^{-1} \mathbf{B}_{k+1} \mathbf{W}_{k+1} \mathbf{B}_{k+1}^\top$ \\
     \bottomrule
\end{tabular}
\end{adjustbox}
\caption{ Numerical representation of key objects and operations defined on cellular complexes. Here, $\mathbf{W}_k = \mathtt{diag}(w_1^k, \ldots, w_{N_k}^k)$ is the matrix of cell-weights and $\mat{B}_k$ is the order-$k$ incidence matrix, whose $j$-th column corresponds to the vector representation of the cell boundary $\partial e^k_j$, viewed as a $(k-1)$-chain.}
\label{table:numerical-representation}
\end{table}
In Table \ref{table:numerical-representation}, we display how the objects considered above can be represented as matrices and vectors in order to make explicit computations with them. We refer to Appendix \ref{app:numerical-representation} for the full derivation.
In the case $k=0$ and taking $\mat{W}_0 = \mat{I}$ (i.e., no vertex weights), we see that the expression for the Hodge Laplacian reduces to the expression for the weighted graph Laplacian $\mathbf{\Delta}_0 = \mathbf{B}_{1} \mathbf{W}_1 \mathbf{B}_{1}^\top$. Furthermore, taking $\mat{W}_1 = \mat{I}$ (no edge weights), we obtain the expression for the standard graph Laplacian $\mathbf{\Delta}_0 = \mathbf{B}_{1} \mathbf{B}_{1}^\top$.

\section{Gaussian Processes on Cellular Complexes}\label{sec:GPs-on-CCs}
In this section, we establish our notion of Gaussian processes on cellular complexes, which we take to be a direct sum of {Gaussian random cochains}. Given a sample space $\Omega$, an event space $\mathcal{F}$ and a probability measure $\mathbb{P}$, we denote by $(\Omega, \mathcal{F}, \mathbb{P})$ the underlying probability triple for which our random variables will be defined over. 

\subsection{Gaussian Random Cochains}
\begin{definition}
    A random variable $f : \Omega \rightarrow C^k(X)$ is called a Gaussian random cochain if for any $k$-chain $c \in C_k(X)$, the random variable $f (c): \Omega \rightarrow \mathbb{R}$ is Gaussian. 
\end{definition}

To facilitate computations, we characterise them using a mean function and a kernel. Defining a mean is straightforward---this is just a fixed cochain. For the kernel, we consider the following definition.

\begin{definition}\label{def:kernel-on-chains}
    A kernel on $C_k(X)$ is a symmetric bilinear form\footnote{\label{note3}A group bi-homomorphism to be more precise.} $\kappa : C_k(X) \times C_k(X) \rightarrow \mathbb{R}$ such that for any $c_1, \ldots, c_m \in C_k(X)$, we have
    \begin{align}
        \sum_{i,j = 1}^m \kappa(c_i, c_j) \geq 0.
    \end{align}
\end{definition}
This is an appropriate notion of the kernel:
\begin{theorem}\label{eq:GRC-characterisation}
    A Gaussian random cochain $f : \Omega \rightarrow C^k(X)$ is fully characterised by a mean $\mu \in C^k(X)$ and a kernel $\kappa : C_k(X) \times C_k(X) \rightarrow \mathbb{R}$.
    \\[2mm]
    Proof: Appendix \ref{app:grc-characterisation}.
\end{theorem}

The vector representation $\vec{f}$ of $f$ is simply a multivariate Gaussian $\vec{f} \sim \mathcal{N}(\vec{\mu}, \mat{K})$, whose covariance
\begin{align}
    \mathbf{K} :=
    \begin{pmatrix}
    \kappa(e^k_1, e^k_1) & \cdots & \kappa(e^k_1, e^k_{N_k}) \\
    \vdots & \ddots & \vdots \\
    \kappa(e^k_{N_k}, e^k_1) & \cdots & \kappa(e^k_{N_k}, e^k_{N_k})
    \end{pmatrix}
\end{align}
is the matrix representation of the kernel. 
Due to the bilinearity of $\kappa$, for any $c, d \in C_k(X)$, and their vector representations $\vec{c}, \vec{d} \in \mathbb{Z}^{N_k}$, we can write
\begin{align}
    \kappa(c, d) = \vec{c}^\top \mat{K} \vec{d}.
\end{align}
By fixing orientations on each $k$-cell, we also have a notion of {\em direction} for the signal $f$ -- for a $k$-cell $e^k_\alpha$, depending on whether the sign of $f(e^k_\alpha)$ is positive or negative, the direction of $f$ at $e^k_\alpha$ is aligned to, or runs counter to the orientation of $e^k_\alpha$, respectively.

\subsection{GPs on Cellular Complexes}
Next, we extend our notion of Gaussian random cochains to direct sums of cochains of different orders. We take this as our definition of Gaussian processes on cellular complexes.

\begin{definition} \label{def:gp-on-cell-complex}
    Let $X$ be an $n$-dimensional cellular complex. We define a Gaussian process on $X$ as a random variable $f : \Omega \rightarrow \bigoplus_{k=0}^n C^k(X)$ such that for any $c = (c_0, \ldots, c_n) \in \bigoplus_{k=0}^n C_k(X)$, the random variable $f (c) : \Omega \rightarrow \mathbb{R}$ is univariate Gaussian.
\end{definition}

As before, we have an appropriate notion of a kernel on this space as a symmetric bilinear form\textsuperscript{\ref{note3}}
\begin{align} \label{eq:kernel-direct-sum-space}
\kappa : \bigoplus_{k=0}^n C_k(X) \times \bigoplus_{k=0}^n C_k(X) \rightarrow \mathbb{R}
\end{align}
satisfying $\sum_{i,j} \kappa(c_i, c_j) \geq 0$ for $c_i, c_j \in \bigoplus_{k=0}^n C_k(X)$. We also have the following result characterising GPs on cellular complexes via a mean and a kernel.

\begin{theorem}\label{eq:CCGP-characterisation}
    A GP on a cellular complex $X$ is fully characterised by a mean $\mu \in \bigoplus_{k=0}^n C^k(X)$ and a kernel $\kappa : \bigoplus_{k=0}^n C_k(X) \times \bigoplus_{k=0}^n C_k(X) \rightarrow \mathbb{R}$.
    \\[2mm]
    Proof: Appendix \ref{app:ccgp-characterisation}.
\end{theorem}

Again, we can represent \eqref{eq:kernel-direct-sum-space} by a matrix
\begin{align}
    \mathbf{K} =
    \begin{pmatrix}
    \mathbf{K}_{11} & \cdots & \mathbf{K}_{1n} \\
    \vdots & \ddots & \vdots \\
    \mathbf{K}_{n1} & \cdots & \mathbf{K}_{nn}
    \end{pmatrix}
\end{align}
with $[\mat{K}_{nm}]_{ij} = \kappa(e^n_i, e^m_j)$, which defines the covariance of $\vec{f} \in \mathbb{R}^{N_1 + \cdots + N_n}$, the vector representation of $f$.

\begin{remark}
    We emphasise that our notion of a GP on a cellular complex $X$ is not defined as a random function $X \rightarrow \mathbb{R}$ (i.e. a GP on the topological space $X$), but rather as a direct sum of Gaussian random cochains, i.e., a random function $\bigoplus_{k=0}^n C_k(X) \rightarrow \mathbb{R}$. This will allow us to define covariance structures between cells themselves, instead of between points on cells.
\end{remark}

\subsection{Kernels on Cellular Complexes}
We provide some concrete examples of kernels defining GPs on cellular complexes, which encompass existing kernels in the literature. For simplicity, we take unit cell-weights $w^k_\alpha = 1$ here and defer the treatment of the general case to Appendix \ref{app:arbitrary-cell-weights}. For any chain $c \in C_k(X)$, we also denote by $c^\flat \in C^k(X)$ the cochain defined by $f(c) = \big\langle f, c^\flat\big\rangle_{L^2(\vec{w}^k)}$ for any $f \in C^k(X)$.

\subsubsection{Matérn Kernel}\label{sec:matern-kernel}
We first consider a generalisation of the Matérn kernel on cellular complexes. Following \cite{borovitskiy2020matern, Borovitskiy2021}, consider the stochastic system
\begin{align}\label{eq:matern-gp-def}
    \left(\frac{2\nu}{\ell^2} + \Delta_k\right)^{\nu/2} f = \mathcal{W},
\end{align}
where $f \in C^k(X)$, $\Delta_k$ is the Hodge Laplacian (Definition \ref{def:hodge-laplacian}), and $\mathcal{W} : \Omega \rightarrow C^k(X)$ is a Gaussian random cochain satisfying $\mathbb{E}[\mathcal{W}(c_0)] = 0$ and $\mathbb{E}[\mathcal{W}(c_1) \mathcal{W}(c_2)] = \big\langle c^\flat_1, c^\flat_2\big\rangle_{L^2}$ for any $c_0, c_1, c_2 \in C_k(X)$.
The operator $\left(\frac{2\nu}{\ell^2} + \Delta_k\right)^{\nu/2}$ is defined rigorously in Appendix \ref{app:matern-kernel}.

The Matérn kernel $\kappa : C_k(X) \times C_k(X) \rightarrow \mathbb{R}$ is then defined as a solution to the system
\begin{align}\label{eq:matern-kernel-def}
    \left(\frac{2\nu}{\ell^2} + \Delta_k\right)^{\nu} \kappa(c, \,\cdot\,) = c^\flat, \quad \forall c \in C_k(X),
\end{align}
which is a kernel in the sense that it satisfies the following property.

\begin{proposition}\label{prop:matern-gp}
    The solution to \eqref{eq:matern-kernel-def} is related to the solution $f$ of the system \eqref{eq:matern-gp-def} by
    \begin{align}
        \kappa(c, c') = \mathbb{E}[f(c)f(c')], \quad 
    \forall c, c' \in C_k(X).
    \end{align}
    Thus, $\kappa$ solving \eqref{eq:matern-kernel-def} is the kernel of the Gaussian random cochain $f$. 
    \\[2mm] 
    Proof: Appendix \ref{app:matern-kernel-proof}.
\end{proposition}

We can express the kernel in \eqref{eq:matern-kernel-def} by a matrix
\begin{align}
    \mathbf{K} = \mathbf{U} \left(\frac{2\nu}{\ell^2}\mathbf{I} + \mathbf{\Lambda}^2\right)^{-\nu} \mathbf{U}^\top, \label{eq:matern-kernel-representation}
\end{align}
where $\mat{\Delta}_k = \mat{U} \mat{\Lambda}^2 \mat{U}^\top$ is the eigendecomposition of the Hodge Laplacian. In the case $k=0$, \eqref{eq:matern-kernel-representation} recovers exactly the graph Matérn kernel by \cite{Borovitskiy2021}.
We may also extend this construction to direct sums of cochains by replacing the Hodge Laplacian in  \eqref{eq:matern-gp-def} and \eqref{eq:matern-kernel-def} by the {\em super Laplacian} $\mathcal{L} := \bigoplus_{k=0}^n \Delta_k$, resulting in kernel \eqref{eq:matern-kernel-representation}, where $\mat{U} \mat{\Lambda}^2 \mat{U}^\top$ is now the eigendecomposition of the super Laplacian matrix $\vec{\mathcal{L}} = \mathtt{blockdiag}(\mathbf{\Delta}_0, \ldots, \mathbf{\Delta}_n)$. While this defines a GP on direct sums of cochains, due to the block-diagonal structure of the super-Laplacian, no mixing occurs between different cochains.

In the special case of GPs defined over the edges of a graph (i.e., Gaussian $1$-cochains), a similar construction of the Mat\'ern kernel has been explored in the concurrent work \citet{Yang2024}, where they also employ the Hodge decomposition to add further flexibility of the model.
Another related work \citep{pinder2021gaussian} explores the construction of Mat\'ern GPs on hypergraphs to model higher-order interactions. However, the work only considers inference on the  vertices, whereas our method focuses on making inferences on signals supported on the interactions themselves (i.e., the cells). 

\subsubsection{Reaction-diffusion Kernel}\label{sec:rd-kernel}
Next, we introduce a new type of kernel, which we term the {\em reaction-diffusion kernel}, that enables mixing of information between cochains of different orders. We will operate here entirely with the vector representation of cochains for ease of presentation.

Consider the {\em Dirac matrix}
\begin{equation} \label{eq:dirac-matrix}
\vec{\mathcal{D}} = 
    \begin{pmatrix}
        \mathbf{0} & \mathbf{B}_1 &  \cdots & \mathbf{0} \\
        \mathbf{B}_1^\top & \ddots & \ddots &  \vdots \\
        \vdots &  \ddots & \ddots & \mathbf{B}_n \\
        \mathbf{0} & \dots & \mathbf{B}_n^\top & \mathbf{0}
    \end{pmatrix},
\end{equation}
whose $k$-th column is a numerical representation of the $k$-th Dirac operator \eqref{eq:dirac-operator}. We can check that
$\vec{\mathcal{D}}^2 = 
\vec{\mathcal{L}}$ holds.
Thus, the Dirac matrix $\vec{\mathcal{D}}$ and the super-Laplacian matrix $\vec{\mathcal{L}}$ share a common eigenbasis $\mat{U}$ with eigenvalues $\mat{\Lambda}$ and $\mat{\Lambda}^2$, respectively. Now consider the stochastic system
\begin{align}\label{eq:reaction-diffusion-spde-representation}
    \left(r \mat{I} - c \vec{\mathcal{D}} +  d \vec{\mathcal{L}} \right)^{\frac{\nu}{2}} \vec{f} = \vec{w}, \quad \vec{w} \sim \mathcal{N}(0, \mat{I})
\end{align}
for some constants $r, c, d, \nu \in \mathbb{R}_{\geq 0}$. Then, the corresponding kernel matrix is given by
\begin{align} \label{eq:reaction-diffusion-kernel}
    \mat{K} = \mat{U} (r \mat{I} - c \mat{\Lambda} + d \mat{\Lambda}^2)^{-\nu} \mat{U}^\top,
\end{align}
which we term the {\em reaction-diffusion kernel}.\footnote{In general, $\mat K$ is  indefinite. To fix this, we set $\nu$ to be an even integer in \eqref{eq:reaction-diffusion-kernel}, which will make it positive definite for all $r, c, d \in \mathbb{R}_{\geq 0}$, excluding the set $\mathbb{X} := \bigcup_i \{(r, c, d) : r \pm c \lambda_i + d \lambda_i^2 = 0\}$. This set has Lebesgue measure zero. If $(r, c, d) \in \mathbb{X}$, $\mathbf{K}$ becomes positive {\em semi-definite}. Therefore it defines a degenerate Gaussian measure. Inference using degenerate Gaussian measures is, however, still valid, provided $\sigma>0$ is strictly positive in \eqref{eq:posterior-mean}--\eqref{eq:posterior-cov}.} 
This is a kernel in the sense that it satisfies the following result.
\begin{proposition}\label{prop:rd-gp}
    The kernel defined by \eqref{eq:reaction-diffusion-kernel} is related to the solution $\vec{f}$ of the system \eqref{eq:reaction-diffusion-spde-representation} by
    \begin{align}
        [\mat{K}]_{ij} = \mathbb{E}[f_i f_j], \quad 
    \forall i, j.
    \end{align}
    Proof: Appendix \ref{app:rd-kernel-proof}.
\end{proposition}

\begin{remark}
    Our naming of the reaction-diffusion kernel derives from the similarity of system \eqref{eq:reaction-diffusion-spde-representation} with the multi-component reaction-diffusion equation
    \begin{equation}
    \frac{\partial \vec{f}}{\partial t} = (r - c \vec{\mathcal{D}} + d \vec{\mathcal{L}}) \vec{f},
\end{equation}
where the first and third term model the reaction and diffusion of a quantity respectively, and the second term models the cross-diffusion of multiple quantities.
\end{remark}

\begin{figure}[t]
    \centering
    \begin{subfigure}[b]{0.2\textwidth}
        \centering
        \includegraphics[width=1.1\textwidth]{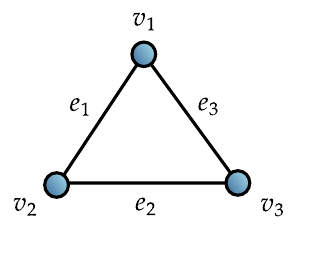}
        \caption[Network2]%
        {{ 1D cellular complex}}    
        \label{fig:pgm_graph}
    \end{subfigure}
    \hspace{1em}
    \begin{subfigure}[b]{0.22\textwidth} 
        \centering 
    \includegraphics[width=0.8\textwidth]{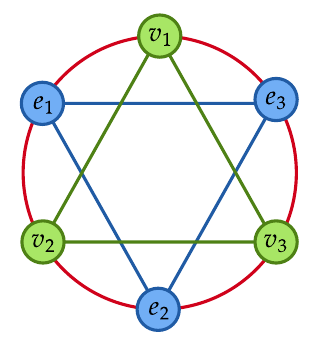}
        \caption[]%
        {{ PGM representation}}    
        \label{fig:pgm}
    \end{subfigure}
    \caption
    {Probabilistic graphical structure of the reaction-diffusion GP. 
    Interactions between vertices (\textcolor{OliveGreen}{green}) 
     and between edges (\textcolor{blue}{blue}) are shown as well as the mixing between cochains of different orders (\textcolor{red}{red}). The cellular Matérn kernel does not have this mixing property.}
    \label{fig:precision_of_rd_kernel}
\end{figure}

To interpret this kernel, we look at the corresponding precision matrix $\mat{P} = \mat{K}^{-1}$, which encodes the probabilistic graphical model (PGM) representation of the model $\vec{f}$, wherein variables $f_i$ and $f_j$ are linked if and only if $[\mat{P}]_{ij} \neq 0$ \cite{rue2005gaussian}. For simplicity, taking $\nu = 1$ and restricting to a 1-dimensional cellular complex, the precision takes the expression
\begin{align}
    \mat{P} =
    \begin{pmatrix}
    r \mat{I} + d \mat{\Delta}_0 & -c \mat{B}_1 \\
    -c \mat{B}_1^\top &r \mat{I} + d \mat{\Delta}_1
    \end{pmatrix}.
\end{align}
For $\vec{f} = (\vec{f}_0, \vec{f}_1)^\top$ in \eqref{eq:reaction-diffusion-spde-representation}, the graphical structure of the $k$-th component $\vec{f}_k$ is summarised by the matrix $r \mat{I} + d \mat{\Delta}_k$ for $k \in \{0, 1\}$, and the dependence between $\vec{f}_0$ and $\vec{f}_1$ is represented by the incidence matrix $\vec{B}_1$ (dotted red in \ref{fig:pgm}).
Since no communication between $\vec{f}_0$ and $\vec{f}_1$ occurs when $c=0$, the Dirac term is essential for allowing information to propagate between cochains of different orders. 

Let us now consider two special cases of the kernel \eqref{eq:reaction-diffusion-kernel}. In the first case, taking $r = 2\nu / \ell^2$, $c = 0$, $d = 1$ and $\nu = \nu$, we see that \eqref{eq:reaction-diffusion-kernel} recovers the Matérn kernel \eqref{eq:matern-kernel-representation}. Since $c=0$, there is no flow of information between cochains of different orders, resulting in independence between the random cochains $\vec{f}_0, \ldots, \vec{f}_n$.

For $r = m^2$, $c = 1$, $d = 0$ and $\nu = 2$, we obtain
\begin{align}
    \mat{K} &= \mat{U}(m \mat{I} - \mat{\Lambda})^{-2}\mat{U}^\top = (m \mat{I} - \vec{\mathcal{D}})^{-2}.
\end{align}
This kernel is considered by \cite{calmon2023dirac}  (in the form of a regulariser) for retrieving mixed topological signals supported on the $k$-cells for $k \leq 2$.

\section{Results}\label{sec:results}
In this section, we demonstrate the results of our GP model defined over cellular complexes (hereafter referred to as CC-GP) on two examples. First, we demonstrate that CC-GPs can make directed predictions on the edges of a graph by considering the problem of ocean current interpolation. In the second example, we investigate the effect of inter-signal mixing in the reaction-diffusion kernel.
We provide details of the experimental setups in Appendix \ref{app:experiments}.

\subsection{Directed Edge Prediction}\label{sec:ocean-experiment}
In the numerical simulation of fluids and especially in finite element methods (FEMs), it is common to treat vector fields as signals supported on the edges of a mesh \citep{arnold2006finite} to give them the flexibility for dealing with complex geometries, e.g., arising from the coastlines in ocean modelling. As there is an increasing adoption of FEM for weather and climate modelling (for example, the UK Met Office's GungHo model uses FEM with cubical elements \citep{staniforth2013gungho}), it is of interest to consider methods that propagate information from observations directly onto the finite element vertices and edges.

Our CC-GP model can naturally be applied in this setting: Consider the geostrophic current data from the \citet{noaa} database. First, we project the geostrophic current around the southern tip of Africa onto the oriented edges of a two-dimensional cubical complex by averaging the flow along each edge \citep{desbrun2006discrete}. This yields directed signals on the edges representing the vector field.  We then train our edge-based Matérn GP (see Section \ref{sec:matern-kernel} defined over a $1$-cochain) on $30\%$ of the data, selected randomly. The main objective of this experiment is to demonstrate that our approach can capture the directional information of the vector field, which would otherwise be difficult with existing approaches.

Figure \ref{fig:ocean} shows resulting predictions. For ease of visualisation, we display the magnitude of the predicted signals by colours on the edges; arrows indicate the predicted direction on each cell, computed by averaging the signals on its boundaries and then taking the resulting direction. We see that our CC-Matérn GP on edges captures the general characteristics of the ground-truth vector field with similar magnitudes and directions, indicating that it can correctly diffuse information onto neighbouring edges.

\begin{figure}[t]
    \centering
    \includegraphics[width=0.45
\textwidth]{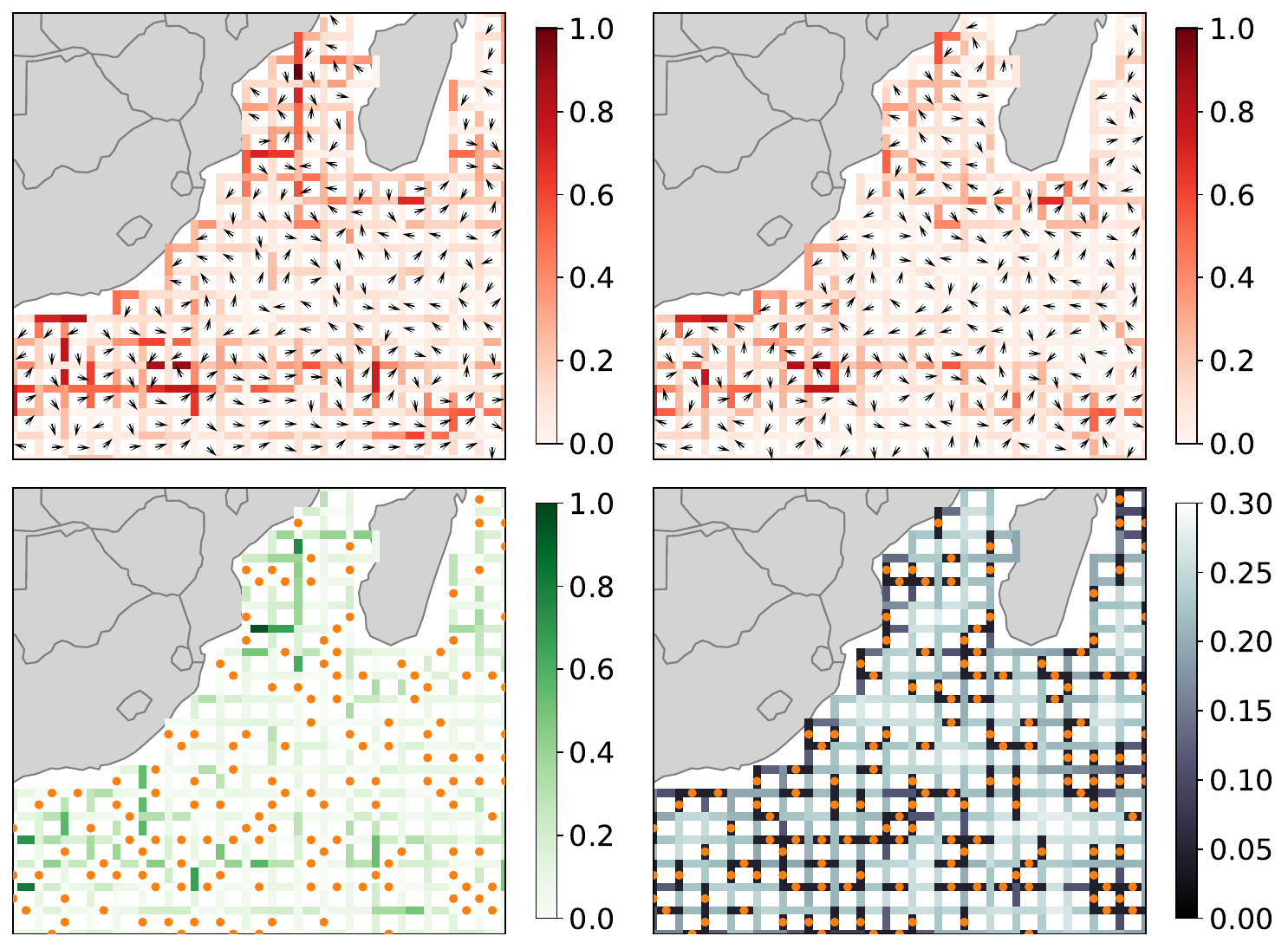}
    \caption{Prediction of geostrophic current around the Southern tip of Africa using the CC-Matérn GP on edges. (Top left) Ground truth. (Top right) Predicted mean. (Bottom left) Absolute error. (Bottom right) Standard deviation. Orange dots are observed edges.}
    \label{fig:ocean}
\end{figure}

We compare the results with a graph Matérn GP baseline defined over the corresponding line graph\footnote{This is the graph constructed by treating the edges as vertices and connecting them if they share a vertex.}. Since this baseline cannot infer directions on the edge signals, we use it to only predict its magnitude.
The results are shown in Table~\ref{table:ocean-baseline-comparison}, where we report the mean square error (MSE) and negative log-likelihood (NLL) scores on the magnitude of the ocean current. MSE results for both models are comparable; however, our edge-based Matérn GP performs better than the graph Matérn GP on the NLL. This suggests that in addition to being able to infer the directions on edges, predictive uncertainties are better and the topological inductive bias contained in the edge-based Matérn GP also helps to improve predictions for the magnitudes.

\begin{table}[t]
\begin{adjustbox}{width=\columnwidth,center}
\begin{tabular}{l  c c } 
     \toprule
      & MSE ($\downarrow$) & NLL ($\downarrow$) \\
     \midrule
     Graph Matérn & ${\bf 0.030 \pm 0.000}$
     & $-684.54 \pm 4.20$ \\
     Edge Matérn (ours) & ${\bf 0.029 \pm 0.001}$ & ${\bf -703.42 \pm 5.10}$ \\
     \bottomrule
\end{tabular}
\end{adjustbox}
\caption{Mean square error (MSE) and negative log-likelihood (NLL) of ocean current magnitude predictions using (a) a graph Matérn baseline, and (b) the edge Matérn GP. Mean and standard error are shown.}
\label{table:ocean-baseline-comparison}
\end{table}

\begin{figure*}[ht]
    \centering
    \begin{subfigure}[b]{0.325\hsize}
        \centering
        \includegraphics[width=0.49\hsize]{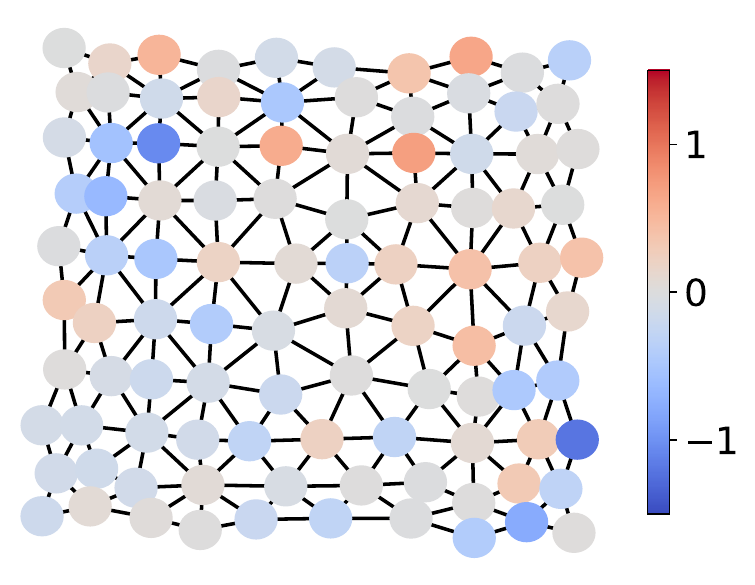}
        \includegraphics[width=0.49\hsize]{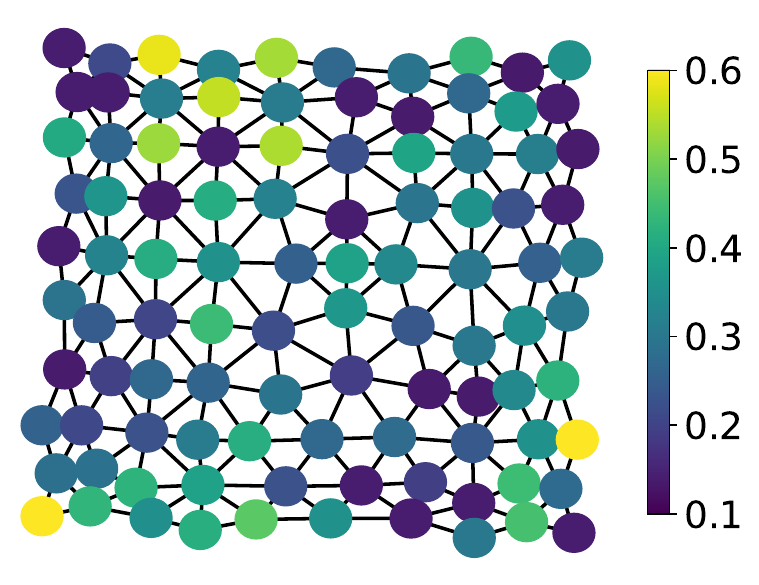}
        \caption{Vertex prediction, RD}  
        \label{fig:fields_rd_vertex}
    \end{subfigure}
    \hfill
        \begin{subfigure}[b]{0.325\hsize}
        \centering
        \includegraphics[width=0.49\hsize]{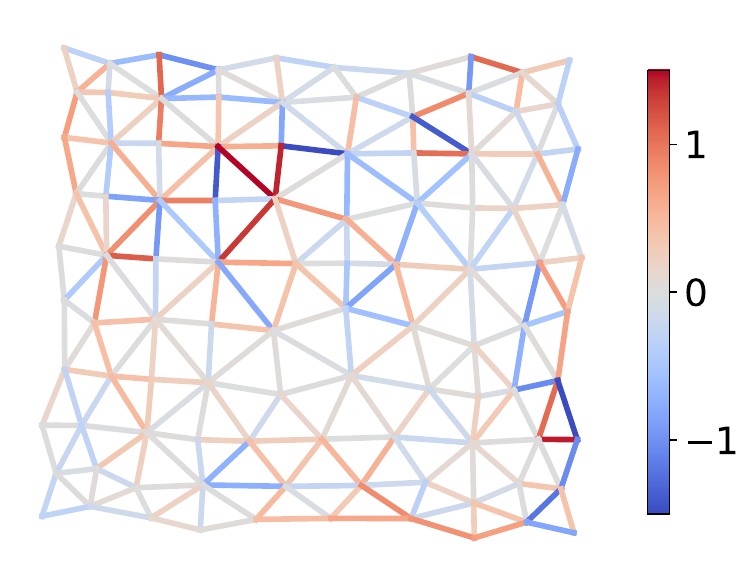}
        \includegraphics[width=0.49\hsize]{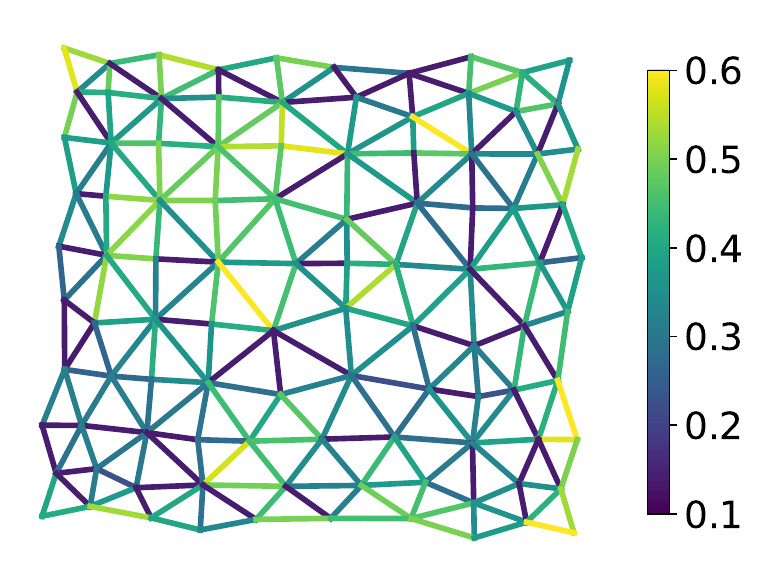}
        \caption{Edge prediction, RD}  
        \label{fig:fields_rd_edge}
    \end{subfigure}
    \hfill
       \begin{subfigure}[b]{0.325\hsize}
        \centering
        \includegraphics[width=0.49\hsize]{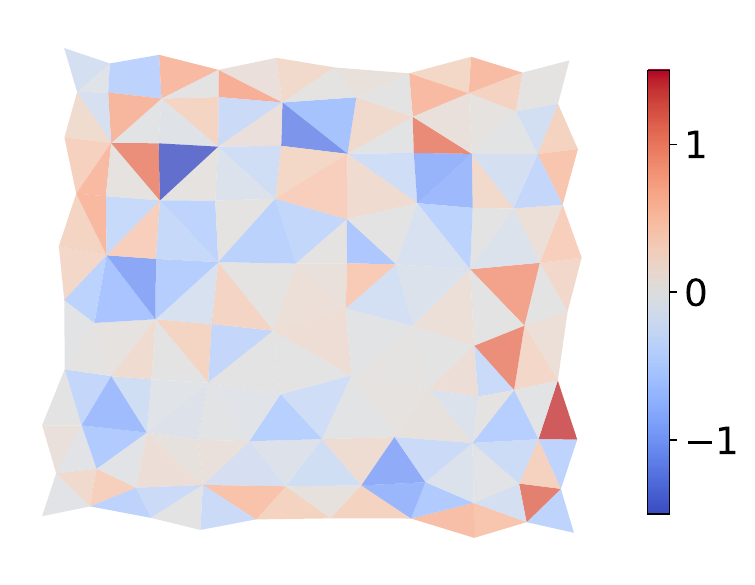}
        \includegraphics[width=0.49\hsize]{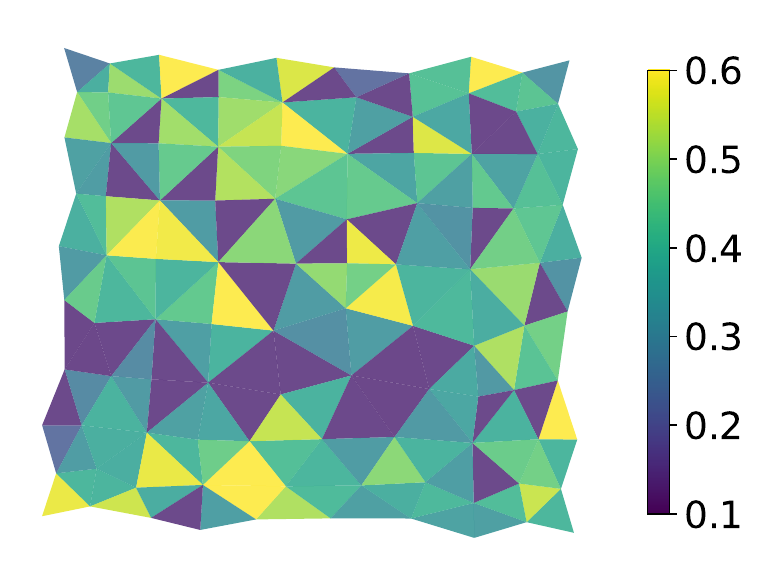}
        \caption{Triangle prediction, RD}  
        \label{fig:fields_rd_triangle}
    \end{subfigure}
\\
    \begin{subfigure}[b]{0.325\hsize}
        \centering
        \includegraphics[width=0.49\hsize]{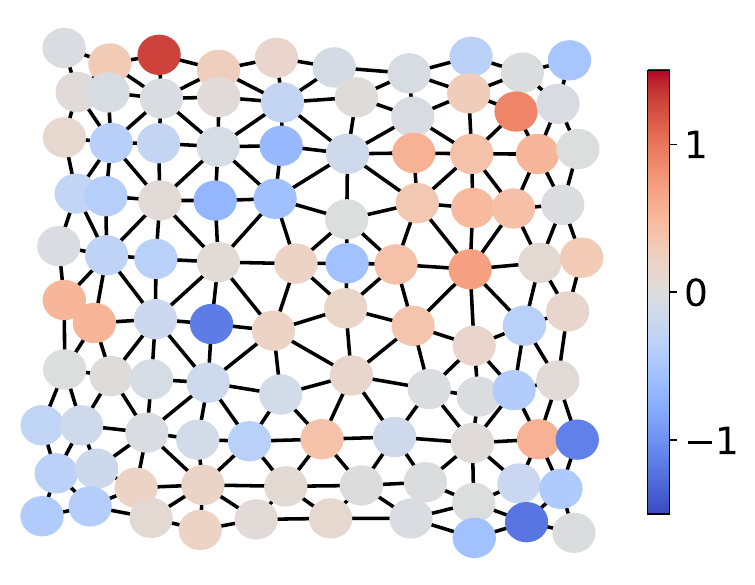}
        \includegraphics[width=0.49\hsize]{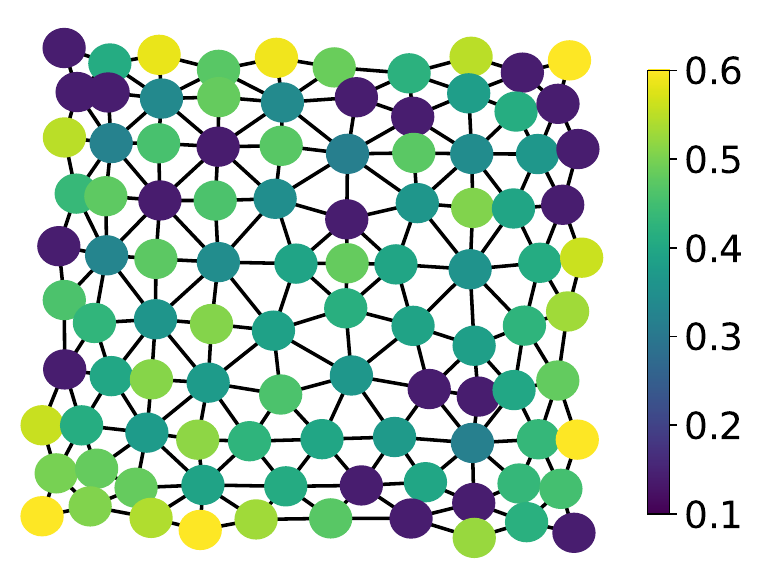}
        \caption{Vertex prediction, CC-Mat\'ern}  
        \label{fig:fields_matern_vertex}
    \end{subfigure}
    \hfill
        \begin{subfigure}[b]{0.325\hsize}
        \centering
        \includegraphics[width=0.49\hsize]{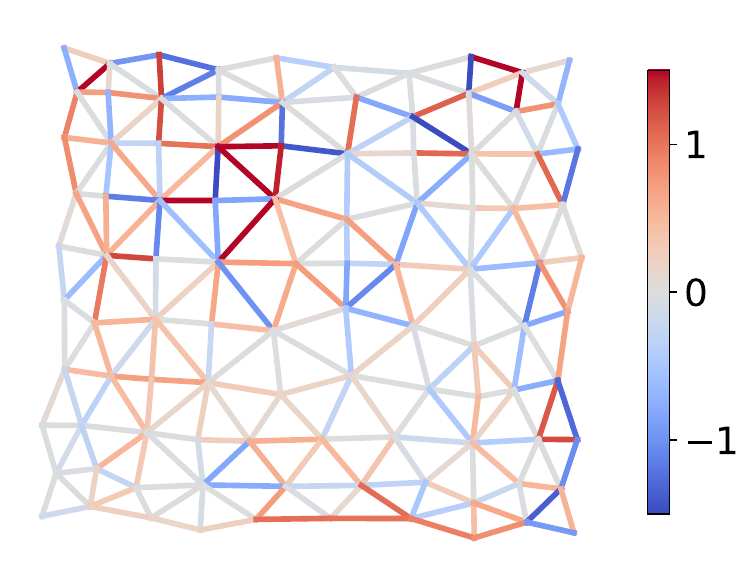}
        \includegraphics[width=0.49\hsize]{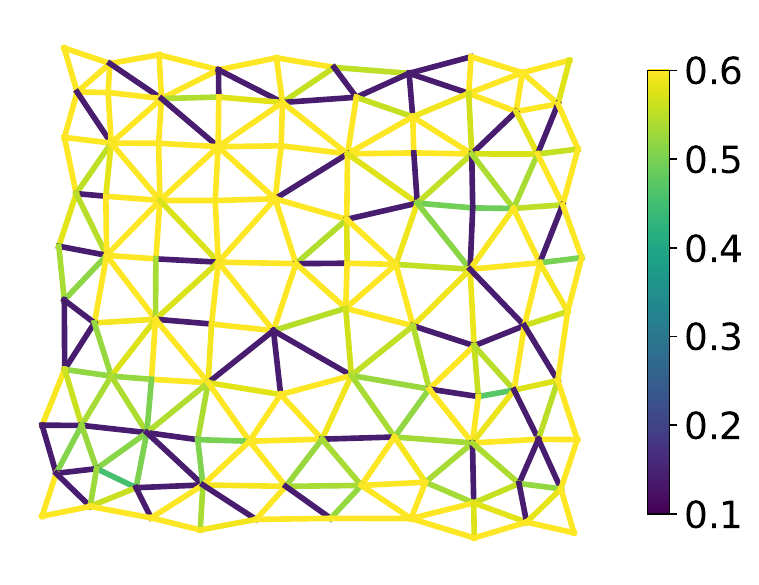}
        \caption{Edge prediction, CC-Mat\'ern}  
        \label{fig:fields_matern_edge}
    \end{subfigure}
    \hfill
       \begin{subfigure}[b]{0.325\hsize}
        \centering
        \includegraphics[width=0.49\hsize]{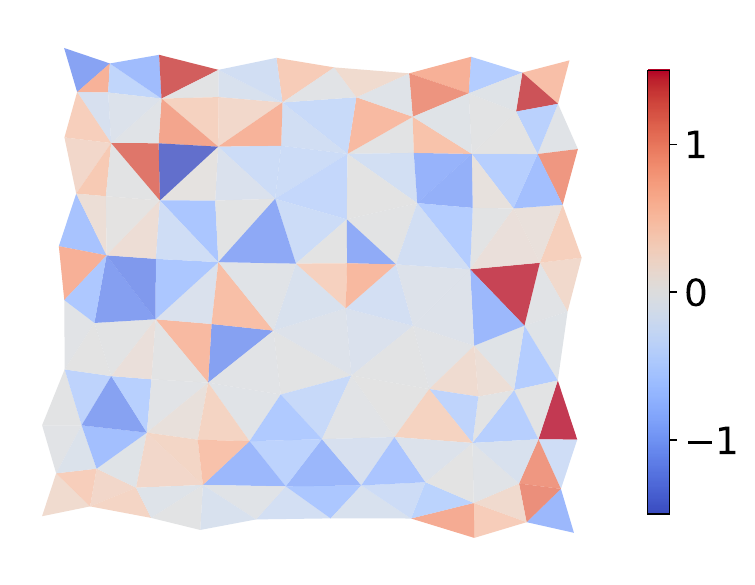}
        \includegraphics[width=0.49\hsize]{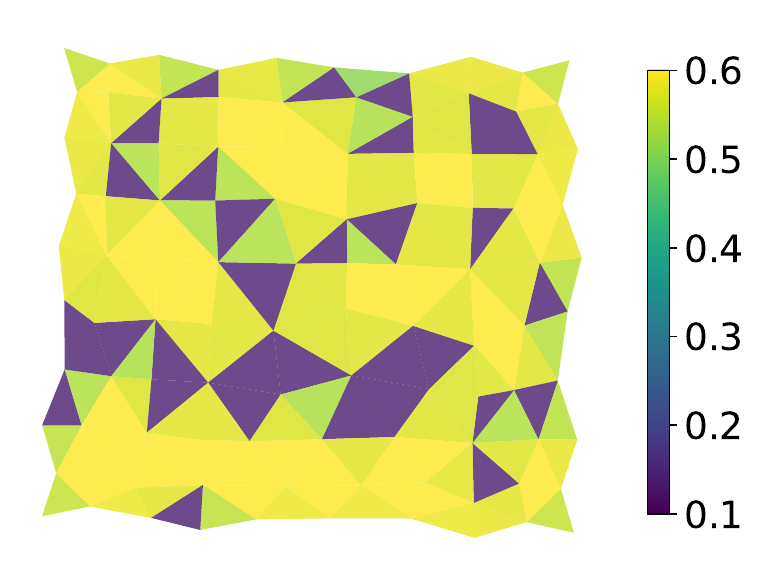}
        \caption{Triangle prediction, CC-Mat\'ern}  
        \label{fig:fields_matern_triangle}
    \end{subfigure}
    \caption{ Predictive distributions on vertices, edges, and triangles using the reaction-diffusion kernel (RD, top row) and the CC-Matérn kernel (bottom row). Left panels: differences between mean prediction and ground truth (values close to 0 are good); right panels: corresponding predictive standard deviations. By taking correlation properties into account, the RD kernel produces better predictions on vertices, edges, and triangles.}
    \label{fig:fields}
\end{figure*}

\begin{table}[t]
    \centering
    \begin{tabular}{l r r} 
       {MSE ($\downarrow$)}  & CC-Mat\'ern & Reaction-diffusion\\
       \midrule
       Vertices & 0.165 $\pm$ 0.005  & {\bf 0.076 $\pm$ 0.004} \\ 
     Edges & 0.335 $\pm$ 0.014 & {\bf 0.200 $\pm$ 0.010} \\
     Triangles & 0.272 $\pm$ 0.005 & {\bf 0.166 $\pm$ 0.005} \\
     \hline
     \rule{0pt}{3ex} NLL ($\downarrow$)
        \\
     \midrule
     Vertices &  28.81 $\pm$ 1.75 & {\bf -9.31 $\pm$ 1.61} \\ 
     Edges & 136.77 $\pm$ 4.32 & {\bf 71.07 $\pm$ 6.29} \\
     Triangles & 82.84 $\pm$ 1.56 & {\bf 39.78 $\pm$ 2.47} \\
    \end{tabular}
    \caption{Mean square error (MSE) and negative log-likelihood (NLL) of predictions on the synthetic data (mean and standard error across 20 random seeds). Overall, the performance of the reaction-diffusion GP is superior to that of the Matérn GP on the cellular complex, highlighting the benefits of mixing information across different cell types.}
    \label{fig:synthetic_results_table}
\end{table}
\subsection{Signal Mixing}\label{sec:synthetic-exp}
We illustrate the benefits of mixing signals using the reaction-diffusion kernel on synthetic data, where we constructed a 2D-simplicial mesh consisting of $10 \times 10$ vertices. We generated artificial signals on the edges by considering a random $1$-cochain $f = \sum_{i=k}^K \xi_i u_i$, where $\xi_i \sim \mathcal{N}(0, \lambda_i^{-1})$, $\{(\lambda_i, u_i)\}_{i}$ are the eigenpairs of the Hodge Laplacian $\Delta_1$ and $0 < k < K$ are the minimal and maximal wavenumbers controlling the smoothness of the edge field.
We then generated data supported on the triangles and vertices of the mesh by applying the coboundary operator $d_1$ and its adjoint $d_1^*$ respectively to $f$ (numerically, this corresponds to applying the matrices $\mat{B}_2^\top$ and $\mat{B}_1$).
We randomly selected a third of the data supported on each type of cell (vertices, edges and triangles) and perturbed them by noise for training. The aim is to recover the signals on the remaining two thirds of the cells.

We compare the results of the Matérn GP on the cellular complex (CC-Matérn) and the reaction-diffusion (RD) GP, used to make predictions across various cell-types. Figure~\ref{fig:fields} shows an example result with $k = 20$ and $K = 100$, where we display the difference (left panels) from the ground truth and the predictive standard deviation (right panels) for vertex (left column), edge (center column), and triangle predictions (right column) for the RD kernel (top row) and the CC-Mat\'ern kernel (bottom row). Overall, we see how the mixing of information across different cell types in the RD kernel helps to improve predictions, as evidenced by the overall smaller errors in each cell type (with small predictive uncertainty), compared to the predictions made by the CC-Matérn kernel. In particular, we see how data supported on one type of cell can be used to enhance the predictions on another cell type as we can infer from the standard deviation plots. The standard deviation of the CC-Matérn GP is more localised around the data points than that of the reaction-diffusion GP, suggesting that information on different cell types is being mixed in the latter.

This behaviour is also quantitatively reflected in Table~\ref{fig:synthetic_results_table}, which reports the mean squared error (MSE) and the negative log-likelihood (NLL) scores of the predictions for both models, computed across 20 random seeds (mean and standard errors). We see how the results for the RD-GP are on average significantly better than that for the CC-Matérn kernel on both metrics, highlighting the benefits of mixing for both prediction and uncertainty quantification.

\subsection{Modelling Electromagnetic Fields}\label{sec:em-experiment}
Another potential application area of GPs defined over CCs is to model electromagnetic fields, which have natural representations as cochains. In particular, the scalar potential, electric field and magnetic field generated by point charges on a 2D plane can be modelled geometrically as scalar, vector and pseudovector fields, respectively. Upon discretisation, these can then be represented by $0, 1$ and $2$-cochains, respectively, using the continuous-discrete correspondence between tensor fields (more precisely, differential forms) and cochains \cite{desbrun2006discrete}.

In this experiment, we use simulations of the scalar potential ($V$), electric field ($E$) and magnetic field ($B$) generated by $10$ randomly sampled point charges. These are then projected onto $0, 1$ and $2$-cochains, respectively, from which we extract noisy observations at a sixth of randomly selected cells in each skeleton of the generated CC. Then, similar to the experiment in Section \ref{sec:synthetic-exp}, the objective is to recover the signals on the remaining cells. In contrast to the previous experiment, the correlation structure between signals supported on different cell types are more complex, provided indirectly through Maxwell's equations. Our goal is thus to see whether the RD kernel is still useful in this setting where the correlation between fields exist, but are not artificially imposed, as with the previous experiment.
The electromagnetic simulation was performed using the Python package \texttt{PyCharge} \cite{filipovich2022pycharge}.

In Table \ref{tab:em-experiment}, we compare the results of the CC-Mat\'ern GP and the reaction-diffusion (RD) GP on this experiment. Generally, we observe that the RD-GP yields slightly better results than the CC-Mat\'ern GP on both the MSE and the NLL, with the exception of the MSE on the scalar potential. For predictions of the electric and magnetic fields, the RD kernel outperformed the Mat\'ern kernel on eight out of ten random seeds that were used to generate the results in Table \ref{tab:em-experiment}. Our results indicate that interestingly, mixing of information between the different cell types using the RD kernel can still be useful in this setting. However, we also note that the improvements that we see here are much less pronounced than what we observed in the previous experiment, where correlations between signals on different cell types were imposed more directly.

\begin{table}
\centering
\begin{adjustbox}{width=\columnwidth}
\begin{tabular}{l | r r | rr} 
   & Matérn MSE & RD MSE & Matérn NLL & RD NLL
    \\
 \midrule
 V &  {\bf 0.102 $\pm$ 0.008} & 0.11 $\pm$ 0.011 & 45.75 $\pm$ 7.20 & {\bf 45.41 $\pm$ 7.35} \\ 
 E & 0.130 $\pm$ 0.008 &  {\bf 0.128 $\pm$ 0.009} & 104.3 $\pm$ 14.0 & {\bf 100.7 $\pm$ 15.8} \\
 B & 0.190 $\pm$ 0.026 & {\bf 0.178 $\pm$ 0.023} & 124.7 $\pm$ 18.8 & {\bf 117.3 $\pm$ 17.2} \\
\end{tabular}
\end{adjustbox}
\caption{Comparison of independent Matérn GPs and the reaction-diffusion (RD) GP for joint modelling of the 2D potential (V), electric (E) and magnetic (B) fields. We compare the mean square error (MSE) and negative log-likelihood (NLL).
For the E and B fields, RD ourperforms independent Matérn on 8/10 seeds.}
\label{tab:em-experiment}
\end{table}

\section{Discussion}\label{sec:discussion}
Our experiments demonstrate the benefits of incorporating the structure of cellular complexes to model data that are naturally supported on higher-order networks, enabling us to treat (a) directed signals and (b) mixed signals with ease, which a standard graph GP cannot handle. This opens up new possibilities for modelling data on non-Euclidean domains. For instance, as a promising direction, we believe that our approach could be useful for modelling vectorial or higher-order tensorial quantities supported on arbitrary manifolds, by relying on the structural parallels between CCs and differential forms, typically employed in numerical simulations of vectorial/tensorial quantities supported on manifolds \cite{arnold2006finite}. This idea is explored in our experiments in Sections \ref{sec:ocean-experiment} and \ref{sec:em-experiment} on simple 2D domains. We aim to extend this to more complex domains in future work, possibly incorporating boundary conditions.

Indeed, vectorial GPs on manifolds have been considered before for example in \cite{hutchinson2021vector} and more recently in \cite{robert2023intrinsic}. However, the former construction relies on embedding the manifold in a larger ambient space in order to make use of scalar kernels, and the latter relies on the Helmholtz decomposition to make use of scalar kernels. On the other hand, by directly encoding the vector information on the edges of a CC, one can easily model vector fields over arbitrary manifolds, possibly with boundaries (see Figure~\ref{fig:simplicial-vector-projection}). Adopting a CC perspective thus makes modelling with vectorial GPs easier, and can be further generalised to the tensorial setting naturally. Further, the reaction-diffusion GP provides a topologically consistent extension to multitask GPs to the CC setting, where the individual GPs can now live on different skeletons of the same CC. This will be useful in situations for modelling multiple fields (cochains) that are correlated.

We note that at present, our reaction-diffusion kernel is controlled only by three hyperparameters $r, c, d$, which may be too restrictive for some modelling purposes.
To make our model more flexible, one possible direction that we can consider is to make use of the Hodge decomposition to split the cochains into exact, co-exact and harmonic components. This will allow us to model each of these components separately, for added flexibility. In a recent work \cite{Yang2024}, the authors consider GPs to model edge signals on a graph and observed that modelling the exact, co-exact and harmonic signals separately can lead to improved results. For signals supported on a cellular complex, a natural analogue of the Hodge decomposition can be provided through the Dirac matrix \cite{calmon2023dirac}. Hence, a promising direction would be to combine this decomposition with our reaction-diffusion kernel to develop a more flexible extension that could fit more complex mixing of information between different cell-types.

Another current limitation of our model is the computational cost associated with computing the eigendecomposition of the Laplacian/Dirac matrix, which can grow very quickly with the size of the network. While this limits the size of networks we can work with, this expensive computation needs to be performed only once and not during training, so we expect one to be able to work with reasonably large networks, consisting of hundreds of thousands of cells.

\begin{figure}
\centering
    \centering
    \includegraphics[height=3cm]{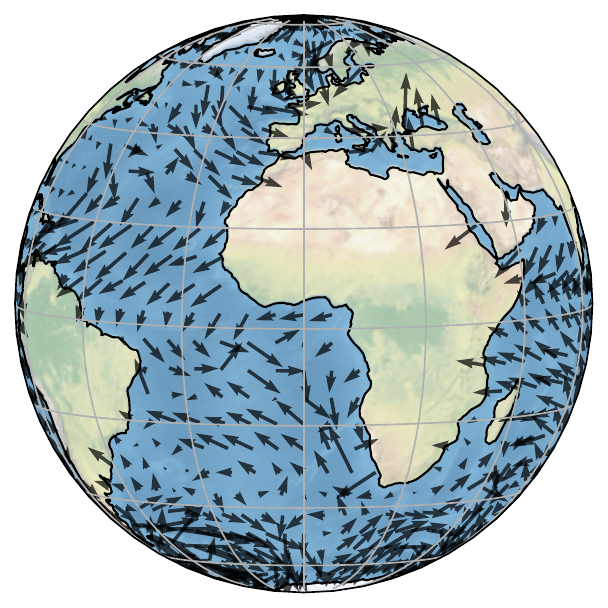}
    \hspace{5mm}
    \includegraphics[height=3cm]{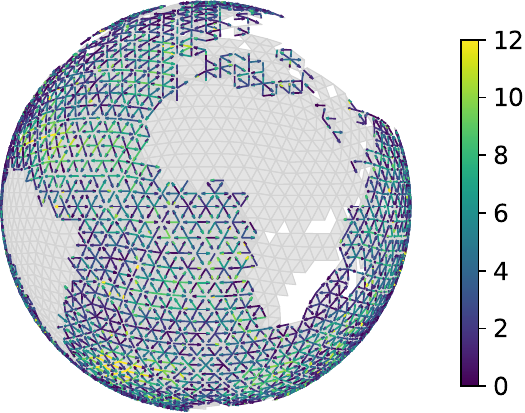}
    \caption{Left: Vector field on a manifold-with-boundaries; Right: Vector field's encoding on the edges of an oriented CC. Modelling such fields is straightforward with CC-GPs.}
    \label{fig:simplicial-vector-projection}
\end{figure}

\section{Conclusion}
\label{sec:conclusion}
We introduced Gaussian processes on cellular complexes as tools for probabilistic modelling on higher-order networks. We identify these as GPs defined over random chains or direct sums thereof, enabling inference on vertices, edges, and higher-order cells. We constructed kernels appropriate for practical modelling. In particular, we generalise the Matérn kernel to cellular complexes and propose the reaction-diffusion kernel, which allows for propagation of information between cells of different orders.

\section*{Acknowledgements}
MA is supported by a Mathematical Sciences Doctoral Training Partnership held by Prof. Helen Wilson, funded by the Engineering and Physical Sciences Research Council (EPSRC), under Project Reference EP/W523835/1. ST is supported by a Department of Defense Vannevar Bush Faculty Fellowship held by Prof. Andrew Stuart, and by the SciAI Center, funded by the Office of Naval Research (ONR), under Grant Number N00014-23-1-2729.

\section*{Impact Statement}
This paper presents work whose goal is to advance the field of machine learning. However, due to its theoretical nature, we do not identify any potential societal consequences of our work.

\bibliography{references}
\bibliographystyle{icml2024}

\newpage
\appendix
\onecolumn
\section{Cell Orientation and Boundaries}\label{app:orientation}
In this appendix, we provide more details on the computation of cell boundaries.
To this end, we first require a notion of orientation on cells. For low dimensional $k$-cells, we have the following intuitive definitions of what we mean by an orientation:
\begin{itemize}
    \item $k=0$ (i.e., a point): A choice of either ``$+$" or ``$-$".
    \item $k=1$ (i.e., a line segment): A choice of direction from one endpoint to the other.
    \item $k=2$ (i.e., a 2D-disk): A choice of rotation direction (clockwise or anticlockwise).
\end{itemize}
Generally, the orientation of a $k$-dimensional unit disk $D^k$ is the choice of a continuously varying unit normal vector field $\hat{\vec{n}} : D^k \rightarrow \mathbb{R}^{k+1}$ on $D^k$, viewed as a surface embedded in $\mathbb{R}^{k+1}$. Here, a unit normal vector $\hat{\vec{n}}(s)$ for $s \in D^k$ is a vector in $\mathbb{R}^{k+1}$ such that $\|\hat{\vec{n}}(s)\| = 1$ and $\hat{\vec{n}}(s) \cdot \vec{v} = 0$ for any $\vec{v} \in T_sD^k \subset \mathbb{R}^{k+1}$, where $T_sD^k$ is the tangent space of $D^k$ at point $s$ (i.e., the space of all vectors in $\mathbb{R}^{k+1}$ that are tangential to $D^k$). Since orientation is a topological invariant, we can define the orientation on a $k$-cell generally to be the orientation of the $k$-dimensional disk that it is homeomorphic to.

A useful way of thinking about the orientation of $D^k$ is in terms of how it is embedded in $\mathbb{R}^{k+1}$. In particular, we may identify a point $x \in D^k$ as a vector $(x_1, \ldots, x_k, 0)^\top \in \mathbb{R}^{k+1}$ such that $\sqrt{\sum_{i=1}^k x_i^2} < 1$. We can thus choose the unit normal field to be given by $\hat{\vec{n}} = (x_1, \ldots, x_k, 1)^\top \in \mathbb{R}^{k+1}$ for any $x_1, \ldots, x_k$ parameterising $D^k$. From this perspective, we can define an {\em induced orientation} $\partial \hat{\vec{n}}$ on its boundary $\partial D^k \cong \mathbb{S}^{k-1}$, by taking the unit normal field pointing outwards from the disk, with respect to this embedding.

\begin{figure}[h]
    \centering
    \begin{subfigure}[b]{0.45\hsize}
    \centering
    \includegraphics[width=0.6\hsize]{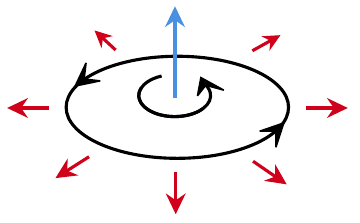}
    \caption{Orientation and induced orientation on a $2$-cell}
    \label{fig:2-cell-orientation}
    \end{subfigure}
    \begin{subfigure}[b]{0.45\hsize}
    \centering
    \includegraphics[width=0.8\hsize]{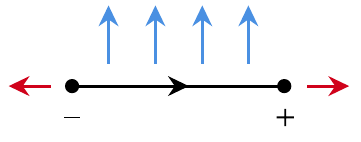}
    \caption{Orientation and induced orientation on a $1$-cell}
    \label{fig:1-cell-orientation}
    \end{subfigure}
    \caption{Illustration of orientations (\textcolor{blue}{blue} arrows) on (a) a $2$-cell, and (b) a $1$-cell, along with the induced orientation (\textcolor{red}{red} arrows) on the corresponding boundaries. The orientation can be understood as a choice of a continuous vector field in the ambient Euclidean space that points in the direction normal to the $k$-cell (in this case, pointing upwards). The induced orientation on the boundaries can be understood as the outward pointing normal field with respect to the embedding.}
    \label{fig:}
\end{figure}
    
\begin{example}
    Consider a 2D-disk $D^2$ embedded in $\mathbb{R}^3$. The orientation on $D^2$ is then given by the unit normal field $\hat{\vec{n}}(x,y,0) = (x,y,1)$ for $\sqrt{x^2 + y^2} < 1$ and the induced orientation on its boundary $\mathbb{S}^1 \cong \partial D^2$ is given by the outward unit normal field $\partial \hat{\vec{n}}(x,y,0) = (x,y,0)$ for $\sqrt{x^2 + y^2} = 1$. Intuitively, the former can be thought of as anticlockwise rotation of the disk, deduced by aligning one's thumb with the unit normal direction and applying the right-hand rule. Now, aligning the thumb with $\hat{\vec{n}}$ and the index finger with $\partial \hat{\vec{n}}$, the middle finger, according to the right-hand rule, points in the anticlockwise direction around the boundary, giving us a more intuitive interpretation of the induced orientation (see Figure \ref{fig:2-cell-orientation}).
\end{example}

\begin{example}
    Consider a line-segment (i.e., a 1D disk) embedded in $\mathbb{R}^2$ with respect to the parameterisation $\{(x, 0) \in \mathbb{R}^2 : x \in (-1, 1)\}$. The orientation of the line segment is determined by the unit normal field $\hat{\vec{n}}(x,0) = (x,1)$ for $x \in (-1,1)$ and the induced orientation is given by the outward normal field $\partial \hat{\vec{n}}(x,0) = (x, 0)$ for $x \in \{-1,1\}$. Again, we can use the right-hand rule to determine the ``intuitive" interpretation of orientations here: pointing the thumb towards the page and the index finger aligned with $\hat{\vec{n}}$, the middle finger points in the direction going from left endpoint $x = -1$ to the right endpoint $x=1$. Since the middle finger as a result aligns with $\partial \hat{\vec{n}}(x,0)$ at $x=1$, we think of the induced orientation at $(1,0)$ as having the sign $``+"$. However, it does not align with $\partial \hat{\vec{n}}(x,0)$ at $x=-1$, hence the induced orientation at $(-1,0)$ has the sign $``-"$ (see Figure \ref{fig:1-cell-orientation}).
\end{example}

Next, the {\em boundary} of a $k$-cell $e_\alpha^k$, is defined generally as a $(k-1)$-chain
\begin{align}
    \partial e_\alpha^k = \sum_{\beta=1}^{N_{k-1}} \mathrm{deg}(\chi^{\alpha\beta}) e_\beta^{k-1},
\end{align}
where $\mathrm{deg}(\chi^{\alpha\beta})$ is the {\em Brouwer degree} of the surjection $\chi^{\alpha\beta} : \mathbb{S}^{k-1} \cong \partial e_\alpha^k \stackrel{\phi^k_{\alpha}}{\rightarrow} X^{k-1} \stackrel{/}{\rightarrow} X^{k-1} / (X^{k-1} - e_\beta^{k-1}) \cong \mathbb{S}^{k-1}$, mapping the boundary of $e_\alpha^k$ (homeomorphic to $\mathbb{S}^{k-1}$) to the $(k-1)$-cell $e_\beta^{k-1}$, identified with $\mathbb{S}^{k-1}$ by collapsing its boundary to a single point \citep{Hatcher2001}. For a regular cellular complex, $\mathrm{deg}(\chi^{\alpha\beta})$ is either $0$ or $\pm 1$ depending on whether or not $e_\beta^{k-1}$ is a part of $\partial e_\alpha^k$ under the attaching map. That is, if $e_\beta^{k-1}$ is not part of the boundary $\partial e_\alpha^k$, then we take $\mathrm{deg}(\chi^{\alpha\beta}) = 0$.
Otherwise, if the induced orientation of $\partial e_\alpha^k$ aligns with the orientation of $e_\beta^{k-1}$ (i.e., take an embedding $X^{k-1} \xhookrightarrow{} \mathbb{R}^k$ and see how the unit normal fields of $\partial e_\alpha^k$ and $e_\beta^{k-1}$ align), then we take $\mathrm{deg}(\chi^{\alpha\beta}) = 1$ and if it has opposite orientations, we take $\mathrm{deg}(\chi^{\alpha\beta}) = -1$. We will also adopt the notation $\mathrm{deg}(e_\alpha^k \rightarrow e_\beta^{k-1})$ to denote $\mathrm{deg}(\chi^{\alpha\beta})$ in order to make the cells involved more explicit.

The computation of the Brouwer degree can in general be challenging. However, in some special cases, the computation can be simplified, as we show in the following examples.

\begin{example}
    Consider a $k$-simplex, which can be identified as a collection of $k+1$ vertices, say $v_0, \ldots, v_k$ without loss of generality. Its orientation is determined by the parity of the permutation of these $k+1$ vertices. Hence, we represent it by an equivalence class $[v_0, \ldots, v_k]$, where equivalence is defined by the parity. We set $e^k_\alpha = [v_0, \ldots, v_k]$.
    Now consider an oriented face of this simplex, represented by an equivalence class of $k$ vertices $e^{k-1}_\beta = [w_1, \ldots, w_k]$, such that $\{w_1, \ldots, w_k\} = \{v_0, \ldots, \cancel{v_\ell}, \ldots, v_k\}$ for some $\ell \in \{0, \ldots, k\}$. Naturally, this can be represented by a permutation $\sigma \in S_{k+1}$ on $\{0, \ldots, k\}$ with $\sigma(0) = \ell$ and $w_i = v_{\sigma(i)}$ for $i = 1, \ldots, k$. Then, we take $\mathrm{deg}(e_\alpha^k \rightarrow e_\beta^{k-1}) = \mathrm{sgn}(\sigma)$, the signature of the permutation $\sigma$. Thus, in general, we can write
    \begin{align}\label{eq:bdry-simplex-case}
        \partial [v_1, \ldots, v_{k+1}] = \sum_{\ell=1}^{k+1} (-1)^{\ell}[v_1, \ldots, \cancel{v_\ell}, \ldots, v_{k+1}].
    \end{align}
\end{example}

\begin{example}
    Consider a 2D polygonal cell $e^k_\alpha$ with vertices $v_1, \ldots, v_k$, and re-order them such that it revolves around the polygon in a clockwise or anticlockwise manner: i.e., $v_{\sigma(1)} \rightarrow v_{\sigma(2)} \rightarrow \cdots \rightarrow v_{\sigma(k-1)} \rightarrow v_{\sigma(k)} \rightarrow v_{\sigma(1)}$ for some permutation $\sigma \in S_k$. Whether the ordering here is clockwise or anticlockwise determines the orientation of the cell. Since the choice of the first vertex in the ordering is not important, we represent this polygon as the equivalence class $e^k_\alpha = [v_{\sigma(1)}, \ldots, v_{\sigma(k)}]$, where equivalence is defined by permutation with respect to the cyclic group $C_k \subset S_k$. Now, an oriented edge of this polygon can be represented by an ordered tuple $e_\beta^{k-1} = (v_i, v_j)$ for $i < j$, such that either $(i, j) = (\sigma(\ell), \sigma(\ell+1))$ or $(i, j) = (\sigma(\ell+1), \sigma(\ell))$ for some $\ell \in \{1, \ldots, k\}$ (we use the convention $k+1 \equiv 1$, i.e., assume the indices are elements of $\mathbb{Z}/k\mathbb{Z}$). Then, we set $\mathrm{deg}(e_\alpha^k \rightarrow e_\beta^{k-1}) = 1$ in the former case and $\mathrm{deg}(e_\alpha^k \rightarrow e_\beta^{k-1}) = -1$ in the latter. Hence, using the convention $(v_i, v_j) = -(v_j, v_i)$, we can check that
    \begin{align}\label{eq:bdry-polygon-case}
        \partial [v_{\sigma(1)}, \ldots, v_{\sigma(k)}] = \sum_{\ell=1}^{k+1} (v_{\sigma(\ell)}, v_{\sigma(\ell+1)}).
    \end{align}
\end{example}

In the general case, we can compute the boundary of an arbitrary cell by considering its {\em simplicial decomposition}. That is, we discretise the $k$-cell $e^k_\alpha$ by a collection of $N$ $k$-simplices. This can be expressed as a $k$-chain
\begin{align}
    e^k_\alpha = \sum_{i=1}^N [v_{i_1}, \ldots, v_{i_k}].
\end{align}
We can compute its boundary by first taking
\begin{align}
    \partial e^k_\alpha = \sum_{i=1}^N 
    \partial [v_{i_1}, \ldots, v_{i_k}],
\end{align}
then using \eqref{eq:bdry-simplex-case} to give a simplicial decomposition of $\partial e^k_\alpha$, and finally collectivising with respect to the simplicial decomposition of $e^{k-1}_\beta$.

\section{Numerical Representation of Cellular Complexes}\label{app:numerical-representation}
In this appendix, we provide a full derivation of the numerical representation of key concepts on cellular complexes, as displayed in Table \ref{table:numerical-representation}.
To make this appendix self-contained, we first re-introduce some definitions.

\begin{definition}[$k$-chains]
A $k$-chain on $X$ is a free Abelian group whose generator is the set of all $k$-cells comprising $X$. The space of all $k$-chains on $X$ is denoted $C_k(X)$.
\end{definition}

\begin{definition}[Boundary operators]
    For $k \in \{1, \ldots, n\}$, the boundary operator is a group homomorphism $\partial_k : C_k(X) \rightarrow C_{k-1}(X)$ mapping $k$-chains to $(k-1)$-chains, i.e.,
    \begin{align}
        \partial_k \left(\sum_{\alpha=1}^{N_k} n_a e^k_\alpha\right) = \sum_{\alpha=1}^{N_k} n_a \partial e^k_\alpha,
    \end{align}
    where $\partial e^k_\alpha$ is the boundary of the cell $e_k^\alpha$, viewed as a $(k-1)$-chain (see Appendix \ref{app:orientation}). For convention, we also take $\partial_k c \equiv 0$ for $k \in \{0, n+1\}$.
\end{definition}

\begin{definition}[$k$-cochains]
A $k$-cochain on $X$ is a group homomorphism $f : C_k(X) \rightarrow \mathbb{R}$ assigning real numbers to $k$-chains, i.e.,
\begin{align} \label{eq:cochain-def-appendix}
    f\Big(\sum_{\alpha=1}^{N_k} n_\alpha e^k_\alpha\Big) = \sum_{\alpha=1}^{N_k} n_\alpha f(e_\alpha^k),
\end{align}
where $f(e_\alpha^k) \in \mathbb{R}$ is the value of $f$ at cell $e_\alpha^k$. The space of all $k$-cochains on $X$ is denoted $C^k(X)$ and forms a real vector space.
\end{definition}

\begin{definition}[Coboundary operators]
    For $0 \leq k < n$, the coboundary operator $d_k : C^k(X) \rightarrow C^{k+1}(X)$ is a linear map defined via the relation
    \begin{align} \label{eq:coboundary-def-appendix}
        d_k f(c) = f(\partial_{k+1} c)
    \end{align}
    for all $f \in C^k(X)$ and $c \in C_{k+1}(X)$.
    For convention, we also take $d_k f \equiv 0$ for $k \in \{-1, n\}$.
\end{definition}

\begin{definition}[$L^2$-inner product on cochains]\label{def:inner-product-appendix}
For each $k$, let $\{w_\alpha^k\}_{\alpha=1}^{N_k}$ be a set of real, positive weights. Then, for any $f, g \in C^k(X)$, we define the weighted $L^2$-inner product as
\begin{align} \label{eq:k-cochain-inner-product-appendix}
    \left<f, g\right>_{L^2(\vec{w}^k)} := \sum_{\alpha=1}^{N_k} w_\alpha^k \,f(e^k_\alpha) \,g(e^k_\alpha).
\end{align}
\end{definition}

\begin{definition}[$L^2$-adjoint of the coboundary operator]
For each $k$, let $\{w_\alpha^k\}_{\alpha=1}^{N_k}$ be a set of real, positive weights. The $L^2$-adjoint of the coboundary operator, denoted $d_k^* : C^{k+1}(X) \rightarrow C^{k}(X)$ is defined by
\begin{align} \label{eq:codifferential-def-appendix}
    \left<d_k^* f, g\right>_{L^2(\vec{w}^{k})} = \left<f, d_{k}g\right>_{L^2(\vec{w}^{k+1})},
\end{align}
for any $f \in C^{k+1}(X)$ and $g \in C^{k}(X)$.
\end{definition}

\begin{definition}[Hodge Laplacian]\label{def:hodge-laplacian-appendix}
    The Hodge Laplacian $\Delta_k : C^k(X) \rightarrow C^k(X)$ on the space of $k$-cochains is defined as
    \begin{align}\label{eq:hodge-laplacian-appendix}
        \Delta_k := d_{k-1} \circ d_{k-1}^* + d_{k}^* \circ d_{k}.
    \end{align}
\end{definition}

To make explicit computations with them, we wish to represent these using matrices and vectors. Fortunately, this is not difficult as the space of chains / cochains forms a free Abelian group / vector space, which is naturally isomorphic to $\mathbb{Z}^n$ / $\mathbb{R}^n$.

To this end, we fix a labelling $\alpha \mapsto e^k_\alpha$ of the $k$-cells comprising a cellular complex $X$, which forms an ordered basis $(e^k_1, \ldots, e^k_{N_k})$. Then, an arbitrary $k$-chain $c = \sum_{\alpha=1}^{N_k} n_\alpha e^k_\alpha \in C_k(X)$ may be represented by a vector $\boldsymbol{c} = (n_1, \ldots, n_{N_k})^\top$ in $\mathbb{Z}^{N_k}$. Similarly, a $k$-cochain $f = \sum_{\alpha=1}^{N_k} f_\alpha (e^k_\alpha)^* \in C^k(X)$ can be represented by a vector $\boldsymbol{f} = (f_1, \ldots, f_{N_k})^\top$ in $\mathbb{R}^{N_k}$. Under this representation, cochain evaluation \eqref{eq:cochain-def-appendix} can simply be expressed as a dot product $f(c) = \boldsymbol{f}^\top \boldsymbol{c} \in \mathbb{R}$.

Next, we consider the boundary and coboundary operators.
The boundary operator can be expressed as a signed incidence matrix $\mathbf{B}_k \in \mathbb{Z}^{N_{k-1} \times N_k}$, whose $j$-th column corresponds to the vector representation of the cell boundary $\partial e^k_j$, viewed as a $(k-1)$-chain (see Appendix \ref{app:orientation}).
That is,
\begin{align}
    [\mathbf{B}_k]_{ij} = \mathrm{deg}(e_j^k \rightarrow e_i^{k-1})
\end{align}
for $k \in \{1, \ldots, n\}$, and by convention, we take $\mat{B}_k \equiv \mat{0}$ for $k \in \{0, n+1\}$.
Similarly, the coboundary operator can be represented by a matrix $\mathbf{D}_k \in \mathbb{R}^{N_{k+1} \times N_k}$. Using \eqref{eq:coboundary-def-appendix}, we have
\begin{align}
    \boldsymbol{f}^\top \mathbf{D}_k^\top \boldsymbol{c} = \boldsymbol{f}^\top \mathbf{B}_{k+1} \boldsymbol{c}, \\ \Leftrightarrow \quad \mathbf{D}_k = \mathbf{B}_{k+1}^\top.
\end{align}
Thus, the coboundary operator is identified with the transpose of the signed incidence matrix. Finally, let $\mathbf{W}_k = \mathtt{diag}(w_1^k, \ldots, w_{N_k}^k)$ be the weight matrix defining the $L^2$-inner product \eqref{eq:k-cochain-inner-product-appendix}. i.e.,
\begin{align}
    \left<f, g\right>_{L^2(\vec{w}^k)} = \boldsymbol{f}^\top\mathbf{W}_k \boldsymbol{g}.
\end{align}
Then, letting $\mathbf{D}^*_k \in \mathbb{R}^{N_k \times N_{k+1}}$ be the matrix representation of the adjoint of the coboundary,  \eqref{eq:codifferential-def-appendix} implies
\begin{align}
    \boldsymbol{f}^\top (\mathbf{D}_k^*)^\top \mathbf{W}_{k} \boldsymbol{g} = \boldsymbol{f}^\top \mathbf{W}_{k+1} \mathbf{B}_{k+1}^\top \boldsymbol{g} \\
    \quad \Leftrightarrow \quad \mathbf{D}_k^* = \mathbf{W}_{k}^{-1} \mathbf{B}_{k+1} \mathbf{W}_{k+1}.
\end{align}

Putting this together, we find the matrix expression $\mathbf{\Delta}_k \in \mathbb{R}^{N_k \times N_k}$ for the Hodge Laplacian operator using \eqref{eq:hodge-laplacian-appendix}:
\begin{align}
    &\mathbf{\Delta}_k = \mathbf{D}_{k} \mathbf{D}^*_{k} + \mathbf{D}^*_{k+1} \mathbf{D}_{k+1} \\
    &= \mathbf{B}_{k}^\top (\mathbf{W}_{k-1}^{-1} \mathbf{B}_{k} \mathbf{W}_{k}) + (\mathbf{W}_{k}^{-1} \mathbf{B}_{k+1} \mathbf{W}_{k+1}) \mathbf{B}_{k+1}^\top.
    \label{eq:weighted-hodge-laplacian-appendix}
\end{align}

\section{Characterisation of GPs on Cellular Complexes}\label{app:gp-characterisation}
Here, we prove the results in Section \ref{sec:GPs-on-CCs} characterising Gaussian random cochains and GPs on cellular complexes by a mean and a kernel.
\subsection{Proof of Theorem \ref{eq:GRC-characterisation}}\label{app:grc-characterisation}
We restate Theorem \ref{eq:GRC-characterisation} below for completeness.
\begin{theorem}
    A Gaussian random cochain $f : \Omega \rightarrow C^k(X)$ is fully characterised by a mean $\mu \in C^k(X)$ and a kernel $\kappa : C_k(X) \times C_k(X) \rightarrow \mathbb{R}$.
\end{theorem}
\begin{proof}
    The proof is almost identical to that of Lemmas 9--10 in \cite{hutchinson2021vector}.

    $(\Rightarrow)$ First, we show that given a Gaussian random cochain $f$,  we can define a mean and a kernel object. For this, we simply set $\mu(c) := \mathbb{E}[f(c)]$ and $\kappa(c, c') := \mathrm{Cov}[f(c), f(c')]$. The former is clearly a cochain since $f$ is a (random) cochain. The latter can be easily checked to be a kernel in the sense of Definition \ref{def:kernel-on-chains}:
    \begin{itemize}
        \item (Symmetry) This follows from the symmetry of the covariance operator $\mathrm{Cov}[f(c), f(c')] = \mathrm{Cov}[f(c'), f(c)]$.
        \item (Group bi-homomorphism) This follows from the fact that by definition, $f$ is a group homomorphism and using the bilinearity of the covariance operator.
        \item (Positive semi-definiteness) Fixing $c_1, \ldots, c_m \in C_k(X)$, we have
        \begin{align}
            \sum_{\alpha, \beta=1}^m \kappa(c_\alpha, c_\beta) &= \sum_{\alpha, \beta=1}^m \mathrm{Cov}[f(c_\alpha), f(c_\beta)] = \mathbb{E}\Bigg[\Big(\sum_{\alpha=1}^m (f(c_\alpha) - \mathbb{E}[f(c_\alpha)])\Big)^2\Bigg] \geq 0.
        \end{align}
    \end{itemize}

    $(\Leftarrow)$ Next, we show that given $\mu \in C^k(X)$ and a kernel $\kappa: C_k(X) \times C_k(X) \rightarrow \mathbb{R}$, we can construct a Gaussian random cochain. Take
    \begin{align}
    \boldsymbol{\mu} =
    \begin{bmatrix}
    \mu(e^k_1) \\
    \vdots \\
    \mu(e^k_{N_k})
    \end{bmatrix} \in \mathbb{R}^{N_k},
    \qquad
    \mathbf{K} =
    \begin{bmatrix}
    k(e^k_1, e^k_1) & \cdots & k(e^k_1, e^k_{N_k}) \\
    \vdots & \ddots & \vdots \\
    k(e^k_{N_k}, e^k_1) & \cdots & k(e^k_{N_k}, e^k_{N_k})
    \end{bmatrix} \in \mathbb{R}^{N_k \times N_k}.
\end{align}
    This uniquely defines a multivariate Gaussian random variable $\boldsymbol{f} \sim \mathcal{N}(\vec{\mu}, \mat{K})$. Now let $\varphi : C^k(X) \stackrel{\sim}{\rightarrow} \mathbb{R}^{N_k}$ be the group isomorphism identifying $k$-cochains by a vector in $\mathbb{R}^{N_k}$ via the labelling $\alpha \mapsto e_\alpha^k$. We also take the group isomorphism $\psi : C_k(X) \stackrel{\sim}{\rightarrow} \mathbb{Z}^{N_k}$, defined by $f(c) = \varphi(f)^\top \psi(c)$ for any $f \in C^k(X)$ and $c \in C_k(X)$. Then, we set $f := \varphi^{-1}\vec{f}$, which defines a Gaussian random cochain, since for any $c \in C_k(X)$, we have $f(c) = [\varphi^{-1}\vec{f}](c) = \vec{f}^\top \psi(c)$, which is univariate Gaussian.
\end{proof}

\subsection{Proof of Theorem \ref{eq:CCGP-characterisation}}\label{app:ccgp-characterisation}
We restate Theorem \ref{eq:CCGP-characterisation} below for completeness.
\begin{theorem}
    A Gaussian process on a cellular complex $X$ (abbreviated as CC-GP) is fully characterised by a mean $\mu \in \bigoplus_{k=0}^n C^k(X)$ and a kernel $\kappa : \bigoplus_{k=0}^n C_k(X) \times \bigoplus_{k=0}^n C_k(X) \rightarrow \mathbb{R}$.
\end{theorem}
\begin{proof}
    The proof is almost identical to the proof of Theorem \ref{eq:GRC-characterisation} (see Appendix \ref{app:grc-characterisation}), hence we will omit some details to avoid repetition.
    
    $(\Rightarrow)$ For a CC-GP $f : \Omega \rightarrow \bigoplus_{k=0}^n C^k(X)$,  we can define a mean and a kernel by setting $\mu(c) := \mathbb{E}[f(c)]$ and $\kappa(c, c') := \mathrm{Cov}[f(c), f(c')]$, for any $c, c' \in \bigoplus_{k=0}^n C^k(X)$.

    $(\Leftarrow)$ Given $\mu \in \bigoplus_{k=0}^n C^k(X)$ and a kernel $\kappa : \bigoplus_{k=0}^n C_k(X) \times \bigoplus_{k=0}^n C_k(X) \rightarrow \mathbb{R}$, we take
    \begin{align}
        \vec{\mu} = 
        \begin{pmatrix}
            \vec{\mu}_1 \\
            \vdots \\
            \vec{\mu}_n
        \end{pmatrix},
        \quad
        \mathbf{K} =
        \begin{pmatrix}
        \mathbf{K}_{11} & \cdots & \mathbf{K}_{1n} \\
        \vdots & \ddots & \vdots \\
        \mathbf{K}_{n1} & \cdots & \mathbf{K}_{nn}
        \end{pmatrix}
    \end{align}
    with $[\vec{\mu}_i]_j = \mu(e^i_j)$ and $[\mat{K}_{nm}]_{ij} = \kappa(e^n_i, e^m_j)$, which uniquely defines a multivariate Gaussian random variable $\boldsymbol{f} \sim \mathcal{N}(\vec{\mu}, \mat{K})$ in $\mathbb{R}^{N_1 + \cdots + N_n}$.
    Now, consider the group isomorphism $\varphi : \bigoplus_{k=0}^n C^k(X) \stackrel{\sim}{\rightarrow} \mathbb{R}^{N_1 + \cdots + N_n}$ by fixing a labelling $\alpha \mapsto e_\alpha^k$ for each $k=0, \ldots, n$. Then, $f := \varphi^{-1}\vec{f}$ defines a CC-GP. 
\end{proof}

\section{Kernels on Cellular Complexes}\label{app:kernels}
In this appendix, we provide further details on the kernels that we consider in this paper, namely the Matérn kernel on cellular complexes (CC-Matérn) and the reaction-diffusion kernel (RD). We will need the following definition to formalise our kernels.

\begin{definition}\label{def:group-homomorphism-appendix}
Consider the weighted $L^2$-inner product $\left<\cdot, \cdot\right>_{L^2(\vec{w}^k)} : C^k(X) \times C^k(X) \rightarrow \mathbb{R}$ on the space of $k$-cochains (Definition \ref{def:inner-product-appendix}). We define a group homomorphism $\flat : C_k(X) \rightarrow C^k(X)$ by
\begin{align}\label{eq:flat-operator}
    f(c) = \big\langle f, c^\flat\big\rangle_{L^2(\vec{w}^k)},
\end{align}
for any $c \in C_k(X)$ and $f \in C^k(X)$. The existence of $c^\flat$ follows from the Riesz representation theorem.
\end{definition}

Next, we introduce the concept of a Gaussian white-noise cochain, defined as follows.

\begin{definition}[Gaussian white-noise cochain]\label{def:gaussian-white-noise-cochain}
    We define a zero-mean Gaussian white noise cochain as a Gaussian random cochain $\mathcal{W} : \Omega \rightarrow C^k(X)$ satisfying
    \begin{itemize}
        \item $\mathbb{E}[\left<\mathcal{W}, f\right>_{L^2(\vec{w}^k)}] = 0$ for any $f \in C^k(X)$.
        \item $\mathbb{E}[\left<\mathcal{W}, f\right>_{L^2(\vec{w}^k)} \left<\mathcal{W}, g\right>_{L^2(\vec{w}^k)}] = \left<f, g\right>_{L^2(\vec{w}^k)}$ for any $f, g \in C^k(X)$
    \end{itemize}
\end{definition}

Now, we are ready to define our notion of the Matérn kernel on cellular complexes.

\subsection{Matérn Kernel}\label{app:matern-kernel}
For simplicity, let us assume for now that the cell weights are $w^k_\alpha = 1$ for all $k, \alpha$. In this special case, we will use the shorthand notation $\left<\cdot, \cdot\right>_{L^2} \equiv \left<\cdot, \cdot\right>_{L^2(\vec{w}^k)}$. We will deal with the more general case in Appendix \ref{app:arbitrary-cell-weights-matern}.

Following the construction in \cite{Borovitskiy2021}, we define a Matérn Gaussian random $k$-cochain as a solution to the stochastic system
\begin{align}\label{eq:matern-gp-def-appendix}
    \left(\frac{2\nu}{\ell^2} + \Delta_k\right)^{\nu/2} f = \mathcal{W},
\end{align}
where $f \in C^k(X)$ and $\mathcal{W} : \Omega \rightarrow C^k(X)$ is the Gaussian white-noise cochain (Definition \ref{def:gaussian-white-noise-cochain}).

The operator $\left(\frac{2\nu}{\ell^2} + \Delta_k\right)^{\nu/2}$ is to be understood as an operation in frequency space, by the following construction.
Let $\{(\lambda_i^2, u_i)\}_{i=1}^{N_k}$, be solutions to the eigenproblem $\Delta_k u_i = \lambda_i^2 u_i$ such that the eigencochains $\{u_i\}_{i=1}^{N_k}$ are orthonormal in $L^2$. Representing $f$ as $f = \sum_i \left<f, u_i\right>_{L^2(w^k)} u_i$, we define
\begin{align}\label{eq:matern-operator}
    \left(\frac{2\nu}{\ell^2} + \Delta_k\right)^{\nu/2} f := \sum_{i=1}^{N_k} \left(\frac{2\nu}{\ell^2} + \lambda_i^2 \right)^{\nu/2} \left<f, u_i\right>_{L^2} u_i,
\end{align}
which is a linear operator on the space of $k$-cochains.
We define the Matérn kernel $\kappa : C_k(X) \times C_k(X) \rightarrow \mathbb{R}$ as as a solution to the linear system
\begin{align}\label{eq:matern-kernel-def-appendix}
    \left(\frac{2\nu}{\ell^2} + \Delta_k\right)^{\nu} \kappa(c, \,\cdot\,) = c^\flat, \quad \forall c \in C_k(X).
\end{align}

\subsubsection{Proof of Proposition \ref{prop:matern-gp}}\label{app:matern-kernel-proof}
We restate Proposition \ref{prop:matern-gp} below for completeness.
\begin{proposition}
    The solution to \eqref{eq:matern-kernel-def-appendix} is related to the solution $f$ of the system \eqref{eq:matern-gp-def-appendix} as:
    \begin{align}\label{eq:matern-kernel-property-appendix}
        \kappa(c, c') = \mathbb{E}[f(c)f(c')], \quad 
    \forall c, c' \in C_k(X)
    \end{align}
    Thus, $\kappa$ is the kernel corresponding to the Gaussian random cochain $f$.
\end{proposition}
\begin{proof}
    We first claim that the unique solution $f$ to \eqref{eq:matern-gp-def-appendix} can be represented as
    \begin{align}
        f = \sum_i \left(\frac{2\nu}{\ell^2} + \lambda_i^2\right)^{-\nu/2} \left<\mathcal{W}, u_i\right>_{L^2} u_i.
    \end{align}
    This can be checked by simply substituting this expression inside \eqref{eq:matern-operator} and using the $L^2$-orthonormality of the eigencochains $\{u_i\}_{i=1}^{N_k}$. The uniqueness can be checked by the linearity of the operator $\left(\frac{2\nu}{\ell^2} + \Delta_k\right)^{\nu/2}$ and the fact that the solution to the system $\left(\frac{2\nu}{\ell^2} + \Delta_k\right)^{\nu/2} f = 0$ is satisfied only by $f \equiv 0$.

    Similarly, the solution to \eqref{eq:matern-kernel-def-appendix} is given by
    \begin{align}
        \kappa(c, \cdot) &= \sum_i \left(\frac{2\nu}{\ell^2} + \lambda_i^2\right)^{-\nu} \big\langle c^\flat, u_i\big\rangle_{L^2} u_i(\cdot) \\
        &\stackrel{\eqref{eq:flat-operator}}{=} \sum_i \left(\frac{2\nu}{\ell^2} + \lambda_i^2\right)^{-\nu} u_i(c) u_i(\cdot) \label{eq:matern-kernel-explicit}
    \end{align}

    Next, we show that for arbitrary $c, c' \in C_k(X)$, we have
    \begin{align}
        \mathbb{E}[f(c)f(c')] &= \mathbb{E}\left[\sum_i \sum_j \left(\frac{2\nu}{\ell^2} + \lambda_i^2\right)^{-\nu/2} \left(\frac{2\nu}{\ell^2} + \lambda_j^2\right)^{-\nu/2} \left<\mathcal{W}, u_i\right>_{L^2} \left<\mathcal{W}, u_j\right>_{L^2} u_i(c) u_j(c')\right] \\
        &= \sum_i \sum_j \left(\frac{2\nu}{\ell^2} + \lambda_i^2\right)^{-\nu/2} \left(\frac{2\nu}{\ell^2} + \lambda_j^2\right)^{-\nu/2} \mathbb{E}\left[\left<\mathcal{W}, u_i\right>_{L^2} \left<\mathcal{W}, u_j\right>_{L^2}\right] u_i(c) u_j(c') \\
        &= \sum_i \sum_j \left(\frac{2\nu}{\ell^2} + \lambda_i^2\right)^{-\nu/2} \left(\frac{2\nu}{\ell^2} + \lambda_j^2\right)^{-\nu/2} \underbrace{\left<u_i, u_j\right>_{L^2}}_{= \delta_{ij}} u_i(c) u_j(c') \\
        &= \sum_i \left(\frac{2\nu}{\ell^2} + \lambda_i^2\right)^{-\nu} u_i(c) u_i(c') \\
        &\stackrel{\eqref{eq:matern-kernel-explicit}}{=} \kappa(c, c'),
    \end{align}
    verifying property \eqref{eq:matern-kernel-property-appendix}. Finally, we show that $\kappa$ is indeed a kernel. The symmetry of $\kappa$ can be easily verified from the explicit expression \eqref{eq:matern-kernel-explicit}, that is, $\kappa(c, c') = \kappa(c', c)$ for any $c, c' \in C_k(X)$. Checking that $\kappa$ is a group bi-homomorphism also follows easily from expression \eqref{eq:matern-kernel-explicit} using the fact that $u_i$ is a group homomorphism by the definition of cochains. Fixing $c_1, \ldots, c_m \in C_k(X)$ such that $c_\alpha \neq 0$ for some $\alpha$, we also have
    \begin{align}
        \sum_{\alpha, \beta=1}^m \kappa(c_\alpha, c_\beta) &= \sum_{\alpha, \beta=1}^m  \sum_i \left(\frac{2\nu}{\ell^2} + \lambda_i^2\right)^{-\nu} u_i(c_\alpha) u_i(c_\beta) \\
        &= \sum_i \left(\frac{2\nu}{\ell^2} + \lambda_i^2\right)^{-\nu} \sum_{\alpha=1}^m u_i(c_\alpha) \sum_{\beta=1}^m u_i(c_\beta) \\
        &= \sum_i \left(\frac{2\nu}{\ell^2} + \lambda_i^2\right)^{-\nu} \left(\sum_{\alpha=1}^m u_i(c_\alpha)\right)^2 \\
        &> 0,
    \end{align}
    verifying the positive-definiteness of $\kappa$. Hence, $\kappa$ is a kernel on $C_k(X)$, as expected.
\end{proof}

\subsubsection{Matrix representation}\label{app:matern-kernel-representation}
From \eqref{eq:matern-kernel-explicit}, we can deduce the matrix representation of the Matérn kernel as
\begin{align}
    \mathbf{K} = \mathbf{U} \left(\frac{2\nu}{\ell^2}\mathbf{I} + \mathbf{\Lambda}^2\right)^{-\nu} \mathbf{U}^\top. \label{eq:matern-kernel-representation_matrix}
\end{align}
Another way to derive this representation is by directly considering the numerical representation of system \eqref{eq:matern-gp-def-appendix}:
\begin{align} \label{eq:numerical-representation-matern-eq}
    \mathbf{L} \boldsymbol{f} = \vec{w},
\end{align}
where
\begin{align}\label{eq:L-nonweighted}
    \mathbf{L} := \mathbf{U} \left(\frac{2\nu}{\ell^2}\mathbf{I} + \mathbf{\Lambda}^2\right)^{\nu/2} \mathbf{U}^\top
\end{align}
is the numerical representation of the linear operator \eqref{eq:matern-operator}, and $\vec{w} \sim \mathcal{N}(0, \mat{I})$. Then, we have
\begin{align}
    \vec{f} = \mat{L}^{-1} \vec{w} \sim \mathcal{N}(0, \mat{K}),
\end{align}
where
\begin{align}
    \mat{K} &= \mat{L}^{-1}\mat{L}^{-\top}.
\end{align}
We claim that
\begin{align}
    \mat{L}^{-1} = \mathbf{U} \left(\frac{2\nu}{\ell^2}\mathbf{I} + \mathbf{\Lambda}^2\right)^{-\nu/2} \mathbf{U}^\top.
\end{align}
This can be verified by computing
\begin{align}
    \mat{L}^{-1}\mat{L} &= \mathbf{U} \left(\frac{2\nu}{\ell^2}\mathbf{I} + \mathbf{\Lambda}^2\right)^{-\nu/2} 
    \underbrace{\mathbf{U}^\top \mathbf{U}}_{= \mat{I}} \left(\frac{2\nu}{\ell^2}\mathbf{I} + \mathbf{\Lambda}^2\right)^{\nu/2} \mathbf{U}^\top \\
    &= \mathbf{U} \underbrace{\left(\frac{2\nu}{\ell^2}\mathbf{I} + \mathbf{\Lambda}^2\right)^{-\nu/2} \left(\frac{2\nu}{\ell^2}\mathbf{I} + \mathbf{\Lambda}^2\right)^{\nu/2}}_{= \mat{I}} \mathbf{U}^\top \\
    &= \mathbf{U} \mathbf{U}^\top \\
    &= \mat{I},
\end{align}
and similarly, $\mat{L}\mat{L}^{-1} = \mat{I}$. Then, we can check that indeed we have
\begin{align}
    \mat{K} &= \mat{L}^{-1}\mat{L}^{-\top} \\
    &= \mathbf{U} \left(\frac{2\nu}{\ell^2}\mathbf{I} + \mathbf{\Lambda}^2\right)^{-\nu/2} \underbrace{\mathbf{U}^\top \mathbf{U}}_{=\mat{I}} \left(\frac{2\nu}{\ell^2}\mathbf{I} + \mathbf{\Lambda}^2\right)^{-\nu/2} \mathbf{U}^\top \\
    &= \mathbf{U} \left(\frac{2\nu}{\ell^2}\mathbf{I} + \mathbf{\Lambda}^2\right)^{-\nu} \mathbf{U}^\top.
\end{align}

\subsection{Reaction-diffusion Kernel}\label{app:rd-kernel}
Here, we provide further details on the reaction-diffusion kernel, presented in Section \ref{sec:rd-kernel}. Since many of the ideas are similar to the Matérn kernel (Appendix \ref{app:matern-kernel}), we omit some details.
We first lift the Dirac operator \eqref{eq:dirac-operator} to the direct sum space $\bigoplus_{k=1}^n C^k(X)$, where it is more natural as it defines a group homomorphism to itself (i.e., an endomorphism) $\mathcal{D} : \bigoplus_{k=1}^n C^k(X) \rightarrow \bigoplus_{k=1}^n C^k(X)$. This is given explicitly as
\begin{align} \label{eq:ext-dirac-operator-appendix}
    \mathcal{D} f = \sum_{k} \delta_k f_k = 
    \begin{pmatrix}
        d^*_0f_1 \\
        d_0f_0 + d^*_1f_2 \\
        \vdots \\
        d_{n-2}f_{n-2} + d^*_{n-1}f_n \\
        d_{n-1}f_{n-1}
    \end{pmatrix}.
\end{align}
Using the property $d_{k+1} \circ d_k = 0$ (equivalently, $d_{k}^* \circ d_{k+1}^* = 0$) of the coboundary operator, one can check that $\mathcal{D}^2 = \mathcal{L}$ (the super-Laplacian operator) holds.
We also extend the $L^2$ inner-product to the direct sum space $\bigoplus_{k=1}^n C^k(X)$, which we define by
\begin{align} \label{eq:extended-inner-product-appendix}
    \left<f, g\right>_{L^2(\vec{w})} := \sum_{k=1}^n \sum_{\alpha=1}^{N_k} w_\alpha^k \,f(e^k_\alpha) \,g(e^k_\alpha).
\end{align}
This trivially lifts Definition \ref{def:inner-product-appendix} and Definition \ref{def:gaussian-white-noise-cochain} to the direct sum setting, defining a group homomorphism $\flat : \bigoplus_{k=1}^n C_k(X) \rightarrow \bigoplus_{k=1}^n C^k(X)$ and a Gaussian white noise process $\mathcal{W} : \Omega \rightarrow \bigoplus_{k=1}^n C^k(X)$, respectively.

Now let $\mathcal{W} : \Omega \rightarrow \bigoplus_{k=1}^n C^k(X)$ be the white noise process on the direct sum space. We define the reaction-diffusion GP to be the solution to the stochastic system
\begin{align}\label{eq:rd-gp-def-appendix}
    \left(r + c \mathcal{D} + d \mathcal{L}\right)^{\nu/2} f = \mathcal{W}.
\end{align}

As before, the operator $\left(r + c \mathcal{D} + d \mathcal{L}\right)^{\nu/2}$ is to be understood as an operation in frequency space, as follows.
Let $\{(\lambda_i, u_i)\}_{i=1}^{N_1 + \cdots + N_n}$, be solutions to the eigenproblem $\mathcal{D} u_i = \lambda_i u_i$ such that $\{u_i\}_{i=1}^{N_1 + \cdots + N_n}$ are orthonormal in $L^2$. Since $\mathcal{D}^2 = \mathcal{L}$, we have that $\mathcal{D}$ and $\mathcal{L}$ share the same eigenfunction $u_i$ with eigenvalues $\lambda_i$ and $\lambda_i^2$, respectively. Thus, we define
\begin{align}\label{eq:rd-operator-appendix}
    \left(r + c \mathcal{D} + d \mathcal{L}\right)^{\nu/2} f := \sum_{k=1}^n\sum_{i=1}^{N_k} \left(r + c \lambda_i + d \lambda_i^2 \right)^{\nu/2} \left<f, u_i\right>_{L^2} u_i,
\end{align}
which is a linear operator on the direct sum space $\bigoplus_{k=1}^n C^k(X)$. We then define the reaction-diffusion kernel to be the solution to the system
\begin{align}\label{eq:rd-kernel-def-appendix}
    \left(r + c \mathcal{D} + d \mathcal{L}\right)^{\nu} k(c, \cdot) = c^\flat,
\end{align}
for any $c \in \bigoplus_{k=1}^n C^k(X)$. The solutions to \eqref{eq:rd-gp-def-appendix} and \eqref{eq:rd-kernel-def-appendix} are given explicitly by
\begin{align}
    f(\cdot) &= \sum_{k=1}^n\sum_{i=1}^{N_k} \left(r + c \lambda_i + d \lambda_i^2 \right)^{-\nu/2} \left<\mathcal{W}, u_i\right>_{L^2} u_i(\cdot), \label{eq:rd-f}\\
    \kappa(c, \cdot) &= \sum_{k=1}^n\sum_{i=1}^{N_k} \left(r + c \lambda_i + d \lambda_i^2 \right)^{-\nu} u_i(c) u_i(\cdot).\label{eq:rd-k}
\end{align}
Then, following the proof in Appendix \ref{app:matern-kernel-proof} line-by-line, one can verify that $\kappa$ is indeed a kernel for the GP $f$, that is, $\kappa(c, c') = \mathbb{E}[f(c) f(c')]$. Below, we present a more explicit proof under the numerical representation of \eqref{eq:rd-f} and \eqref{eq:rd-k}, which can be written in the form
\begin{align}
    \vec{f} &= \mat{U} \left(r\mat{I} + c\mat{\Lambda} + d\mat{\Lambda}^2\right)^{-\nu/2} \mat{U}^\top\vec{w} \label{eq:rd-f-explicit} \\
    \mat{K} &= \mat{U} \left(r\mat{I} + c\mat{\Lambda} + d\mat{\Lambda}^2\right)^{-\nu} \mat{U}^\top.\label{eq:rd-k-explicit}
\end{align}
Here, we denoted by $\mat{\Lambda} = \mathtt{diag}(\lambda_1, \ldots, \lambda_{N_1 + \cdots + N_n})$ the diagonal matrix of eigenvalues, $\vec{w} \in \mathcal{N}(0, \mat{I})$ is the numerical representation of $\mathcal{W}$, and $\mat{U} = (\vec{u}_1, \ldots, \vec{u}_{N_1 + \cdots + N_n})$ is the matrix of eigenvectors.

\subsubsection{Proof of Proposition \ref{prop:rd-gp}}\label{app:rd-kernel-proof}
\begin{proposition}
    The kernel defined by \eqref{eq:rd-k-explicit} is related to $\vec{f}$ given by \eqref{eq:rd-f-explicit}, as
    \begin{align}
        [\mat{K}]_{ij} = \mathbb{E}[f_i f_j], \quad 
    \forall i, j = 1, \ldots, N_1 + \cdots + N_n.
    \end{align}
\end{proposition}
\begin{proof}
    We have
    \begin{align}
        \mathbb{E}[\vec{f}\vec{f}^\top] &= \mat{U} \left(r\mat{I} - c\mat{\Lambda} + d \mat{\Lambda}^2\right)^{-\nu/2} \mat{U}^\top \underbrace{\mathbb{E}[\vec{w}\vec{w}^\top]}_{= \mat{I}} \mat{U} \left(r\mat{I} - c\mat{\Lambda} + d \mat{\Lambda}^2\right)^{-\nu/2} \mat{U}^\top \\
        &= \mat{U} \left(r\mat{I} - c\mat{\Lambda} + d \mat{\Lambda}^2\right)^{-\nu/2} \underbrace{\mat{U}^\top \mat{U}}_{=\mat{I}} \left(r\mat{I} - c\mat{\Lambda} + d \mat{\Lambda}^2\right)^{-\nu/2} \mat{U}^\top \\
        &= \mat{U} \left(r\mat{I} - c\mat{\Lambda} + d \mat{\Lambda}^2\right)^{-\nu} \mat{U}^\top \\
        &= \mat{K},
    \end{align}
    which proves the claim.
\end{proof}

\subsection{Generalisation to Arbitrary Cell Weights}\label{app:arbitrary-cell-weights}
Here, we consider the case of general cell weights, extending the results in Appendix \ref{app:matern-kernel} and \ref{app:rd-kernel}. In particular, we demonstrate how we arrive at identical expressions for the kernels, only the eigenbasis must be orthonormal with respect to the weighted $L^2$ inner product instead of the standard one.

\subsubsection{Matérn Kernel}\label{app:arbitrary-cell-weights-matern}
Let $\{(\lambda_i^2, u_i)\}_{i=1}^{N_k}$ be the eigenpairs of the Hodge Laplacian operator $\Delta_k$, defined with respect to the weighted $L^2$-inner product $\langle \cdot, \cdot \rangle_{L^2(\vec{w}^k)}$.
We claim that in this case, the eigencochains $\{u_i\}_{i=1}^{N_k}$ can be set to be orthonormal under the weighted $L^2$-inner product. To see this, we first show that $\Delta_k$ is self-adjoint with respect to the weighted $L^2$-inner product:
\begin{align}
    \langle f, \Delta_k g\rangle_{L^2(\vec{w}^k)} &= \langle f, \, d_{k-1} d_{k-1}^* g \rangle_{L^2(\vec{w}^k)} + \langle f, \, d_{k}^* d_{k} g \rangle_{L^2(\vec{w}^k)} \\
    &= \langle d_{k-1}^* f, \, d_{k-1}^* g \rangle_{L^2(\vec{w}^{k-1})} + \langle d_{k} f, \, d_{k} g \rangle_{L^2(\vec{w}^{k+1})} \\
    &= \langle d_{k-1} d_{k-1}^* f, \, g \rangle_{L^2(\vec{w}^k)} + \langle d_{k}^* d_{k} f, \, g \rangle_{L^2(\vec{w}^k)} \\
    &= \langle \Delta_k f, g\rangle_{L^2(\vec{w}^k)}.
\end{align}
Now consider
\begin{align}\label{eq:u-delta-u pairing}
    \langle u_i, \Delta_k u_j\rangle_{L^2(\vec{w}^k)} = \lambda_j \langle u_i, u_j \rangle_{L^2(\vec{w}^k)}, \qquad \langle \Delta_k u_i, u_j\rangle_{L^2(\vec{w}^k)} = \lambda_i \langle u_i, u_j \rangle_{L^2(\vec{w}^k)}.
\end{align}
Since $\langle u_i, \Delta_k u_j\rangle_{L^2(\vec{w}^k)} = \langle \Delta_k u_i, u_j\rangle_{L^2(\vec{w}^k)}$ owing to the self-adjointness of $\Delta_k$, we have $\lambda_j \langle u_i, u_j \rangle_{L^2(\vec{w}^k)} = \lambda_i \langle u_i, u_j \rangle_{L^2(\vec{w}^k)}$ by \eqref{eq:u-delta-u pairing}. For $\lambda_i \neq \lambda_j$, this relation is true if and only if $\langle u_i, u_j \rangle_{L^2(\vec{w}^k)} = 0$. In the case $\lambda_i = \lambda_j = \lambda$, letting $E(\lambda)$ denote the corresponding eigenspace, we can simply take the orthonormal basis of $E(\lambda)$ to be the elements of $\{u_i\}_{i=1}^{N_k}$ corresponding to the eigenvalue $\lambda$. Thus, we have a choice of $\{u_i\}_{i=1}^{N_k}$ such that $\langle u_i, u_j \rangle_{L^2(\vec{w}^k)} = \delta_{ij}$ for all $i, j$. Now, with this choice of the eigenbasis, we can follow the arguments in Appendix \eqref{app:matern-kernel} almost identically to show that Proposition \ref{prop:matern-gp} still holds in the weighted setting, under the same definition for the Matérn GP \eqref{eq:matern-gp-def-appendix} and the Matérn kernel \eqref{eq:matern-kernel-def-appendix}. The only difference is that the $\flat$ operator and the Gaussian white noise $\mathcal{W}$ must take into consideration the cell weights, according to Definitions \ref{def:inner-product-appendix} and \ref{def:gaussian-white-noise-cochain}.

To illustrate this better, we consider its explicit representation in terms of a matrix-vector system. We first represent the weighted orthonormality condition of the eigenbasis by
\begin{align}\label{eq:orthonormality-1}
    \mat{U}^\top \mat{W} \mat{U} = \mat{I},
\end{align}
where $\mat{U}$ is the matrix of eigenvectors of $\mat{\Delta}_k$ and $\mat{W} = \mathtt{diag}(w^k_1, \ldots, w^k_{N_k})$. Due to the orthonormality, we also have that for any $f \in C^k(X)$, we have the expression $f = \sum_i \langle f, u_i\rangle_{L^2(\vec{w}^k)} u_i$. This has the vector expression
\begin{align}
    \mat{U} (\mat{U}^\top \mat{W} \vec{f}) = \vec{f},
\end{align}
which implies
\begin{align}\label{eq:orthonormality-2}
    \mat{U} \mat{U}^\top \mat{W} = \mat{I},
\end{align}
since $\vec{f}$ is arbitrary.
We can check that the operator \eqref{eq:matern-operator} takes the form (contrast this with \eqref{eq:L-nonweighted} in the non-weighted case):
\begin{align}
    \mathbf{L} := \mathbf{U} \left(\frac{2\nu}{\ell^2}\mathbf{I} + \mathbf{\Lambda}^2\right)^{\nu/2} \mathbf{U}^\top \mat{W}.
\end{align}
Again, we can check that its inverse reads
\begin{align}
    \mathbf{L}^{-1} := \mathbf{U} \left(\frac{2\nu}{\ell^2}\mathbf{I} + \mathbf{\Lambda}^2\right)^{-\nu/2} \mathbf{U}^\top \mat{W},
\end{align}
by verifying
\begin{align}
    \mathbf{L}^{-1}\mathbf{L} &= \mathbf{U} \left(\frac{2\nu}{\ell^2}\mathbf{I} + \mathbf{\Lambda}^2\right)^{-\nu/2} \underbrace{\mathbf{U}^\top \mat{W} \mathbf{U}}_{= \mat{I}} \left(\frac{2\nu}{\ell^2}\mathbf{I} + \mathbf{\Lambda}^2\right)^{\nu/2} \mathbf{U}^\top \mat{W} \\
    &= \mathbf{U} \underbrace{\left(\frac{2\nu}{\ell^2}\mathbf{I} + \mathbf{\Lambda}^2\right)^{-\nu/2} \left(\frac{2\nu}{\ell^2}\mathbf{I} + \mathbf{\Lambda}^2\right)^{\nu/2}}_{= \mat{I}} \mathbf{U}^\top \mat{W} \\
    &= \mathbf{U} \mathbf{U}^\top \mat{W} \\
    &= \mat{I},
\end{align}
and similarly, $\mathbf{L}\mathbf{L}^{-1} = \mat{I}$. Now, we consider the numerical representation $\vec{w}$ of the white noise process $\mathcal{W}$ in the weighted setting. We claim that
\begin{align}
    \vec{w} \sim \mathcal{N}(0, \mat{W}^{-1}).
\end{align}
This can be checked by using its definition (Definition \ref{def:gaussian-white-noise-cochain}), we have
\begin{align}
    \mathbb{E}[\left<\mathcal{W}, f\right>_{L^2(\vec{w}^k)} \left<\mathcal{W}, g\right>_{L^2(\vec{w}^k)}] &= \langle f, g \rangle_{L^2(\vec{w}^k)} \\
    \Leftrightarrow (\vec{f}^\top \mat{W}) \mathbb{E}[\vec{w}\vec{w}^\top] (\mat{W} \vec{g}) &= \vec{f}^\top \mat{W} \vec{g} \\
    \Leftrightarrow \mathbb{E}[\vec{w}\vec{w}^\top] &= \mat{W}^{-1}.
\end{align}
Hence by \eqref{eq:numerical-representation-matern-eq}, we have
\begin{align}
    \vec{f} = \mat{L}^{-1} \vec{w} \sim \mathcal{N}(0, \mat{K}),
\end{align}
where
\begin{align}
    \mat{K} &= \mat{L}^{-1}\mat{W}^{-1}\mat{L}^{-\top} \\
    &= \mathbf{U} \left(\frac{2\nu}{\ell^2}\mathbf{I} + \mathbf{\Lambda}^2\right)^{-\nu/2} \mathbf{U}^\top \underbrace{\mat{W} \mat{W}^{-1} \mat{W}}_{\mat{W}} \mathbf{U} \left(\frac{2\nu}{\ell^2}\mathbf{I} + \mathbf{\Lambda}^2\right)^{-\nu/2} \mathbf{U}^\top \\
    &= \mathbf{U} \left(\frac{2\nu}{\ell^2}\mathbf{I} + \mathbf{\Lambda}^2\right)^{-\nu/2} \underbrace{\mathbf{U}^\top \mat{W} \mathbf{U}}_{= \mat{I}} \left(\frac{2\nu}{\ell^2}\mathbf{I} + \mathbf{\Lambda}^2\right)^{-\nu/2} \mathbf{U}^\top \\
    &= \mathbf{U} \left(\frac{2\nu}{\ell^2}\mathbf{I} + \mathbf{\Lambda}^2\right)^{-\nu} \mathbf{U}^\top.
\end{align}
The expression for the kernel is the same as in the unweighted case, except that now $\mat{\Lambda}, \mat{U}$ are the eigenpairs of the weighted Hodge Laplacian \eqref{eq:weighted-hodge-laplacian-appendix}, the latter being orthonormal with respect to the weighted $L^2$-inner product instead of the standard one.

\subsubsection{Reaction-diffusion Kernel}\label{app:arbitrary-cell-weights-rd}
The same approach also applies to extending the reaction-diffusion kernel to the weighted setting.
In this case, the matrix expression for the Dirac operator \eqref{eq:ext-dirac-operator-appendix} reads
\begin{equation}\label{eq:weighted-dirac}
\vec{\mathcal{D}} = 
    \begin{pmatrix}
        \mathbf{0} & \mathbf{D}_0^* &  \cdots & \mathbf{0} \\
        \mathbf{D}_0 & \ddots & \ddots &  \vdots \\
        \vdots &  \ddots & \ddots & \mathbf{D}_{n-1}^* \\
        \mathbf{0} & \dots & \mathbf{D}_{n-1} & \mathbf{0}
    \end{pmatrix} =
    \begin{pmatrix}
        \mathbf{0} & \mat{W}_0^{-1}\mathbf{B}_1\mat{W}_1 &  \cdots & \mathbf{0} \\
        \mathbf{B}_1^\top & \ddots & \ddots &  \vdots \\
        \vdots &  \ddots & \ddots & \mat{W}_{n-1}^{-1}\mathbf{B}_n \mat{W}_n\\
        \mathbf{0} & \dots & \mathbf{B}_n^\top & \mathbf{0}
    \end{pmatrix}.
\end{equation}
We can check that $\vec{\mathcal{D}}$ is self-adjoint under the weighted $L^2$-inner product, that is
\begin{align}
    \vec{f}^\top \mat{W} \vec{\mathcal{D}} \vec{g} = \vec{g}^\top \mat{W} \vec{\mathcal{D}} \vec{f}.
\end{align}
Hence, we can choose an eigenbasis $\mat{U}$ of $\vec{\mathcal{D}}$ that is orthonormal under the weighted $L^2$-inner product, with eigenvalues given by $\mat{\Lambda} = \mathtt{diag}(\lambda_1, \ldots, \lambda_{N_1 + \cdots + N_n})$. The operator \eqref{eq:rd-operator-appendix} in this case can be expressed as
\begin{align}
    \left(r\mat{I} - c \boldsymbol{\mathcal{D}} + d \boldsymbol{\mathcal{L}} \right)^{\nu/2} \vec{f} &= \mat{U} \left(r\mat{I} - c \mat{\Lambda} + d \mat{\Lambda}^2\right)^{\nu/2} (\mat{U}^{-1}\mat{U})\mat{U}^\top \mat{W} \vec{f} \\
    &= \mat{U} \left(r\mat{I} - c \mat{\Lambda} + d \mat{\Lambda}^2\right)^{\nu/2}\mat{U}^\top \mat{W} \vec{f}.
\end{align}
Then as before, taking
\begin{align}
    \mathbb{E}[\vec{w}\vec{w}^\top] = \mat{W}^{-1}
\end{align}
and following the remaining steps in Appendix \ref{app:arbitrary-cell-weights-matern}, we arrive at the expression for the weighted reaction-diffusion kernel
\begin{align}
    \mat{K} = \mat{U} \left(r\mat{I} - c\mat{\Lambda} + d \mat{\Lambda}^2\right)^{-\nu} \mat{U}^\top.
\end{align}
This is essentially the same as in the non-weighted case, except $\mat{\Lambda}, \mat{U}$ are now the eigenpairs of the weighted Dirac operator \eqref{eq:weighted-dirac}, the latter being orthonormal under the weighted $L^2$-inner product.

\subsubsection{Amplitude of the process}
As a special case, given a positive constant $\sigma > 0$, let us consider the weights $w^k_\alpha = \sigma^{-2}$ for all $k, \alpha$. That is,
\begin{align}
    \mat{W} = \sigma^{-2}\mat{I}.
\end{align}
In this case, notice that conditions \eqref{eq:orthonormality-1} and \eqref{eq:orthonormality-2} become
\begin{align}
    \mat{U}^\top \mat{U} = \mat{U} \mat{U}^\top = \sigma^2 \mat{I}.
\end{align}
Thus, the normalised eigenbasis $\widehat{\mat{U}} := \sigma^{-1}\mat{U}$ is orthonormal under the standard $L^2$-inner product, i.e.,
\begin{align}
\widehat{\mat{U}}^\top \widehat{\mat{U}} = \widehat{\mat{U}} \widehat{\mat{U}}^\top = \mat{I}.
\end{align}
Under this basis, the expressions for the Matérn and reaction-diffusion kernels read
\begin{align}
    \mat{K}_{\text{Matérn}} &= \sigma^2 \,\widehat{\mat{U}} \left(\frac{2\nu}{\ell^2}\mathbf{I} + \mathbf{\Lambda}^2\right)^{-\nu} \widehat{\mat{U}}^\top, \\
    \mat{K}_{\text{r.d.}} &= \sigma^2 \,\widehat{\mat{U}} \left(r\mat{I} - c\mat{\Lambda} + d \mat{\Lambda}^2\right)^{-\nu} \widehat{\mat{U}}^\top.
\end{align}
The extra parameter $\sigma$ controls the {\em amplitude} of the process
\begin{align}
    \mathrm{Var}(f_i) = [\mat{K}]_{ii} =  \sigma^2 c_i,
\end{align}
for $c_i = [\widehat{\mat{U}}\Phi(\mat{\Lambda})\widehat{\mat{U}}^\top]_{ii}$, where $\Phi(\mat{\Lambda}) = \left(\frac{2\nu}{\ell^2}\mathbf{I} + \mathbf{\Lambda}^2\right)^{-\nu}$ in the case of the Matérn kernel and $\Phi(\mat{\Lambda}) = \left(r\mat{I} - c\mat{\Lambda} + d \mat{\Lambda}^2\right)^{-\nu}$ in the case of the reaction-diffusion kernel.
This can be introduced as an extra hyperparameter in the model to fit the data more appropriately, which is recommended to obtain better results.

\section{Experimental Details}\label{app:experiments}

In this paper, all Gaussian processes (graph Matérn GP, CC-GP and RD-GP) are implemented using the GPJax library \citep{Pinder2022}. The objective function is the conjugate marginal log-likelihood and the optimiser is an implementation of Adam from Optax \citep{jax2018github} with a learning rate set at $0.1$.

\subsection{Directed Edge Prediction}\label{app:edge-prediction-experiment}

This experiment compares our CC-GP on edges (Matérn-CC kernel) and the graph Matérn kernel \citep{borovitskiy2020matern}. The task is to predict the edge flow constructed from the geostrophic current around the southern tip of Africa. Here, the geostrophic current refers to the dominant component of the ocean current derived by balancing the pressure gradient with the Coriolis effect. The unprocessed data is retrieved from the \cite{noaa} database, which comes in the form of two scalar fields: one representing the $x$-component and the other the $y$-component of the geostrophic current vector field. The current around the southern tip of Africa is then extracted (lat = $[-45.0, -15.0]$, lon = $[20.0, 53.1]$) and its components are rescaled to a 2D grid of dimension $20\times 20$ (see Figure \ref{fig:geostrophic_field}).
\begin{figure}[ht]
    \centering
    \begin{subfigure}[t]{.3\textwidth}
        \centering  \includegraphics[width=.9\linewidth]{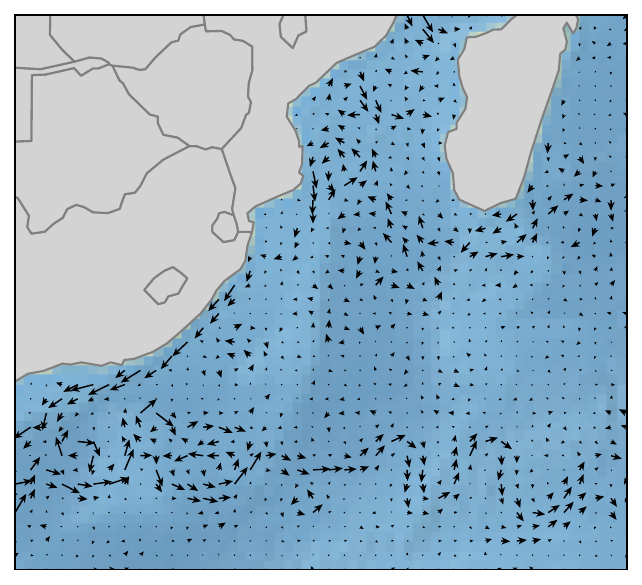}
        \caption[]%
        {{Geostrophic current}}    
    \label{fig:geostrophic_current}
    \end{subfigure}
    \centering
    \begin{subfigure}[t]{.3\textwidth}
        \centering  \includegraphics[width=.9\linewidth]{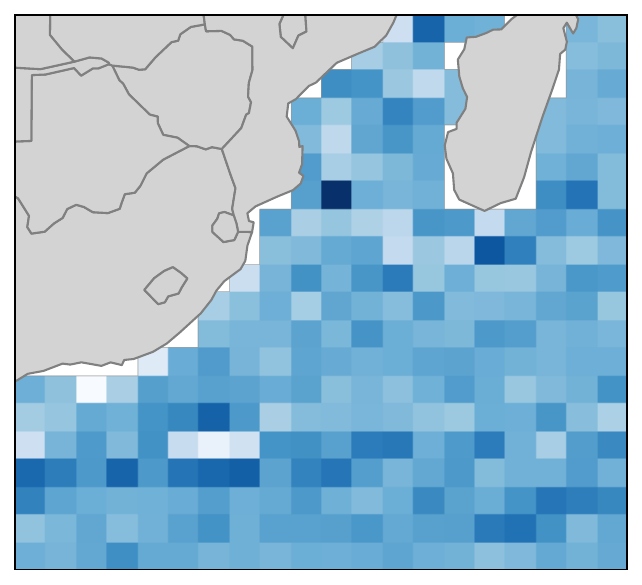}
        \caption[Network2]%
        {{$x$-component (coarsened)}}    
    \label{fig:geostrophic_current_x}
    \end{subfigure}%
    \begin{subfigure}[t]{.3\textwidth} 
        \centering 
\includegraphics[width=.9\linewidth]{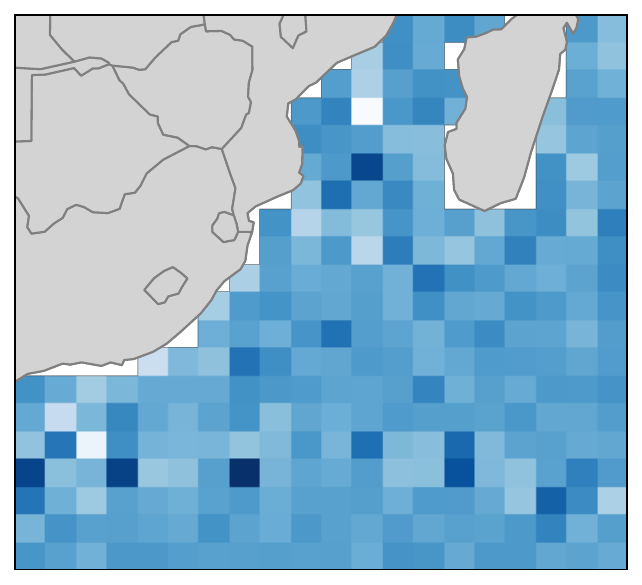}
        \caption[]%
        {{$y$-component (coarsened)}}    
        \label{fig:geostrophic_field_y}
    \end{subfigure}
    \caption
    {The geostrophic current around the southern tip of Africa. (Left) Quiver plot of the geostrophic current, (Middle) The $x$-component of the geostrophic current coarsened to a $20 \times 20$ grid, (Right) The $y$-component of the geostrophic current coarsened to a $20 \times 20$ grid.}
    \label{fig:geostrophic_field}
\end{figure}

The next step in the pre-processing is to transform this data into edge signals of a cubical 1-complex. We adopt the method in \cite{desbrun2006discrete} to generate these signals. To do so, a cubical mesh of the same resolution as the data ($20 \times 20$) is first generated, where each edge $e$ in the mesh is assigned an orientation. Here, the orientation is represented by a unit vector $\hat{\vec{t}}_e$ pointing from one endpoint to the other.
Then for each edge $e$, we compute how much of the geostrophic current flows along $e$ in the direction specified by its orientation.

 \begin{figure}[ht]
    \centering
    \begin{subfigure}[t]{.3\textwidth}
        \centering  \includegraphics[width=.8\linewidth]{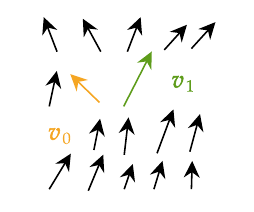}
        \caption[Network2]%
        {Two vectors of a vector field.}    
    \label{fig:two_scalar_fields}
    \end{subfigure}%
    \begin{subfigure}[t]{.4\textwidth} 
        \centering 
\includegraphics[width=.9\linewidth]{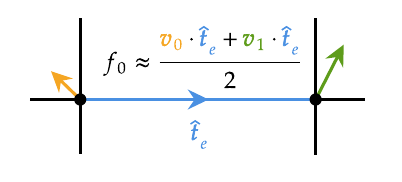}
        \caption[]%
        {{Directed edge signals}}    
        \label{fig:oriented_edge_signals}
    \end{subfigure}
    \caption
    {The construction of the edge signals.}
    \label{fig:edge_signals_construction}
\end{figure}
More precisely, the value $f_0$ on the edge $e$ (illustrated in Figure \ref{fig:edge_signals_construction}) is computed according to
\begin{equation}
    f_0 = \int_e \vec{v}(s) \cdot \hat{\vec{t}}_e \,\mathrm{d}s \approx \frac{\vec{v}_0 \cdot \hat{\vec{t}}_e + \vec{v}_1 \cdot \hat{\vec{t}}_e}{2},
\end{equation}
where $s : [0,1] \rightarrow e$ is a parameterisation of the edge $e$ and $\vec{v}$ is the geostrophic current.
 This yields directed edge signals on a cubical mesh, where we use the usual rule of setting the direction to be aligned with the orientation of $e$ if $f_0$ is positive and opposite to it if $f_0$ is negative. 
 The training data is obtained by randomly selecting 30\% of the generated edge signals and adding i.i.d. noise from a Gaussian $\mathcal{N}(0, 10^{-4})$.

For the training of the graph Matérn GP, the smoothness hyperparameter $\nu$ is fixed at $2$. 
The amplitude and lengthscale hyperparameters $\sigma^2$, $\ell$ are both initialised at $1.0$ and optimised for $1000$ iterations using Adam. The training took less than 30 seconds on a MacBook Pro with M1 chip. In a similar way, when training CC-Matérn GP on edges, the smoothness hyperparameter $\nu$ is set to $2$, and the amplitude and lengthscale hyperparameters $\sigma^2$, $\ell$ are initialised at $1.0$, before optimising them for $1000$ iterations using Adam. The training for this model also takes less than 30 seconds. 

\subsection{Signal Mixing}\label{app:signal-mixing-experiment}

This experiment compares the performance of the RD-GP and the Mat\'ern CC-GP in the task of predicting signals on the vertices, edges and triangles of a 2D simplicial mesh. The mesh is constructed by first defining a $10 \times 10$ grid, then subdividing this grid into triangles to transform it into a 2D simplicial mesh. The resulting complex is composed of 523 simplices: 100 vertices, 261 edges and 162 triangles.

The signals on the edges are created by taking inspiration from the Karhunen-Loève theorem, which states that a stochastic process can be expressed as a linear combination of $L^2$-orthogonal basis functions with random coefficients (one may view this as a stochastic analogue of the Fourier expansion). Here, the orthogonal basis functions are the set of eigenfunctions $\{u_i\}_i$ of the Hodge Laplacian $\Delta_1$. The orthogonality of the eigenfunctions is ensured by the symmetry of the operator $\Delta_1$. This forms a basis for edge signals (i.e. $1$-cochains) that encodes the topology of the mesh through the information contained in $\Delta_1$. For the coefficients in the basis expansion, we use i.i.d. Gaussians $\xi_i \sim \mathcal{N}(0, \lambda_i^{-1})$, where $\lambda_i$ is the eigenvalue of $\Delta_1$ corresponding to $u_i$. This expansion is truncated to lie between $0 < k < K$, which represent the minimal and maximal wavenumbers controlling the smoothness of the edge field. Putting this together yields the random $1$-cochain
\begin{equation}
    f = \sum_{i=k}^K \xi_i u_i, \quad \xi_i \sim \mathcal{N}(0, \lambda_i^{-1}).
\end{equation}
Once the signals on the edges are obtained, the signals on the vertices and the triangles are computed by applying the coboundary operator $d_1$ and its adjoint $d^*_1$ to $f$, respectively. Using the numerical representation of cochains, $\{\vec{u}_i\}_i$ becomes the set of eigenvectors of the Hodge Laplacian matrix $\boldsymbol{\Delta}_1$, the coboundary operator becomes the matrix $\mat{B}_1$ and its adjoint becomes $\mat{B}_2^\top$ (see Appendix \ref{app:numerical-representation}). The vertex signals and the triangle signals are obtained by computing $\mat{B}_1 \vec{f} $ and $\mat{B}_2^\top \vec{f}$, respectively. An example signal for $k=20$ and $K=100$ is displayed in Figure \ref{fig:synthetic_signals}, which we use as the ground truth in our experiment. 

\begin{figure}[ht]
    \centering
    \begin{subfigure}[t]{.3\textwidth}
        \centering  \includegraphics[width=.9\linewidth]{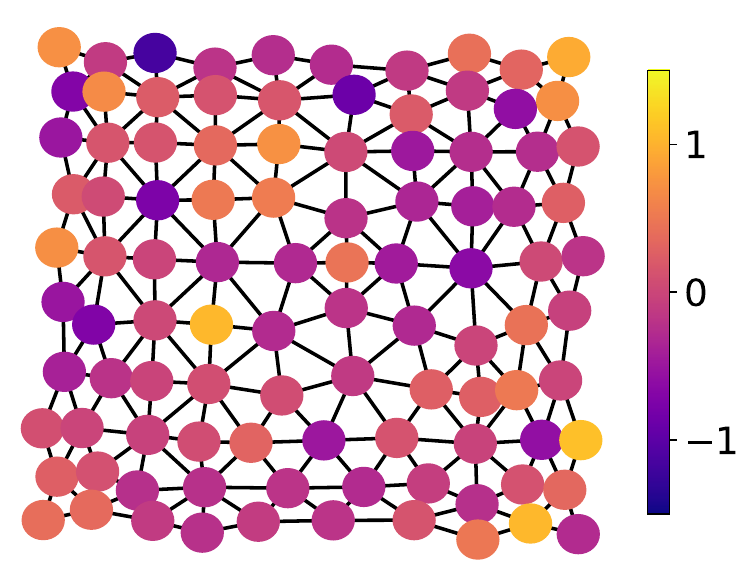}
        \caption[Network2]%
        {{Vertex signals}}    
    \label{fig:vertex_signals}
    \end{subfigure}%
    \begin{subfigure}[t]{.3\textwidth} 
        \centering 
\includegraphics[width=.9\linewidth]{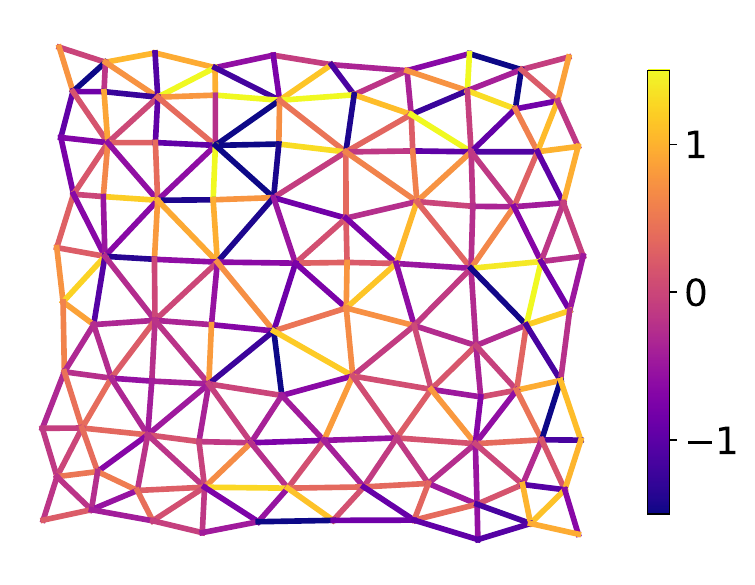}
        \caption[]%
        {{Edge signals}}    
        \label{fig:synthetic_edge}
    \end{subfigure}
        \begin{subfigure}[t]{.3\textwidth} 
        \centering 
\includegraphics[width=.9\linewidth]{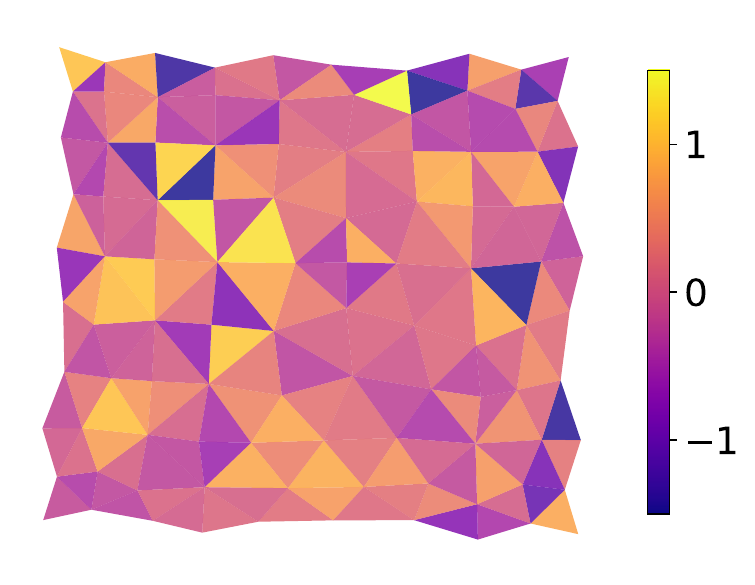}
        \caption[]%
        {{Triangle signals }}    
        \label{fig:synthetic_triangle}
    \end{subfigure}
    \caption
    {An example synthetic signal on the vertices, edges and triangles.}
    \label{fig:synthetic_signals}
\end{figure}

The training data is generated by randomly selecting a third of the vertices, a third of the edges and a third of the triangles from this ground truth field, and corrupting them by i.i.d. noise from a Gaussian $\mathcal{N}(0, 10^{-2})$.

For training the CC-Matérn GP, the smoothness hyperparameter $\nu$ is fixed at $2$, and the amplitude and lengthscale hyperparameters $\sigma^2$, $\ell$ are both initialised at $1.5$, before optimising them for 1000 iterations using Adam.  The training takes less than a minute on a MacBook Pro equipped with a M1 Pro chip. The training of RD-GP is similar: The smoothness hyperparameter $\nu$ is fixed at $2$, and the amplitude hyperparameter $\sigma^2$, the reaction coefficient $r$, the diffusion coefficient $d$, and the cross-diffusion coefficient $c$ are all initialised at $1.5$. They are then optimised for 1000 iterations using Adam, again taking less than a minute to run.

\subsection{Modelling Electromagnetism}\label{app:electromagnetic-experiment}
In this experiment, we compare the performance of the RD-GP and the Mat\'ern GP on imputing signals on the vertices, edges and faces of a $20 
\times 20$ square lattice. The signals come from simulations of electromagnetic fields. In particular, we used the Python package \texttt{PyCharge} to generate 2D electromagnetic fields on a square domain, generated by $10$ oscillating point charges at randomly generated locations. The fields that were computed were the scalar potential ($V$), electric field ($E$) and the magnetic field ($B$).
Physically, these are a scalar field, a vector field and a two-form / pseudovector field (i.e., a field of vectors whose sign depends on the orientation of the manifold), respectively. An example of such fields is displayed in Figure \ref{fig:electromagnetic-fields}.

\begin{figure}[ht]
    \centering
    \begin{subfigure}[t]{.3\textwidth}
        \centering  \includegraphics[width=.7\linewidth]{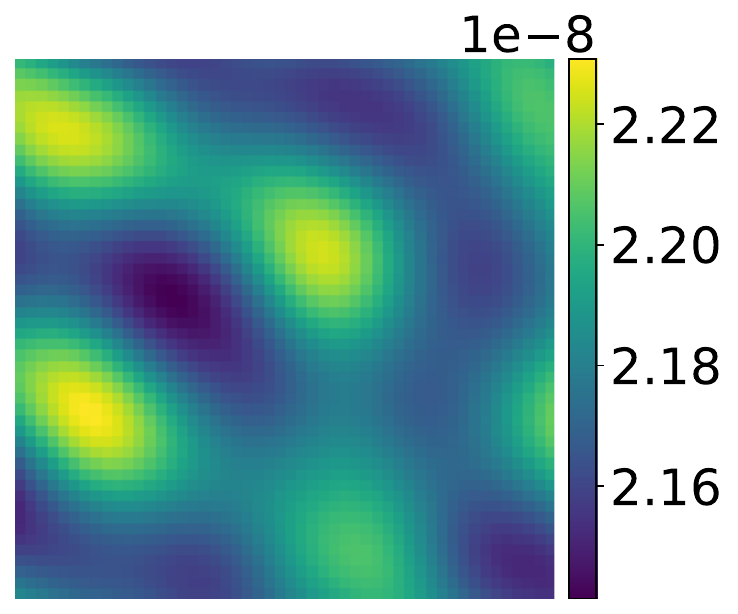}
        \caption[Network2]%
        {{Scalar potential}}    
    \end{subfigure}%
    \begin{subfigure}[t]{.3\textwidth} 
        \centering 
\includegraphics[width=.7\linewidth]{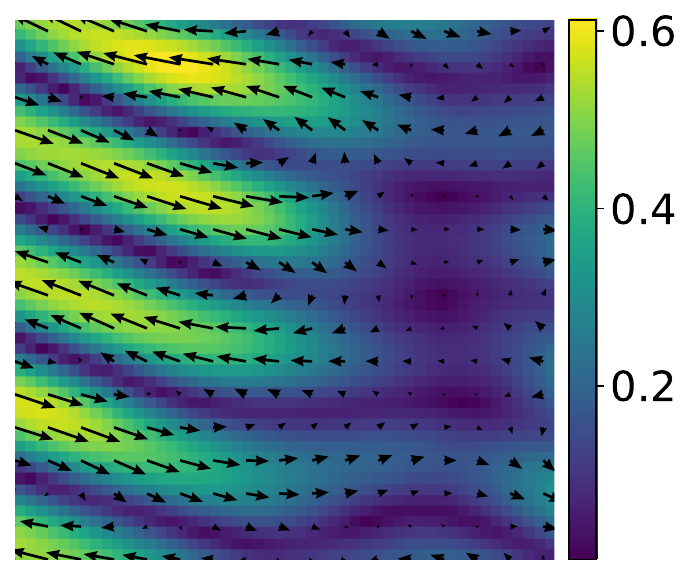}
        \caption[]%
        {{Electric field}}    
    \end{subfigure}
        \begin{subfigure}[t]{.3\textwidth} 
        \centering 
\includegraphics[width=.7\linewidth]{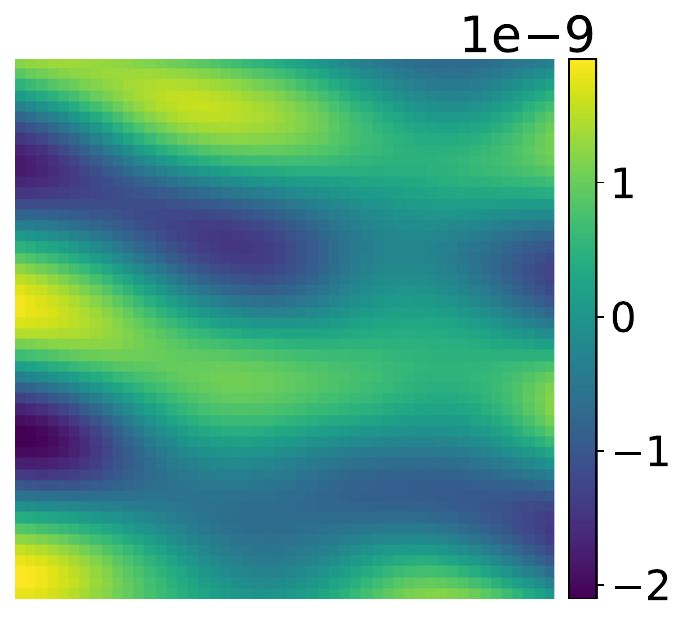}
        \caption[]%
        {{Magnetic flux into page}}    
    \end{subfigure}
    \caption{We plot an example scalar potential, electric field and magnetic field generated from ten randomly sampled oscillating point charges. For the electric field, we display the amplitudes of the vectors in colour in the background. For the magnetic field, we plot the magnetic flux going into the page, which becomes a scalar field.}
    \label{fig:electromagnetic-fields}
\end{figure}

The next step involves projecting these fields onto a cellular complex of dimension two, given by a $20 \times 20$ square lattice. Projecting the scalar potential on the vertices of a square lattice involves just extracting the point values of the field at the vertex locations.
To project the electric field onto the edges of the lattice, the procedure is similar to that described in Appendix \ref{app:edge-prediction-experiment}. Finally, projecting the magnetic field onto the square faces of the lattice involves averaging  the magnetic flux (i.e., $B \cdot \hat{n}$) over the square cells, where the unit normal $\hat{n}$ is given by the normal vector determining the orientation of the cell (see Appendix \ref{app:orientation}). We also normalise the projected values, due to the large discrepancies of magnitudes between the different fields.
The final projections of the fields in Figure \ref{fig:electromagnetic-fields} onto the cells of a square lattice are displayed in Figure \ref{fig:electromagnetic-fields-discretised}.

\begin{figure}[ht]
    \centering
    \begin{subfigure}[t]{.3\textwidth}
        \centering  \includegraphics[width=.7\linewidth]{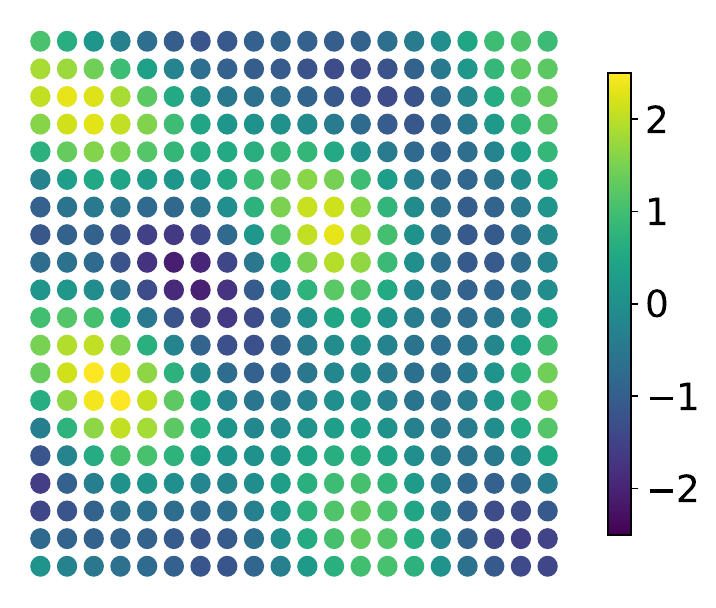}
        \caption[Network2]%
        {{Scalar potential}}    
    \end{subfigure}%
    \begin{subfigure}[t]{.3\textwidth} 
        \centering 
\includegraphics[width=.7\linewidth]{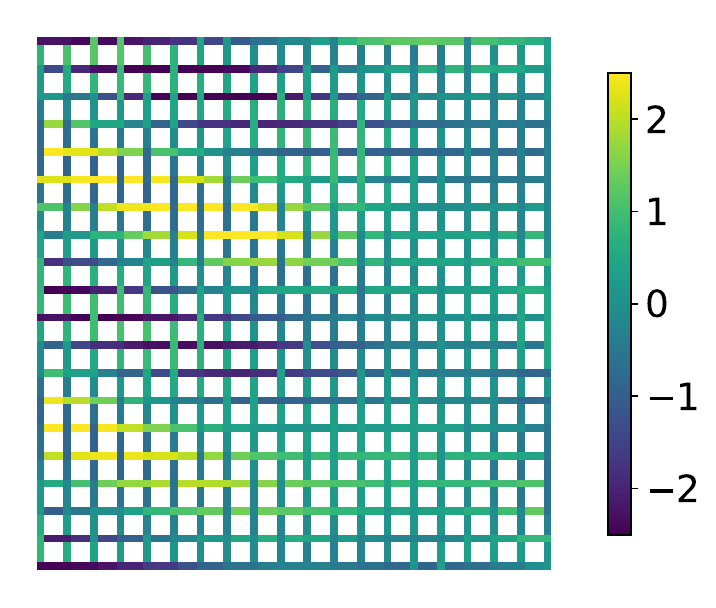}
        \caption[]%
        {{Electric field}}    
    \end{subfigure}
        \begin{subfigure}[t]{.3\textwidth} 
        \centering 
\includegraphics[width=.7\linewidth]{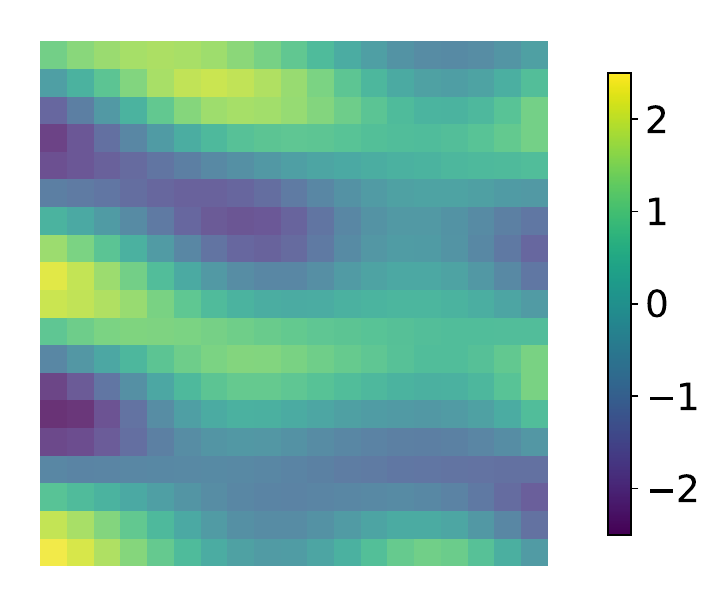}
        \caption[]%
        {{Magnetic field}}    
    \end{subfigure}
    \caption{Discrete representations of the scalar potential, electric field and magnetic field as $0$, $1$ and $2$-cochains of a square lattice respectively. Each square cell is assigned clockwise orientation.}
    \label{fig:electromagnetic-fields-discretised}
\end{figure}

The training data is generated by randomly selecting a sixth of the vertices, edges and square faces from the projected fields and adding i.i.d Gaussian noise with standard deviation of $10^{-2}$.

For training the CC-Matérn GP, the smoothness hyperparameter $\nu$ is fixed at $2$, and the amplitude and lengthscale hyperparameters $\sigma^2$, $\ell$ are both initialised at $1.5$, before optimising them for 1000 iterations using Adam.  The training takes less than a minute on a MacBook Pro equipped with a M1 Pro chip. The training of RD-GP is similar: The smoothness hyperparameter $\nu$ is fixed at $2$, the amplitude hyperparameter $\sigma^2$, the reaction coefficient $r$ and the diffusion coefficient $d$ are initialised at $1.5$. The cross-diffusion coefficient $c$ is initialised at $2.5$. They are then optimised for 1000 iterations using Adam, again taking less than a minute to run.

The predictions made by the RD-GP is displayed in Figure \ref{fig:electromagnetism-rd-preds} and those made by the CC-Mat\'ern GP is displayed in Figure \ref{fig:electromagnetism-matern-preds}. We see that both GPs recover the ground truth field (Figure \ref{fig:electromagnetic-fields-discretised}) fairly accurately from the observations. While the metrics indicate that the RD-GP output is slightly better than those of CC-Mat\'ern (Table \ref{fig:electromagnetic_example_results_table}), perceptually, the differences are too small to see.

\begin{figure}[ht]
    \centering
    \begin{subfigure}[t]{.3\textwidth}
        \centering  \includegraphics[width=.7\linewidth]{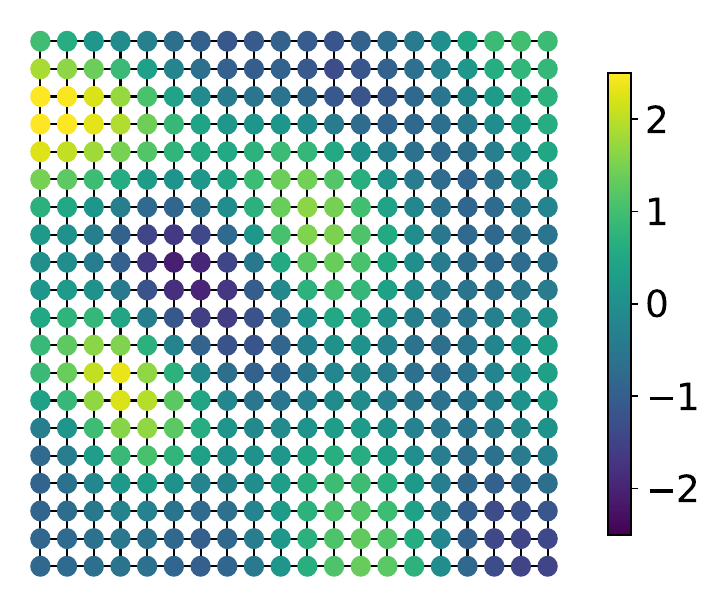}   
    \end{subfigure}%
    \begin{subfigure}[t]{.3\textwidth} 
        \centering 
\includegraphics[width=.7\linewidth]{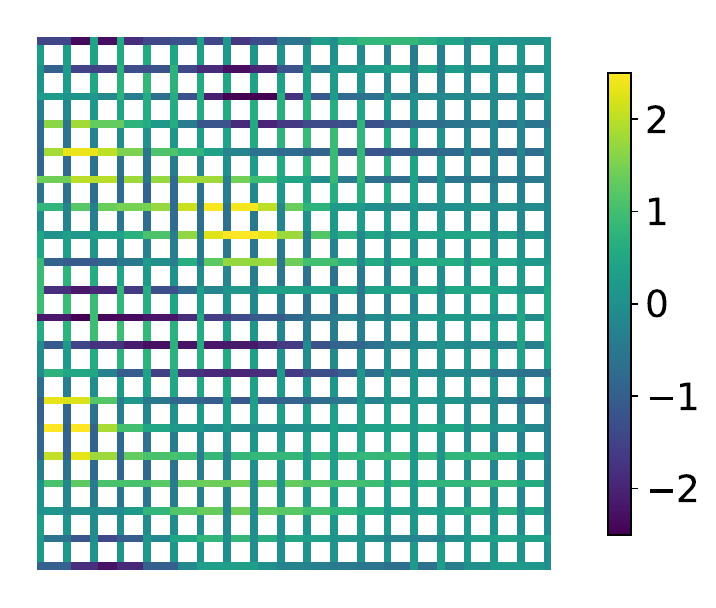}   
    \end{subfigure}
        \begin{subfigure}[t]{.3\textwidth} 
        \centering 
\includegraphics[width=.7\linewidth]{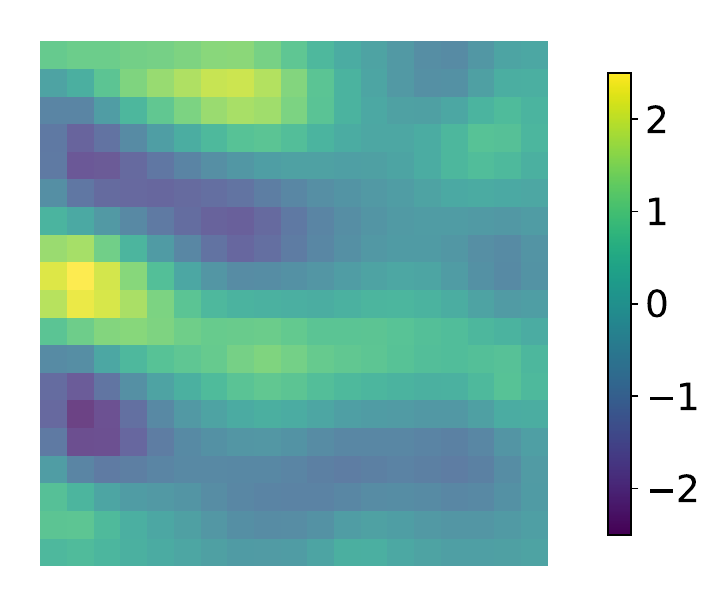}  
    \end{subfigure}
    \begin{subfigure}[t]{.3\textwidth}
        \centering  \includegraphics[width=.7\linewidth]{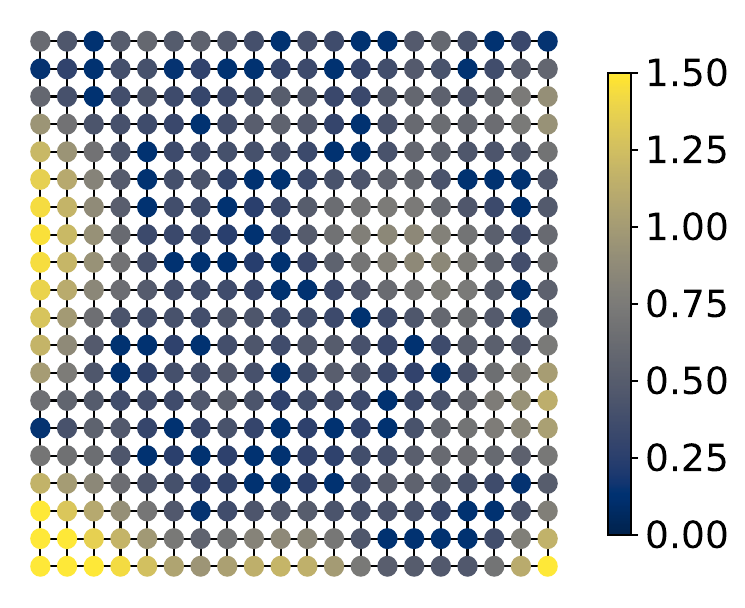}
        \caption[Network2]%
        {{Scalar potential}}    
    \end{subfigure}%
    \begin{subfigure}[t]{.3\textwidth} 
        \centering 
\includegraphics[width=.7\linewidth]{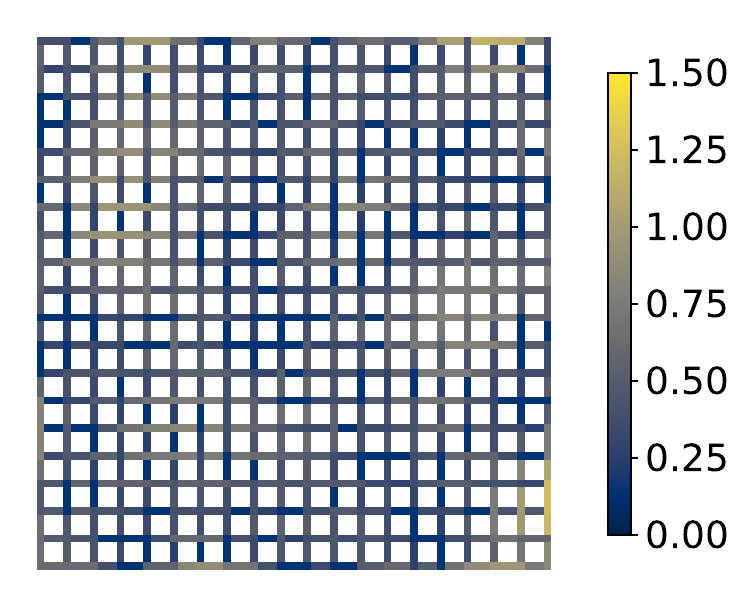}
        \caption[]%
        {{Electric field}}    
    \end{subfigure}
        \begin{subfigure}[t]{.3\textwidth} 
        \centering 
\includegraphics[width=.7\linewidth]{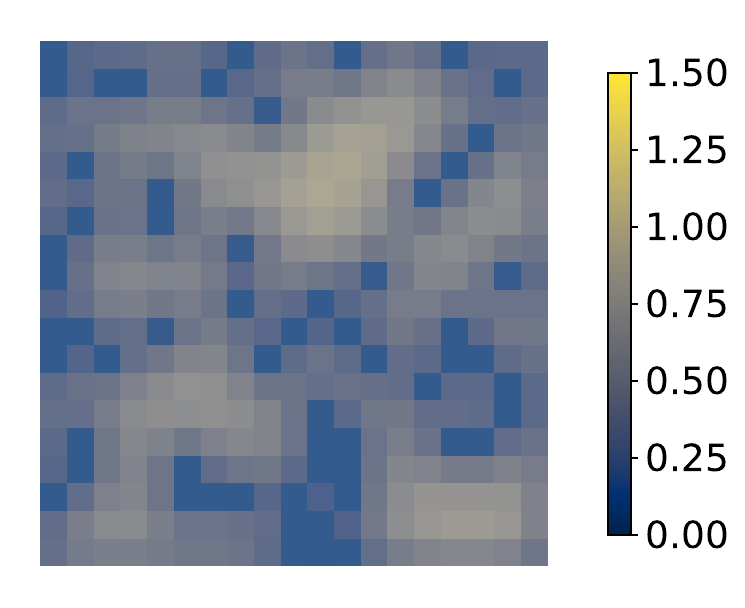}
        \caption[]%
        {{Magnetic field}}    
    \end{subfigure}
    \caption{Predictions of the scalar potential, electric field and magnetic field made from the reaction-diffusion GP. The top row displays the predictive mean and the bottom row displays the standard deviations.}
    \label{fig:electromagnetism-rd-preds}
\end{figure}

\begin{figure}[ht]
    \centering
    \begin{subfigure}[t]{.3\textwidth}
        \centering  \includegraphics[width=.7\linewidth]{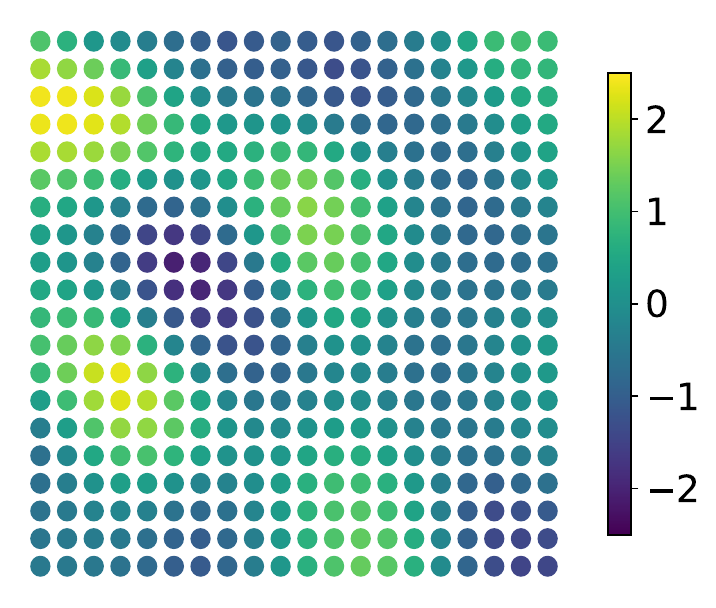} 
    \end{subfigure}%
    \begin{subfigure}[t]{.3\textwidth} 
        \centering 
\includegraphics[width=.7\linewidth]{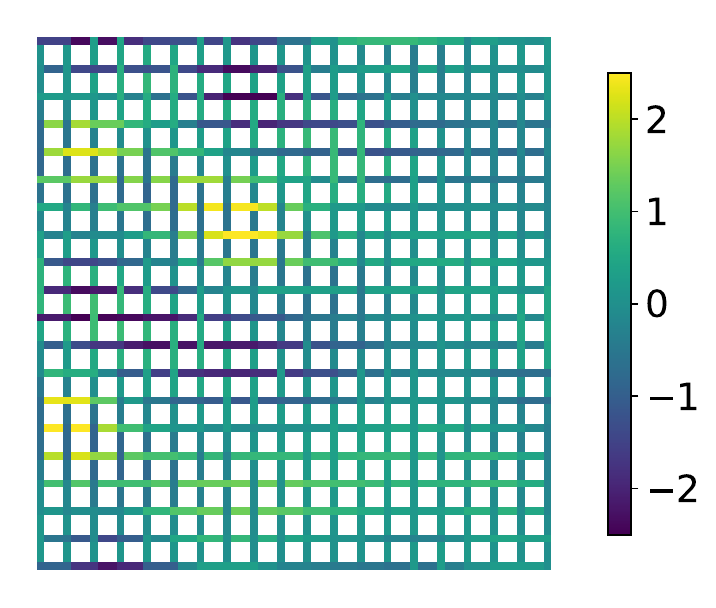}
    \end{subfigure}
        \begin{subfigure}[t]{.3\textwidth} 
        \centering 
\includegraphics[width=.7\linewidth]{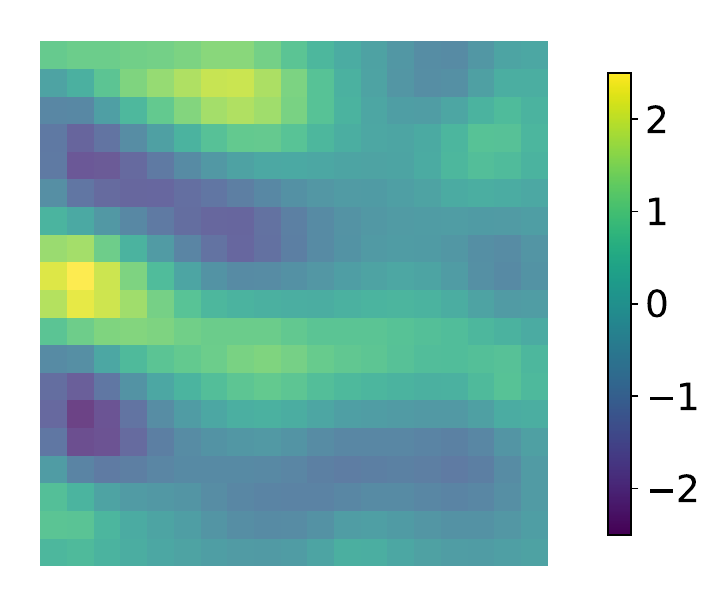}  
    \end{subfigure}
    \begin{subfigure}[t]{.3\textwidth}
        \centering  \includegraphics[width=.7\linewidth]{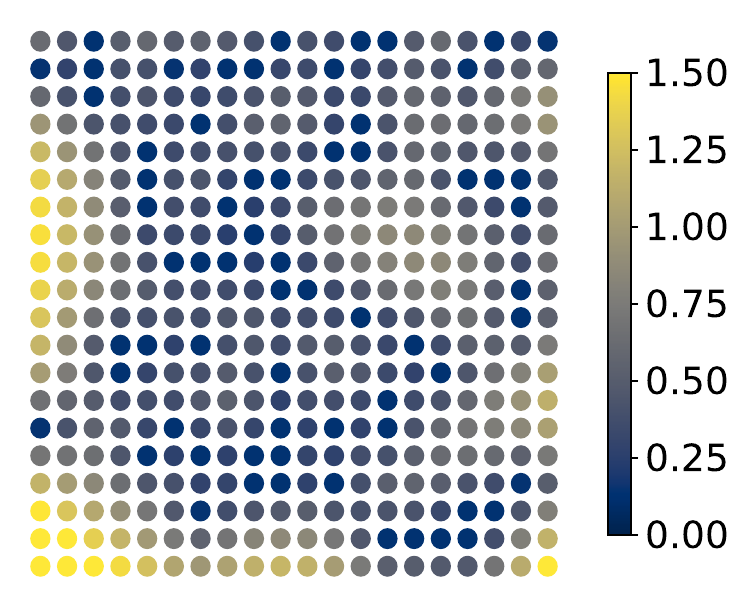}
        \caption[Network2]%
        {{Scalar potential}}    
    \end{subfigure}%
    \begin{subfigure}[t]{.3\textwidth} 
        \centering 
\includegraphics[width=.7\linewidth]{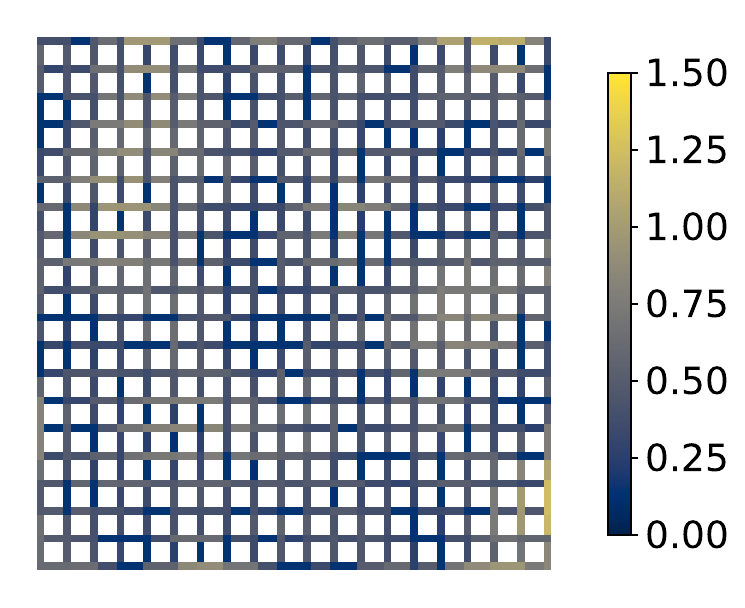}
        \caption[]%
        {{Electric field}}    
    \end{subfigure}
        \begin{subfigure}[t]{.3\textwidth} 
        \centering 
\includegraphics[width=.7\linewidth]{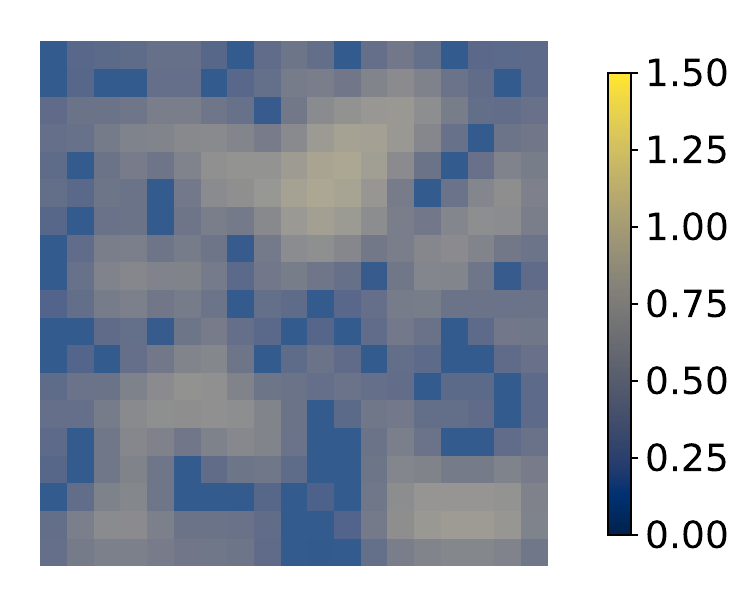}
        \caption[]%
        {{Magnetic field}}    
    \end{subfigure}
    \caption{Predictions of the scalar potential, electric field and magnetic field made from the CC-Mat\'ern GP. The top row displays the predictive mean and the bottom row displays the standard deviations.}
    \label{fig:electromagnetism-matern-preds}
\end{figure}

\begin{table}[t]
    \centering
    \begin{tabular}{l r r} 
       {MSE ($\downarrow$)}  & CC-Mat\'ern & Reaction-diffusion\\
       \midrule
       Scalar potential & 0.113 & {\bf 0.110} \\ 
     Electric field & 0.125 & {\bf 0.108} \\
     Magnetic field & 0.151 & {\bf 0.137} \\
     \hline
     \rule{0pt}{3ex} NLL ($\downarrow$)
        \\
     \midrule
     Scalar potential &  70.8 & {\bf 68.5} \\ 
     Electric field & 110.3 & {\bf 92.3} \\
     Magnetic field & 118.4 & {\bf 110.1} \\
    \end{tabular}
    \caption{Mean square error (MSE) and negative log-likelihood (NLL) of predictions of the electromagnetic fields in Figure \ref{fig:electromagnetism-rd-preds} (RD-GP) and Figure \ref{fig:electromagnetism-matern-preds} (CC-Mat\'ern GP). The performance of the reaction-diffusion GP is slightly better than the Matérn GP on the cellular complex, suggesting that mixing on this example has some positive impact on the predictions.}
    \label{fig:electromagnetic_example_results_table}
\end{table}

\end{document}


%

%

\onecolumn
\aistatstitle{Instructions for Paper Submissions to AISTATS 2024: \\
Supplementary Materials}

\section{Cell Orientations and Boundaries}

\section{Numerical Representation of Cellular Complexes}
To make computations explicit, we wish to represent the above concepts using matrices and vectors. Fortunately, this is not difficult as the space of chains / cochains forms a free Abelian group / vector space, which is isomorphic to $\mathbb{Z}^n$ / $\mathbb{R}^n$.

To this end, we fix a labelling $\alpha \mapsto e^k_\alpha$ of the $k$-cells comprising a cellular complex $X$, which forms an ordered basis $(e^k_1, \ldots, e^k_{N_k})$. Then, an arbitrary $k$-chain $c = \sum_{\alpha=1}^{N_k} n_\alpha e^k_\alpha \in C_k(X)$ may be represented by a vector $\boldsymbol{c} = (n_1, \ldots, n_{N_k})^\top$ in $\mathbb{Z}^{N_k}$. Similarly, a $k$-cochain $f = \sum_{\alpha=1}^{N_k} f_\alpha (e^k_\alpha)^* \in C^k(X)$ can be represented by a vector $\boldsymbol{f} = (f_1, \ldots, f_{N_k})^\top$ in $\mathbb{R}^{N_k}$. Under this representation, cochain evaluation \eqref{eq:cochain-def} can simply be expressed as a dot product $f(c) = \boldsymbol{f}^\top \boldsymbol{c} \in \mathbb{R}$.

Next, we consider the boundary and coboundary operators. The boundary operator can be expressed as a signed incidence matrix $\mathbf{B}_k : \mathbb{Z}^{N_k} \rightarrow \mathbb{Z}^{N_{k-1}}$, whose $j$-th column corresponds to the vector representation of the cell boundary $\partial e^k_j$ (see Figure \ref{fig:chains}).
Similarly, the coboundary operator can be represented by a matrix $\mathbf{D}_k : \mathbb{R}^{N_k} \rightarrow \mathbb{R}^{N_{k+1}}$. Using \eqref{eq:coboundary-def}, we have
\begin{align}
    \boldsymbol{f}^\top \mathbf{D}_k^\top \boldsymbol{c} = \boldsymbol{f}^\top \mathbf{B}_{k+1} \boldsymbol{c}, \quad \Leftrightarrow \quad \mathbf{D}_k = \mathbf{B}_{k+1}^\top.
\end{align}
Thus, the coboundary operator is identified with the transpose of the incidence matrix. Finally, let $\mathbf{W}_k = \mathtt{diag}(w_1^k, \ldots, w_{N_k}^k)$ be the weight matrix defining the $L^2$-inner product \eqref{eq:k-cochain-inner-product}, i.e.,
\begin{align}
    \left<f, g\right>_{L^2(\vec{w}^k)} = \boldsymbol{f}^\top\mathbf{W}_k \boldsymbol{g}.
\end{align}
Then, letting $\mathbf{D}^*_k : \mathbb{R}^{N_{k+1}} \rightarrow \mathbb{R}^{N_k}$ be the matrix representation of the adjoint of the coboundary,  \eqref{eq:codifferential-def} implies
\begin{align}
    \boldsymbol{f}^\top (\mathbf{D}_k^*)^\top \mathbf{W}_{k} \boldsymbol{g} = \boldsymbol{f}^\top \mathbf{W}_{k+1} \mathbf{B}_{k+1}^\top \boldsymbol{g} \\
    \quad \Leftrightarrow \quad \mathbf{D}_k^* = \mathbf{W}_{k}^{-1} \mathbf{B}_{k+1} \mathbf{W}_{k+1}.
\end{align}

Putting this together, we find the matrix expression $\mathbf{\Delta}_k : \mathbb{R}^{N_k} \rightarrow \mathbb{R}^{N_k}$ for the Hodge-Laplacian operator:
\begin{align}
    &\mathbf{\Delta}_k = \mathbf{D}_{k} \mathbf{D}^*_{k} + \mathbf{D}^*_{k+1} \mathbf{D}_{k+1} \\
    &= \mathbf{B}_{k}^\top (\mathbf{W}_{k-1}^{-1} \mathbf{B}_{k} \mathbf{W}_{k}) + (\mathbf{W}_{k}^{-1} \mathbf{B}_{k+1} \mathbf{W}_{k+1}) \mathbf{B}_{k+1}^\top. \label{eq:hodge-laplace-matrix-general}
\end{align}
We observe that the matrix \eqref{eq:hodge-laplace-matrix-general} is symmetric if and only if $\mat{W}_k = \mat{I}$. We also observe that this recovers the familiar (weighted) graph Laplacian matrix.

\section{Characterisation of GPs on Cellular Complexes}
\subsection{Proof of Theorem \ref{eq:GRC-characterisation}}

\subsection{Proof of Theorem \ref{eq:CCGP-characterisation}}

\section{Rigorous treatment}

\subsection{Cellular complexes}
A regular cellular complex $X$ of finite dimension $n < \infty$ is a topological space obtained by iterating the following steps:

\begin{enumerate}
\setcounter{enumi}{-1}
    \item Start with a discrete set $X^0 = \{e^0_\alpha\}_{\alpha=1}^{N_0}$ of $0$-dimensional points (or {\em $0$-cells}), called a {\em $0$-skeleton}.
    \item Take a discrete set $\tilde{X}^1 = \{e^1_\alpha\}_{\alpha=1}^{N_1}$ of line segments (or {\em $1$-cells}). For each line segment $e^1_\alpha$ in the set, attach an endpoint of $e^1_\alpha$ to a point in $X^0$ (mathematically, this is described by the {\em attaching maps} $\phi_\alpha^1 : \partial e^1_\alpha \rightarrow X^0$). The resulting topological space $X^1 = X^0 \cup \tilde{X}^1 / \sim$, where $x \sim y \Leftrightarrow y = \phi_\alpha^1(x)$ for $x \in \partial e^1_\alpha$ is called a $1$-skeleton.
    \item Take a discrete set $\tilde{X}_2 = \{e^2_\alpha\}_{\alpha=1}^{N_2}$ of $2$d disks (or {\em $2$-cells}) and attach the boundary of each disk $e^2_\alpha$ to $X^1$ by the attaching maps $\phi_\alpha^2 : \partial e^2_\alpha \rightarrow X^1$. As before, the resulting topological space $X^2 = X^1 \cup \tilde{X}^2 / \sim$, where $x \sim y \Leftrightarrow y = \phi_\alpha^2(x)$ for $x \in \partial e_2^\alpha$ is called a $2$-skeleton.
    \item Likewise, for $k = 3, \ldots, n$, define a $k$-skeleton $X^k$ by considering a discrete set $\tilde{X}^k = \{e^k_\alpha\}_{\alpha=1}^{N_k}$ of $k$-dimensional disks ($k$-cells) and taking $X^k = X^{k-1} \cup \tilde{X}^{k} / \sim$, where the equivalence relation is defined by the attaching maps $\phi_\alpha^k : \partial e^k_\alpha \rightarrow X^{k-1}$ as before.
    \item Finally, the cellular complex is defined as a topological space $X = \bigcup_{k=1}^n X^k$, where $A \subset X$ is open if and only $A \cap X^k$ is open for all $k = 1, \ldots, n$.
\end{enumerate}

Note that the skeletons form a filtration $\emptyset \subset X^0 \subset X^1 \subset \cdots $ so in the finite setting, $X$ is homeomorphic to $X^n$.

\subsection{Chains and cochains}
Given a cellular complex $X$, a $k$-chain $C_k(X) = \{\sum_{\alpha=1}^{N_k} n_\alpha e^k_\alpha : n_\alpha \in \mathbb{Z}, \,\forall \alpha=1,\ldots,N_k\}$ on $X$ is a free Abelian group (can think of as a vector space but only admitting integer-valued scalars) whose generators are the $k$-cells in $X$. Intuitively, a $k$-chain abstracts the notion of a set of ``directed paths" on a graph to higher order cells (a set of ``directed path" in this case can be viewed as a $1$-chain).

There exists a canonical operation on $C_k(X)$ termed the {\em boundary operator} $\partial_k : C_k(X) \rightarrow C_{k-1}(X)$, $k=1, \ldots, n$. For $c_k = \sum_{\alpha} n_\alpha e^\alpha_k$, we have
\begin{align}
    \partial_k c_k = \sum_{\alpha=1}^{N_k} n_a \partial_k e_k^\alpha, \quad \text{where} \quad
    \partial_k e^\alpha_k = \sum_{\beta=1}^{N_{k-1}} \mathrm{deg} \left(\chi_k^{\alpha\beta}\right) e_{k-1}^\beta,
\end{align}
where $\text{deg}$ is the Brouwer degree of the mapping $\chi_k^{\alpha\beta} : \mathbb{S}^{k-1} \stackrel{\sim}{\rightarrow} \partial e^k_{\alpha} \stackrel{\phi_\alpha^k}{\rightarrow} X^{k-1} \stackrel{q}{\rightarrow} X^{k-1} / (X^{k-1} - e^{k-1}_\beta) \stackrel{\sim}{\rightarrow} \mathbb{S}^{k-1}$, which can be interpreted as the winding number of the boundary of $e^k_{\alpha}$ around the cell $e^{k-1}_\beta$ upon attachment.

\begin{proposition}
    The boundary operator satisfies the following relation
    \begin{align}\label{prop:boundary-of-boundary}
        \partial_{k-1} \circ \partial_k = 0.
    \end{align}
\end{proposition}

Next, we introduce the notion of a cochain, which will be essential to define our GP.
\begin{definition}[Cochain]
A {\em cochain} $f \in C^k(X)$ on $X$ is a ``linear"\footnote{The term ``linear" is not quite accurate as the underlying space is not a vector space. This should in fact be viewed as a group homomorphism $C_k(x) \rightarrow \mathbb{R}$ to be more precise.} function assigning a number to each $k$-chain. That is,
\begin{align}
    f : C_k(X) \rightarrow \mathbb{R},  \text{ such that } f\Big(\sum_{\alpha=1}^{N_k} n_\alpha e^k_\alpha\Big) = \sum_{\alpha=1}^{N_k} n_\alpha f(e_\alpha^k),
\end{align}
where the choice of $f(e_\alpha^k) \in \mathbb{R}$ for all $\alpha=1, \ldots, N_k$ defines the cochain. This itself defines a free Abelian group with generators $(e_\alpha^k)^*$ such that $(e_\alpha^k)^*e_\beta^k = \delta_{\alpha\beta}$. Hence we can represent $f \in C^k(X)$ as $f = \sum_{\alpha=1}^{N_k} f_\alpha (e_\alpha^k)^*$, where $f_\alpha := f(e_\alpha^k) \in \mathbb{R}$.
\end{definition}

\begin{remark}
In general, the co-domain of $C^k(X)$ can be taken to be any Abelian group $G$ \cite{Hatcher2001}, although we don't consider this general case here. In our case, when $G = \mathbb{R}$, the cochains also admit a vector space structure.
\end{remark}

The boundary operator on $k$-chains induces an analogous operator on the cochains, referred to as the co-boundary operator.
\begin{definition}[Coboundary operator]
    The co-boundary operator $d_k : C^k(X) \rightarrow C^{k+1}(X)$ for $k=0, \ldots, n$ is defined by the relation
    \begin{align} \label{eq:coboundary-def}
        d_k f(c) = f(\partial_{k+1} c), \text{ for all } f \in C^k(X) \text{ and } c \in C_{k+1}(X),
    \end{align}
    for $k < n$ and $d_k f \equiv 0$ for $k = n$.
\end{definition}

Now, given some positive ``weights" $w_\alpha^k : e^k_\alpha \rightarrow \mathbb{R}_{>0}$ on each $k$-cell, one can define an inner-product on the cochains. For $f, g \in C^k(X)$, we define the inner-product
\begin{align} \label{eq:k-cochain-inner-product}
    \left<f, g\right>_{L^2(w^k)} := \sum_\alpha w_\alpha^k \,f(e^k_\alpha) \,g(e^k_\alpha).
\end{align}
This allows us to define an operator $d_k^* : C^{k+1}(X) \rightarrow C^{k}(X)$ as follows.

\begin{definition}[Codifferential operator]
    Given a family of inner-products $\left<\cdot, \cdot \right>_{L^2(w^k)}$ on the space of $k$ cochains for all $k = 0, \ldots, n$, we define the co-differential operator $d_k^* : C^{k+1}(X) \rightarrow C^{k}(X)$ by
    \begin{align} \label{eq:codifferential-def}
        \left<d_k^* f, g\right>_{L^2(w^{k})} = \left<f, d_{k}g\right>_{L^2(w^{k+1})},
    \end{align}
    for any $f \in C^{k+1}(X)$ and $g \in C^{k}(X)$.
\end{definition}

The coboundary and codifferential opertors satisfy the following key property.
\begin{proposition}
    The co-boundary and co-differential operators satisfy the following relation.
    \begin{align}\label{prop:coboundary-of-coboundary}
        d_{k+1} \circ d_k = 0, \qquad d_{k}^* \circ d_{k+1}^* = 0.
    \end{align}
\end{proposition}

Finally, we define the Hodge-Laplacian operator on the cochains as follows.
\begin{definition}[Hodge Laplacian]\label{def:hodge-laplacian}
    The Hodge-Laplacian $\Delta_k : C^k(X) \rightarrow C^k(X)$ is defined as
    \begin{align}\label{eq:hodge-laplacian}
        \Delta_k = d_{k-1} \circ d_{k-1}^* + d_{k}^* \circ d_{k}.
    \end{align}
\end{definition}

\subsection{Gaussian processes on cellular complexes}
Here we define our notion of Gaussian processes on a cellular complex, which may be viewed as a direct sum of random cochains satisfying some notion of Gaussianity. We make this notion precise by first discussing what we mean by a Gaussian random cochain.

\begin{definition}[Gaussian random cochain]
    Let $(\Omega, \mathcal{F}, \mathbb{P})$ be a probability space. A random variable $f_\bullet : \Omega \rightarrow C^k(X)$ is called a {\em Gaussian random cochain} if for any chain $c \in C_k(X)$, the random variable $f_\bullet (c) : \Omega \rightarrow \mathbb{R}$ is a univariate Gaussian random variable.
\end{definition}

To facilitate computations, we wish to characterise this object using some notion of a mean and a kernel. To define a mean is straightforward -- this is just going to be a cochain. The appropriate notion of a kernel is less trivial and is given as follows.

\begin{definition}[Kernels on cochains]
    A {\em kernel} on a $k$-cochain $C^k(X)$ is defined as a group bi-homomorphism $k : C_k(X) \times C_k(X) \rightarrow \mathbb{R}$ such that for any set of $k$-chains $c_1, \ldots, c_m \in C_k(X)$, we have
    \begin{align}
        \sum_{i,j = 1}^m k(c_i, c_j) \geq 0.
    \end{align}
\end{definition}

We now prove that a Gaussian random cochain can indeed be characterised by a mean cochain and a kernel.
\begin{theorem}
    A Gaussian random cochain $f_\bullet : \Omega \rightarrow C^k(X)$ is fully characterised by a mean $\mu \in C^k(X)$ and a kernel $k : C_k(X) \times C_k(X) \rightarrow \mathbb{R}$.
\end{theorem}
\begin{proof}
    TODO. Follow the strategy of proof of Theorem 21, \cite{hutchinson2021vector}.
\end{proof}

To represent the mean and the kernel numerically, we fix a labelling $\alpha \mapsto e^k_\alpha$ of the $k$-chains so that we can represent chains and cochains using vectors (see Section \ref{sec:numerical-representation}). Giving this labelling, we can represent a mean $\mu : C_k(X) \rightarrow \mathbb{R}$ and a kernel $k : C_k(X) \times C_k(X) \rightarrow \mathbb{R}$ as 
\begin{align}
    \boldsymbol{\mu} =
    \begin{bmatrix}
    \mu(e^k_1) \\
    \vdots \\
    \mu(e^k_{N_k})
    \end{bmatrix} \in \mathbb{R}^{N_k},
    \qquad
    \mathbf{K} =
    \begin{bmatrix}
    k(e^k_1, e^k_1) & \cdots & k(e^k_1, e^k_{N_k}) \\
    \vdots & \ddots & \vdots \\
    k(e^k_{N_k}, e^k_1) & \cdots & k(e^k_{N_k}, e^k_{N_k})
    \end{bmatrix} \in \mathbb{R}^{N_k \times N_k}.
\end{align}
Thus, as a representation in $\mathbb{R}^{N_k}$, the Gaussian random cochain is simply a multivariate Gaussian $\mathcal{N}(\boldsymbol{\mu}, \mathbb{K})$, with the additional property that it transforms equivariantly under re-labelling.
Below, we show how the kernel matrix transforms given another labelling of the $k$-cells.

\begin{proposition}
Consider two different labellings of the $k$-cells $\alpha \mapsto e^k_\alpha$ and $\beta \mapsto f^k_\beta$ such that $f^k_\beta = e^k_{\sigma(\alpha)}$ for some permutation $\sigma \in S_{N_k}$ of the set $\{1, \ldots, N_k\}$. Let  $\boldsymbol{\mu}_\alpha, \boldsymbol{\mu}_\beta$ be the vector representations of the mean and $\mathbf{K}_\alpha, \mathbf{K}_\beta$ be the matrix representations of the kernel with respect to the two labellings. Representing the permutation as a matrix $\mathbf{\Sigma}_{\alpha \beta}$, we have the following relations
\begin{align}
    \boldsymbol{\mu}_\beta = \mathbf{\Sigma}_{\alpha\beta}^\top \boldsymbol{\mu}_\alpha, \qquad \mathbf{K}_\beta = \mathbf{\Sigma}_{\alpha\beta}^\top \mathbf{K}_\alpha \mathbf{\Sigma}_{\alpha \beta}.
\end{align}
\end{proposition}
\begin{proof}
    TODO.
\end{proof}

In the above, we only considered a Gaussian random cochain of a single order $k$. We can extend this notion in a rather straightforward manner to Gaussian processes on direct sums of cochains of different orders. We take this as our definition of a Gaussian process on a cellular complex.

\begin{definition}[Gaussian processes on cellular complexes] \label{def:gp-on-cell-complex}
    Let $X$ be an $n$-dimensional cellular complex. We define a Gaussian process on $X$ as a random variable $f_\bullet : \Omega \rightarrow \bigoplus_{k=0}^n C^k(X)$ such that for $c = (c_0, \ldots, c_n) \in \bigoplus_{k=0}^n C_k(X)$, the random variable $f_\bullet (c) : \Omega \rightarrow \mathbb{R}$ is univariate Gaussian.
\end{definition}

As before, we have an appropriate notion of a kernel on this space as a group bi-homomorphism
\begin{align}
k : \bigoplus_{k=0}^n C_k(X) \times \bigoplus_{k=0}^n C_k(X) \rightarrow \mathbb{R},
\end{align}
satisfying $\sum_{i,j} k(c_i, c_j) \geq 0$ for $c_i, c_j \in \bigoplus_{k=0}^n C_k(X)$. We thus have the following result characterising GPs on cellular complexes via a mean and a kernel.

\begin{theorem}
    A Gaussian process on a cellular complex $X$ is fully characterised by a mean $\mu \in \bigoplus_{k=0}^n C^k(X)$ and a kernel $k : \bigoplus_{k=0}^n C_k(X) \times \bigoplus_{k=0}^n C_k(X) \rightarrow \mathbb{R}$.
\end{theorem}
\begin{proof}
    TODO.
\end{proof}

\begin{remark}
    To avoid confusion, it is important to note that our notion of a Gaussian process on a cellular complex $X$ is {\em not} defined as a Gaussian process $X \rightarrow \mathbb{R}$ (i.e. a GP on the topological space $X$), but rather as a direct sum of Gaussian random cochains.
\end{remark}

\subsection{Learning and inference}
Here, we describe how to perform model learning and inference in the current setting. Let us consider a dataset $\{(c_i, y_i)\}_{i=1}^N$ consisiting of $y_i \in \mathbb{R}$ and $c_i \in C_k(X)$, modelled as noisy signals of a Gaussian random cochain $f$:
\begin{align}
    y_i = f(c_i) + \epsilon_i.
\end{align}
Here, the $\epsilon_i$'s are i.i.d. mean-zero Gaussian noise with variance $\sigma^2$. The marginal likelihood of observing this data is given by
\begin{align}
    p(\vec{y} | \vec{c}, \vec{\theta}) &= \int p(\vec{y} | f(\vec{c}), \vec{\theta}) p(f(\vec{c}) | \vec{\theta}) \,\mathrm{d} f(\vec{c}) \\
    &= \int \mathcal{N}(\vec{y} | f(\vec{c}), \sigma^2) \,\mathcal{N}(f(\vec{c}) | \vec{\mu}_{\mat{f}}(\vec{\theta}), \mat{K}_{\mat{f}\mat{f}}(\vec{\theta})) \,\mathrm{d} f(\vec{c}) \\
    &= \mathcal{N}(\vec{y} | \vec{\mu}_{\mat{f}}(\vec{\theta}), \mat{K}_{\mat{f}\mat{f}}(\vec{\theta}) + \sigma^2 \mat{I}),
\end{align}
where
\begin{align}
    [\vec{\mu}_{\mat{f}}]_i = \mu(c_i), \quad [\mat{K}_{\mat{f}\mat{f}}]_{ij} = k(c_i, c_j).
\end{align}
Thus, we can learn the model hyperparameters $\vec{\theta}$ by minimising the negative log-likelihood loss
\begin{align}
    \mathcal{L}(\vec{\theta}) &= - \log p(\vec{y} | \vec{c}, \vec{\theta}) \\
    &= \frac12 (\vec{y} - \vec{\mu}_{\mat{f}}(\vec{\theta}))^\top (\mat{K}_{\mat{f}\mat{f}}(\vec{\theta}) + \sigma^2 \mat{I})^{-1} (\vec{y} - \vec{\mu}_{\mat{f}}(\vec{\theta})) + \frac12 \log |\mat{K}_{\mat{f}\mat{f}}(\vec{\theta}) + \sigma^2 \mat{I}|.
\end{align}

Now, given some test points $\vec{c}^*$, we can make inference on this new point by taking
\begin{align}
    p(f(\vec{c}^*) | \vec{y}, \vec{c}) &= \int p(\vec{f}^*| \vec{f}) p(\vec{f} | \vec{y}) \mathrm{d} \vec{f} \\
    &\propto \int p(\vec{y} | \vec{f}) \,p(\vec{f}^*, \vec{f}) \mathrm{d} \vec{f} \\
    &= \mathcal{N}(\vec{f}^* | \vec{\mu}_{\mat{f}^*|\mat{y}}, \mat{\Sigma}_{\mat{f}^*|\mat{y}}),
\end{align}
where
\begin{align}
    \vec{\mu}_{\mat{f}^*|\mat{y}} = \vec{\mu}_{\mat{f}^*} + \mat{K}_{\mat{f}^*\mat{f}}(\mat{K}_{\mat{f}\mat{f}} + \sigma^2 \mat{I})^{-1}(\vec{y} - \vec{\mu}_{\mat{f}}), \quad \mat{\Sigma}_{\mat{f}^*|\mat{y}} = \mat{K}_{\mat{f}^*\mat{f}^*} - \mat{K}_{\mat{f}^*\mat{f}}(\mat{K}_{\mat{f}\mat{f}} + \sigma^2 \mat{I})^{-1}\mat{K}_{\mat{f}\mat{f}^*}.
\end{align}

We can also consider learning with inducing points... (discuss inducing points).

In our examples, the training and test points $\vec{c}, \vec{c}^*$ will be given by a single cell and rarely do we consider inputs that are more general chains. The general chain case may be viewed as interdomain data points since evaluation on a general chain is analogous to taking integrals in the continuous setting.

\subsection{Matérn Gaussian random cochain}
Let $C^k(X)$ be the space of cochains on $X$ equipped with the inner product $\left<\cdot, \cdot\right>_{L^2(w^k)} : C^k(X) \times C^k(X) \rightarrow \mathbb{R}$. Following the construction in \cite{Borovitskiy2021}, we define a Gaussian random $k$-cochain of the Matérn type as a solution to the stochastic system (which we shall make more precise next)
\begin{align}\label{eq:matern-gp-def}
    \left(\frac{2\nu}{\kappa^2} + \Delta_k\right)^{\nu/2} f = \mathcal{W},
\end{align}
where $f \in C^k(X)$ and $\mathcal{W} : \Omega \rightarrow C^k(X)$ is a Gaussian random cochain satisfying $\mathbb{E}[\mathcal{W}(c_0)] = 0$ and $\mathbb{E}[\mathcal{W}(c_1) \mathcal{W}(c_2)] = \left<\tilde{c}_1, \tilde{c}_2\right>_{L^2(w^k)}$ for any $c_0, c_1, c_2 \in C_k(X)$. Recall that $\Delta_k : C^k(X) \rightarrow C^k(X)$ is the Hodge-Laplacian operator on $(C^k(X), \left<\cdot, \cdot\right>_{L^2(w^k)})$ (see Definition \ref{def:hodge-laplacian}). 

The linear operator $\left(\frac{2\nu}{\kappa^2} + \Delta_k\right)^{\nu/2}$ is to be understood as an operation in frequency space, by the following construction.
Let $\{(\lambda_i, u_i)\}_{i=1}^{N_k}$, be solutions to the eigenproblem $\Delta_k u_i = \lambda_i u_i$ such that the ``eigencochains" $\{u_i\}_{i=1}^{N_k}$ are orthonormal in $L^2(w^k)$ (can be checked). Representing $f$ as $f = \sum_i \left<f, u_i\right>_{L^2(w^k)} u_i$, we define
\begin{align} \label{eq:matern-operator}
    \left(\frac{2\nu}{\kappa^2} + \Delta_k\right)^{\nu/2} f := \sum_{i=1}^{N_k} \left(\frac{2\nu}{\kappa^2} + \lambda_i \right)^{\nu/2} \left<f, u_i\right>_{L^2(w^k)} u_i.
\end{align}

Using the numerical representation in Section \ref{sec:numerical-representation}, let $\mathbf{\Lambda} = \mathtt{diag}(\lambda_1, \ldots, \lambda_{N_k})$ be the diagonal matrix of eigenvalues, $\boldsymbol{u}_i$ the vector representation of the eigencochain $u_i$, and $\mathbf{U} = (\boldsymbol{u}_1, \ldots, \boldsymbol{u}_{N_k})$. Then \eqref{eq:matern-operator} may be represented numerically as
\begin{align}
    \mathbf{L} := \mathbf{U} \left(\frac{2\nu}{\kappa^2}\mathbf{I} + \mathbf{\Lambda}\right)^{\nu/2} \mathbf{U}^\top
\end{align}
and \eqref{eq:matern-gp-def} implies
\begin{align}
    &\mathbf{L} \boldsymbol{f} = \boldsymbol{W}, \quad \boldsymbol{W} \sim \mathcal{N}(0, \mathtt{diag}(
    \boldsymbol{w}^k)^{-1}) \\
    &\Leftrightarrow \boldsymbol{f} = \mathbf{L}^{-1} \boldsymbol{W} \sim \mathcal{N}(0, \mathbf{L}^{-1}\mathtt{diag}(
    \boldsymbol{w}^k)^{-1}\mathbf{L}^{-\top}).
\end{align}

This gives us a matrix representation of the kernel as follows
\begin{align}
    \mathbf{K} &= \mathbf{L}^{-1}\mathtt{diag}(
    \boldsymbol{w}^k)^{-1}\mathbf{L}^{-\top} \\
    &= \mathbf{U} \left(\frac{2\nu}{\kappa^2}\mathbf{I} + \mathbf{\Lambda}\right)^{-\nu/2} \underbrace{\mathbf{U}^\top \, \mathtt{diag}(
    \boldsymbol{w}^k)^{-1} \mathbf{U}}_{=\mat{I}} \left(\frac{2\nu}{\kappa^2}\mathbf{I} + \mathbf{\Lambda}\right)^{-\nu/2} \mathbf{U}^\top \\
    &= \mathbf{U} \left(\frac{2\nu}{\kappa^2}\mathbf{I} + \mathbf{\Lambda}\right)^{-\nu} \mathbf{U}^\top. \label{eq:matern-kernel-representation}
\end{align}
We see that the inverse in \eqref{eq:matern-kernel-representation} is applied to a diagonal matrix so most of the computational cost is dominated by the computation of the eigenvalues and vectors $\mathbf{\Delta}_k \boldsymbol{u}_i = \lambda_i \boldsymbol{u}_i$.

\subsection{Reaction-diffusion kernel}
We introduce a new kernel defining a Gaussian process on a cellular complex $X$ (in the sense of Definition \ref{def:gp-on-cell-complex}), which we coin the {\em reaction-diffusion kernel} due to its similarity with the reaction-diffusion PDE. Hereafter, we will operate with the numerical representation of chains and cochains by fixing a labelling of the cells $\{e_\alpha^k\}_{\alpha=1}^{N_k}$.

First, we define the {\em Dirac matrix} as a square matrix of dimension $n = N_1 + N_2 + \ldots + N_d$ defined as
\begin{equation*}
\vec{\mathcal{D}} = 
    \begin{pmatrix}
        \mathbf{0} & \mathbf{B}_1 &  \cdots & \mathbf{0} \\
        \mathbf{B}_1^\top & \ddots & \ddots &  \vdots \\
        \vdots &  \ddots & \ddots & \mathbf{B}_d \\
        \mathbf{0} & \dots & \mathbf{B}_d^\top & \mathbf{0}
    \end{pmatrix}.
\end{equation*}
The Dirac matrix is connected to the Hodge Laplacian matrices under the observation that  
\begin{equation*}
\vec{\mathcal{D}}^2 = 
\vec{\mathcal{L}} = 
\begin{pmatrix}
    \mathbf{\Delta}_1 & \dots & \mathbf{0} \\
    \vdots & \ddots & \vdots \\
    \mathbf{0} & \dots & \mathbf{\Delta}_d
\end{pmatrix},
\end{equation*}
where we used the boundary property \eqref{eq:boundary-of-boundary-matrix}. 
Letting $\vec{\mathcal{L}} = \mat{U} \mat{\Lambda} \mat{U}^\top$ be the eigendecomposition of the symmetric matrix $\vec{\mathcal{L}}$, we have $\vec{\mathcal{D}} = \mat{U} \mat{\Lambda}^{\frac12} \mat{U}^\top$. In particular, the eigenvectors $\mat{U}$ of $\vec{\mathcal{D}}$ are orthogonal even if $\vec{\mathcal{D}}$ is not symmetric.

Now let us consider the stochastic system
\begin{align}
    \left(\mathbf{R} - \mathbf{C} \boldsymbol{\mathcal{D}} +  \mathbf{D} \boldsymbol{\mathcal{L}} \right)^\alpha \vec{f} = \vec{W},
\end{align}
where $\vec{W}$ is a standard Gaussian (TODO: also consider the case of general weights), which in general yields the kernel, expressed formally as
\begin{align} \label{eq:reaction-diffusion-kernel}
    \mathbf{K} = (\mathbf{R} - \mathbf{C} \boldsymbol{\mathcal{D}} +  \mathbf{D} \boldsymbol{\mathcal{L}})^{-\alpha} (\mathbf{R} - \mathbf{C} \boldsymbol{\mathcal{D}}^\top +  \mathbf{D} \boldsymbol{\mathcal{L}})^{-\alpha}.
\end{align}

We refer to this as the {\em reaction-diffusion kernel} as the linear operator $\mathbf{R} - \mathbf{C} \boldsymbol{\mathcal{D}} +  \mathbf{D} \boldsymbol{\mathcal{L}}$ defines the RHS of the reaction-diffusion equation
\begin{equation}
    \frac{\mathrm{d}\vec{f}}{\mathrm{d}t} = \mathbf{R} \vec{f} - \mathbf{C} \vec{\mathcal{D}} \vec{f} + \mathbf{D} \vec{\mathcal{L}} \vec{f}.
\end{equation}
Here, the first term models the reaction, representing the creation of destruction of a quantity, the third term models the diffusion of quantity on a single level, and finally the second term models the cross-diffusion of quantities between different levels. The cross-diffusion term is crucial here for introducing coupling between cochains of different dimensions.

We consider two special cases. In the case $\mat{R} = 2\nu / \kappa^2 \mat{I}$, $\mat{C} = \mat{0}$, $\mat{D} = \mat{I}$ and $\alpha = \nu/2$, we obtain
\begin{align}
    \left(\frac{2\nu}{\kappa^2} \mat{I} + \boldsymbol{\mathcal{L}} \right)^{\nu/2} \vec{f} = \vec{W},
\end{align}
which is analogous to the expression \eqref{eq:matern-gp-def} for a Matérn Gaussian random cochain. However, since the matrix $\frac{2\nu}{\kappa^2} \mat{I} + \boldsymbol{\mathcal{L}}$ is block-diagonal, this does not introduce any coupling between the cochains of different dimensions. Thus, this may be thought of as a system of independent Matérn Gaussian random cochains.

In the second case, we let $\mat{R} = m^2 \mat{I}$, $\mat{C} = \mat{I}$, $\mat{D} = \mat{0}$ and $\alpha = 1$, giving us
\begin{align}
    \left(m \mat{I} - \vec{\mathcal{D}} \right) \vec{f} = \vec{W}.
\end{align}
This has the kernel
\begin{align}
    \mat{K} &= (m \mat{I} - \vec{\mathcal{D}})^{-1} (m \mat{I} - \vec{\mathcal{D}})^{-T} \\
    &= \mat{U}(m \mat{I} - \mat{\Lambda}^\frac12)^{-1} \underbrace{\mat{U}^\top\mat{U}}_{= \mat{I}} (m \mat{I} - \mat{\Lambda}^\frac12)^{-1}\mat{U}^\top \\
    &= \mat{U}(m \mat{I} - \mat{\Lambda}^\frac12)^{-2}\mat{U}^\top \\
    &= (m \mat{I} - \vec{\mathcal{D}})^{-2}.
\end{align}
A similar prior (in the form of a regularizer) is considered in the work \cite{calmon2023dirac} for retrieving topological signals supported on $k$-cells for $k \leq 2$.

In general, when we have $\mat{R} = r \mat{I}$, $\mat{C} = c \mat{I}$ and $\mat{D} = d \mat{I}$ for some constants $r, c, d > 0$, upon considering the eigendecompositions of $\mat{I}, \vec{\mathcal{D}}, \vec{\mathcal{L}}$ and noting that they have a common eigenbasis $\mat{U}$, \eqref{eq:reaction-diffusion-kernel} can be re-expressed as
\begin{align}
    \mathbf{K} &= (\mathbf{R} - \mathbf{C} \boldsymbol{\mathcal{D}} +  \mathbf{D} \boldsymbol{\mathcal{L}})^{-\alpha} (\mathbf{R} - \mathbf{C} \boldsymbol{\mathcal{D}}^\top +  \mathbf{D} \boldsymbol{\mathcal{L}})^{-\alpha} \\
    &= \mat{U} (r \mat{I} - c \mat{\Lambda}^\frac12 + d \mat{\Lambda})^{-2\alpha} \mat{U}^\top,
\end{align}
which can be defined for arbitrary $\alpha \in \mathbb{R}_+$. The main computational burden is in computing the eigendecomposition, which may be prohibitive for large $n$. Thus, a downside of this approach is that it is not applicable for very large complexes (TODO: in the experiments, see if reducing the order of SVD yield results that are any good).

\begin{remark}
Noting that $\vec{\mathcal{L}} = \vec{\mathcal{D}}^2$, we see that the reaction diffusion operator $\mathbf{R} - \mathbf{C} \boldsymbol{\mathcal{D}} +  \mathbf{D} \boldsymbol{\mathcal{L}}$ is a second order polynomial in $\mathcal{D}$. One can further generalise this to consider higher order operators of the form $p_r(\vec{\mathcal{D}})$ where $p_r$ is an arbitrary polynomial of degree $r$.
\end{remark}

\section{Experimental details}

\section{Experimental details}
\section{Experimental details}

\section{FORMATTING INSTRUCTIONS}

To prepare a supplementary pdf file, we ask the authors to use \texttt{aistats2024.sty} as a style file and to follow the same formatting instructions as in the main paper.
The only difference is that the supplementary material must be in a \emph{single-column} format.
You can use \texttt{supplement.tex} in our starter pack as a starting point, or append the supplementary content to the main paper and split the final PDF into two separate files.

Note that reviewers are under no obligation to examine your supplementary material.

\section{MISSING PROOFS}

The supplementary materials may contain detailed proofs of the results that are missing in the main paper.

\subsection{Proof of Lemma 3}

\textit{In this section, we present the detailed proof of Lemma 3 and then [ ... ]}

\section{ADDITIONAL EXPERIMENTS}

If you have additional experimental results, you may include them in the supplementary materials.

\subsection{The Effect of Regularization Parameter}

\textit{Our algorithm depends on the regularization parameter $\lambda$. Figure 1 below illustrates the effect of this parameter on the performance of our algorithm. As we can see, [ ... ]}

\vfill